%% file: main.tex
\documentclass[twoside]{article}

\input{preamble}

\usepackage[accepted]{aistats2024}

\begin{document}

\twocolumn[

\aistatstitle{Optimising Distributions with Natural Gradient Surrogates}

\aistatsauthor{Jonathan So \And Richard E. Turner}

\aistatsaddress{University of Cambridge \And University of Cambridge} ]

\begin{abstract}
    \input{abstract}
\end{abstract}

\section{INTRODUCTION}
\input{introduction}

\section{BACKGROUND}
\input{background}

\section{METHOD}
\input{method}

\section{RESULTS}
\input{results}

\section{RELATED WORK}
\input{relatedwork}

\section{DISCUSSION}
\input{discussion}

\subsubsection*{Acknowledgements}
\input{acknowledgements}

\bibliography{bibliography}

\section*{Checklist}

\begin{enumerate}
    \item For all models and algorithms presented, check if you include:
        \begin{enumerate}
            \item A clear description of the mathematical setting, assumptions, algorithm, and/or model. [Yes, see Section \ref{sec:method} and Appendix \ref{app:validity}.]
            \item An analysis of the properties and complexity (time, space, sample size) of any algorithm. [Yes, see Appendix \ref{app:algorithms_sngd}.]
            \item (Optional) Anonymized source code, with specification of all dependencies, including external libraries. [Yes, see Appendix \ref{app:experiment_details}.]
        \end{enumerate}

    \item For any theoretical claim, check if you include:
        \begin{enumerate}
            \item Statements of the full set of assumptions of all theoretical results. [Yes, see Section \ref{subsec:method_equivalence} and Appendix \ref{app:validity}.]
            \item Complete proofs of all theoretical results. [Yes, see Appendix \ref{app:validity}.]
            \item Clear explanations of any assumptions. [Yes, see Section \ref{subsec:method_equivalence} and Appendix \ref{app:validity}.]
        \end{enumerate}

    \item For all figures and tables that present empirical results, check if you include:
        \begin{enumerate}
            \item The code, data, and instructions needed to reproduce the main experimental results (either in the supplemental material or as a URL). [Code to be released prior to publication.]
            \item All the training details (e.g., data splits, hyperparameters, how they were chosen). [Yes, see Appendix \ref{app:experiment_details}.]
            \item A clear definition of the specific measure or statistics and error bars (e.g., with respect to the random seed after running experiments multiple times). [Yes, see Appendix \ref{app:experiment_details}.]
            \item A description of the computing infrastructure used. (e.g., type of GPUs, internal cluster, or cloud provider). [Yes, see Appendix \ref{app:experiment_details}.]
        \end{enumerate}

    \item If you are using existing assets (e.g., code, data, models) or curating/releasing new assets, check if you include:
        \begin{enumerate}
            \item Citations of the creator If your work uses existing assets. [Yes, see Appendix \ref{app:experiment_details}.]
            \item The license information of the assets, if applicable. [Yes, see Appendix \ref{app:experiment_details}.]
            \item New assets either in the supplemental material or as a URL, if applicable. [Not Applicable.]
            \item Information about consent from data providers/curators. [Not Applicable.]
            \item Discussion of sensible content if applicable, e.g., personally identifiable information or offensive content. [Not Applicable.]
        \end{enumerate}

    \item If you used crowdsourcing or conducted research with human subjects, check if you include:
        \begin{enumerate}
            \item The full text of instructions given to participants and screenshots. [Not Applicable.]
            \item Descriptions of potential participant risks, with links to Institutional Review Board (IRB) approvals if applicable. [Not Applicable.]
            \item The estimated hourly wage paid to participants and the total amount spent on participant compensation. [Not Applicable.]
        \end{enumerate}
\end{enumerate}

\onecolumn

\appendix

\section{GRADIENT NOTATION}
\input{appendix_gradientnotation}

\section{EQUIVALENCE WITH OPTIMISATION OF $f$}
\input{appendix_validity}

\section{ALGORITHMS}
\input{appendix_algorithms}

\section{EXPERIMENT DETAILS}
\input{appendix_experimentdetails}

\section{PARAMETER MAPPINGS}
\input{appendix_mappings}

\section{ADDITIONAL RESULTS}
\input{appendix_results}

\newpage
\section{CHOOSING EFFECTIVE SURROGATES}
\input{appendix_choosingsurrogates}

\section{COMPARISON WITH NGD UNDER $q$}
\input{appendix_ngdcomparison}

\section{EF NATURAL GRADIENTS}
\input{appendix_efnatgrads}

\section{EF MIXTURE MODELS}
\input{appendix_efmixture}

\end{document}

%% file: preamble.tex
\usepackage{algorithm}
\usepackage{algpseudocode}
\usepackage{appendix}
\usepackage{amsmath}
\usepackage{amssymb}
\usepackage{amsthm}
\usepackage{array}
\usepackage{bm}
\usepackage{booktabs}
\usepackage{cancel}
\usepackage{enumitem}
\usepackage{graphicx}
\usepackage{mathtools}
\usepackage{natbib}
\usepackage{physics}
\usepackage{rotating}
\usepackage{subcaption}
\usepackage{url}
\usepackage{xcolor}
\usepackage[hidelinks]{hyperref}

\setcitestyle{authoryear,open={(},close={)}}

\DeclareMathOperator*{\argmin}{arg\,min}

\newcommand{\mbb}[1]{{\mathbb{#1}}}
\newcommand{\mc}[1]{{\mathcal{#1}}}

\newcommand{\bs}{\boldsymbol}

\newtheorem{proposition}{Proposition}

\DeclarePairedDelimiterX{\infdivx}[2]{[}{]}{%
  #1\;\delimsize\|\;#2%
}
\newcommand{\kl}{\mathrm{KL}\infdivx}

%% file: abstract.tex
Natural gradient methods have been used to optimise the parameters of probability distributions in a variety of settings, often resulting in fast-converging procedures. Unfortunately, for many distributions of interest, computing the natural gradient has a number of challenges. In this work we propose a novel technique for tackling such issues, which involves reframing the optimisation as one with respect to the parameters of a \textit{surrogate} distribution, for which computing the natural gradient is easy. We give several examples of existing methods that can be interpreted as applying this technique, and propose a new method for applying it to a wide variety of problems. Our method expands the set of distributions that can be efficiently targeted with natural gradients. Furthermore, it is fast, easy to understand, simple to implement using standard autodiff software, and does not require lengthy model-specific derivations. We demonstrate our method on maximum likelihood estimation and variational inference tasks.

%% file: introduction.tex
\label{sec:introduction}
Many problems in scientific research require the optimisation of probability distribution parameters. Such problems are of fundamental importance in the fields of machine learning and statistics.  The goal of these optimisations is to minimise a function $f(\theta)$ where $\theta \in \Theta$ is the parameter vector of a probability distribution $q$. Notable examples include maximum likelihood estimation (MLE) and variational inference (VI).

When $f$ is differentiable, the archetypal optimisation method is gradient descent (GD).  GD takes steps in the direction in parameter space that gives the greatest decrease in $f$ for an infinitesimally small step size, which is equivalent to following the direction of the negative gradient.  While GD has many desirable properties, convergence can often be slow, and much effort has been spent in designing methods that offer improved convergence properties.

Natural gradient descent (NGD) is one such method for optimising probability distribution parameters \citep{amari1998natgrad}. In contrast with GD, NGD moves the parameters in the direction of steepest descent on a manifold of probability distributions, where steepness is defined with respect to a divergence measure between distributions. This notion of steepness is intrinsic to the manifold, and so it is not dependent on parameterisation, dispensing with a well-known pathology of GD.

There are some distributions for which the natural gradient can be computed efficiently. However, in its most general form, computing the natural gradient involves instantiating and inverting the Fisher information matrix, defined as
\begin{align}
    F(\theta) &= \mbb{E}_{q_\theta(x)}\Big[\nabla_\theta[\log q_\theta(x)]\nabla_\theta[\log q_\theta(x)]^\top\Big], \label{eqn:fisher}
\end{align}
posing several challenges.\footnote{We use e.g. $q$, to refer to both a distribution and its density function, with the intention made clear from context. We will use a subscript, e.g. $q_\theta$, when we wish to make the dependence on parameters explicit. Our gradient notation is explained in Appendix \ref{app:gradientnotation}.} First, for some distributions, the Fisher can become singular, in which case the natural gradient is undefined. Second, it may not be available in closed form, in which case it must be estimated using sampling or other numerical methods.  Finally, instantiating and inverting it explicitly costs $O(m^2)$ in memory and $O(m^3)$ in computation, for $m$ the number of parameters. For even moderately large $m$, this can be prohibitively expensive.

In this paper we propose a simple technique for tackling optimisation problems in which computing the natural gradient is problematic. Namely, we reframe the problem as an optimisation with respect to a \textit{surrogate} distribution $\tilde q$, for which computing natural gradients is easy, and perform the optimisation in that space. With a judiciously chosen surrogate, convergence can be rapid.

We find that a number of existing methods can be interpreted as applications of this technique: the natural gradient VI method of \citet{lin2020mcef} for exponential family mixtures, stochastic natural gradient expectation propagation \citep{hasenclever2017}, a fixed-point iteration scheme for optimising elliptical copulas \citep{hernandez2014copula}, as well as the typical use of natural gradients in large-scale supervised learning settings. We also describe a new method for applying the technique to a wide variety of problems, using exponential family (EF) surrogate distributions. Our method is easy to understand, simple to implement using standard autodiff software, and does not require the use of lengthy model-specific derivations.

We present a number of experiments in which we find that our method significantly reduces the time to convergence compared to existing best-practice methods, both in number of iterations and wall-clock time. Our experiments consist of a variety of MLE and VI tasks, but we note that the method is applicable more generally to optimisations involving probability distributions, or indeed to any optimisation in which we think a manifold of probability distributions may serve as a useful surrogate for the solution space.

Our main contributions can be summarised as follows:
\begin{enumerate}
    \item We propose a novel technique for optimising probability distribution parameters.
    \item We prove the validity of this technique under stated conditions.
    \item We find several examples of existing methods that can be viewed as applying this technique.
    \item We describe a simple new method for applying it, using known properties of EF distributions.
    \item We present a variety of MLE and VI experiments in which our method achieves significantly faster convergence than existing best-practice methods.
\end{enumerate}

This paper is structured as follows. In Section \ref{sec:background} we provide a brief overview of NGD and EF distributions. Section \ref{sec:method} covers items 1 and 4 of the contributions above, as well as an overview of item 2 (details found in Appendix \ref{app:validity}). Sections \ref{sec:experiments} and \ref{sec:relatedwork} cover items 5 and 3 respectively. Finally, in Section \ref{sec:discussion}, we discuss limitations of our method, and avenues for future research.

%% file: background.tex
\label{sec:background}
In this section we begin with an overview of NGD. We then briefly describe a class of distribution families for which we can efficiently compute natural gradients, namely exponential families. NGD and EF distributions will together serve as a foundation for the method we introduce in Section \ref{sec:method}.

\subsection{Natural Gradient Descent}
\label{subsec:background_ngd}

NGD updates the parameters of a distribution $q$ by taking a step proportional to the gradient of the objective $f$, preconditioned by the inverse of the Fisher matrix of $q$.  That is, the update in parameters at step $t$ of the optimisation is given by
\begin{align}
    \theta_{t+1} &= \theta_t - \epsilon_t [F(\theta_t)]^{-1}\nabla f(\theta_t) \label{eqn:ngd_step}
\end{align}
where $F(\theta)$ is given by \eqref{eqn:fisher}. The step size $\epsilon_t$ may follow some pre-defined schedule, or be found by line search. It can be shown that as $\epsilon_t \rightarrow 0^+$, update \eqref{eqn:ngd_step} moves $\theta$ in the direction of steepest descent in $f$, on the manifold of probability distributions spanned by $q$, where steepness is defined with respect to a divergence measure between distributions \citep{ollivier2017natgrads}. Because this notion of steepness is intrinsic to the manifold, NGD is locally invariant to parameterisation.

When NGD is applied to MLE objectives, it is Fisher-efficient \citep{amari1998natgrad}, and can be seen as a robust approximation to Newton's method \citep{martens2015kfac}. These properties do not, in general, apply outside of this setting. Nevertheless, natural gradient methods have been applied in several other settings, often resulting in rapidly converging procedures \citep{kakade2001npg,hoffman2013svi,khan2017van,hasenclever2017}.

\subsection{Exponential Family Distributions}
\label{subsec:background_ef}

In this section we provide a brief introduction to EF distributions, whose properties support efficient natural gradient computation.

The EF of distributions defined by the vector-valued statistic function $t$ and base measure $\nu$, has density
\begin{align}
    q_\eta(x) &= \nu(x)\exp\big(t(x)^\top \eta - A(\eta)\big),
\end{align}
where $A$ is the log partition function, and $\eta$ are the \textit{natural parameters}. When the components of $t$ are linearly independent, the family is said to be \textit{minimal}.

There is an alternative parameterisation of $q$, given by 
\begin{align}
    \mu(\eta) &= \mbb{E}_{q_\eta(x)}[t(x)],
\end{align} where $\mu$ are known as the the \textit{mean parameters} of $q$. In minimal families the correspondence between $\eta$ and $\mu$ is one-to-one, and we denote the reverse map as $\eta(.)$.\footnote{Due to our overloaded use of $\mu$ and $\eta$ to denote both specific parameters and maps between parameterisations, we will use e.g. $\mu(.)$ to disambiguate the latter if the intention is not clear from context.} In this paper we will only consider minimal EFs.

Notable examples of EF distributions include the multivariate normal, gamma, categorical, and Dirichlet distributions to name but a few.

Let $f_\eta$ and $f_\mu$ be functions related by $f_\eta = f_\mu \circ \mu$. It is a remarkable property of EFs that the natural gradients of $f_\eta$ and $f_\mu$ with respect to $\eta$ and $\mu$, respectively, are given by
\begin{align}
    [F(\eta)]^{-1}\nabla f_\eta(\eta) &= \nabla f_\mu\boldsymbol(\mu(\eta)\bs) \label{eqn:ef_natgrad_natparams} \\
    [F(\mu)]^{-1}\nabla f_\mu(\mu) &= \nabla f_\eta\bs(\eta(\mu)\bs).\label{eqn:ef_natgrad_meanparams}
\end{align}
That is, the natural gradient in one parameterisation is simply given by the regular gradient with respect to the other \citep{hensman2012expfam}. Provided that we have an efficient way of converting between natural and mean parameters, this allows us to compute the natural gradient without explicitly instantiating or inverting $F(\theta)$.

%% file: method.tex
\label{sec:method}

The exponential families introduced in Section \ref{sec:background} provide us with a relatively rich set of distributions for which we can efficiently compute natural gradients.  However, there are many distributions outside of this set for which computing natural gradients remains difficult. In this section, we present details of our technique, which expands the set of distributions that we can efficiently target with natural gradients.

\subsection{Surrogate Natural Gradient Descent}
\label{subsec:method_sngd}

The main idea in this paper is simple.  When faced with an optimisation objective $f(\theta)$, where $\theta \in \Theta$ are the parameters of a distribution $q$ for which computing the natural gradient update is problematic, we solve the problem in two steps. First, we re-parameterise $f$ as a function of some other parameters $\tilde{\theta} \in \tilde{\Theta}$, related by $\theta = g(\tilde\theta)$, and define the reparameterised objective
\begin{align}
    \tilde{f}(\tilde{\theta}) = f\bs(g(\tilde\theta)\bs).
\end{align}
Second, we \textit{interpret} $\tilde{\theta}$ as the parameters of a surrogate distribution $\tilde q$, for which computing natural gradients is easy, and then perform NGD in $\tilde{f}$ with respect to $\tilde q$ and $\tilde\theta$. That is, we perform a sequence of updates
\begin{align}
    \tilde{\theta}_{t+1} &= \tilde{\theta}_t - \epsilon_t [\tilde{F}(\tilde{\theta}_t)]^{-1}\nabla\tilde{f}(\tilde\theta_t), \label{eqn:sngd_update}
\end{align}
where $\tilde{F}(\tilde\theta)$ is the Fisher information matrix of $\tilde q$. We refer to $q$ and $\tilde q$ as the \textit{target} and \textit{surrogate} distributions respectively. Update \eqref{eqn:sngd_update} has the straightforward interpretation of performing preconditioned GD in a reparameterised objective. Upon converging to a local minimiser $\tilde\theta^*$ of $\tilde f$, a solution to the original problem is obtained by $\theta^* = g(\tilde\theta^*)$. We call this technique \textit{surrogate natural gradient descent} (SNGD).

As a simple example, when $q$ is a multivariate Student's $t$ distribution with known degrees of freedom $\nu$, we can choose $\tilde q$ to be multivariate normal. If $\tilde\theta$ contains the mean and scale matrix parameters of $\tilde q$, and $\theta$ the mean and covariance matrix of $q$, then a natural choice for $g$ is simply the identity map on $\tilde\Theta = \Theta = \mathbb{R}^d\times\mathbb{S}_{++}^d$, where $\mathbb{S}_{++}^d$ is the set of $d$-dimensional positive-definite matrices.

It is natural to question which properties SNGD shares with NGD under $q$. Given that SNGD \emph{is} NGD (with respect to $\tilde q$), it remains locally invariant to parameterisation, and performs steepest descent in a statistical manifold (that of $\tilde q$). What is not guaranteed, is that the statistical manifold of $\tilde q$ remains useful for the optimisation of $\tilde f$. For example, in MLE settings, SNGD will not in general retain the asymptotic efficiency of NGD under $q$. However, what it gains is tractability, and as we demonstrate in Section \ref{sec:experiments}, the practical performance benefits can be significant. In some cases, remarkably, SNGD can actually outperform NGD under $q$, with respect to both the original parameters $\theta$, and the reparameterisation defined by $g$.

\subsection{Choice of Surrogate}
\label{subsec:method_choosingsurrogates}

The performance of SNGD relies crucially on appropriate choices for $\tilde q$ and $g$. A natural assumption would be that $\tilde q_{\tilde\theta}$ should be an approximation to $q_{g(\tilde\theta)}$, and in some cases this can be a useful guide. However, often an effective surrogate can be found that has support over an entirely different space than that of $q$; in fact, this is true for most of the examples in this paper, as can be seen in Table \ref{tbl:models}.

A more general principle is that a surrogate should be chosen such that $\tilde\theta$ has similar \emph{local} effects on Kullback-Leibler (KL) divergences in $q$ than it does in $\tilde q$. More specifically, the effect of a small (infinitesimal) perturbance $\delta$ in $\tilde\theta$ should have a similar impact on the KL divergence from $\tilde q_{\tilde\theta+\delta}$ to $\tilde q_{\tilde\theta}$ as it does on that from $q_{g(\tilde\theta+\delta)}$ to $q_{g(\tilde\theta)}$.\footnote{See Appendix \ref{app:choosingsurrogates} for a definition of the KL divergence.} That is, at $\delta =\mathbf{0}$, we would like:
\begin{align}
    \nabla_\delta\Big(\kl[\big]{\tilde q_{\tilde\theta+\delta}}{\tilde q_{\tilde\theta}}\Big) &= \nabla_\delta\Big(\kl[\big]{ q_{g(\tilde\theta+\delta)}}{q_{g(\tilde\theta)}}\Big).\label{eqn:klapprox_1}
\end{align}
If equality \eqref{eqn:klapprox_1} is satisfied then SNGD will move in the same direction in $\tilde\theta$ as NGD under (reparameterised) $q$ (see Appendix \ref{app:choosingsurrogates} for details). Finding tractable surrogates for which equality holds exactly will not typically be possible, but this motivates choosing $\tilde q$, $g$, for which it is approximately true.

\subsection{Equivalence with Optimisation of $f$}
\label{subsec:method_equivalence}

In Appendix \ref{app:validity} we prove a number of results regarding the validity of SNGD, which we summarise here.

Let $\Theta$ and $\tilde\Theta$ be open subsets of $\mbb{R}^i$ and $\mbb{R}^j$ respectively, with $i \le j$. Let $g: \tilde\Theta \rightarrow \Theta$ be twice differentiable, with Jacobian of rank $i$ everywhere, and with $g(\tilde\Theta) = \Theta$. Finally, let $f : \Theta \rightarrow \mbb{R}$ be twice continuously differentiable, and define $\tilde f = f \circ g$.

Under these conditions, optimising $\tilde f$ is equivalent to optimising $f$ in the following sense: finding a local minimiser of $\tilde f$ also gives us a local minimiser for $f$, and all local minima of $f$ are attainable through $\tilde f$. Furthermore, $\tilde f$ does not have any non-strict saddle points that are not also present at the corresponding points in $f$, and so $\tilde f$ does not introduce any additional spurious attractors \citep{lee2016gd}.

\subsection{Exponential Familly Surrogates}

In this paper we choose $\tilde q$ from the set of EF distributions described in Section \ref{sec:background}, allowing us to perform update \eqref{eqn:sngd_update} efficiently without explicitly instantiating or inverting $\tilde F(\tilde\theta)$.  Algorithm \ref{alg:sngd} provides pseudocode for an implementation of SNGD when $\tilde\Theta$ is either the mean or natural domain of an EF. It assumes the existence of an autodiff operator $\texttt{grad}$, and an overloaded function $\texttt{dualparams}$, which maps from natural to mean parameters, or vice-versa, depending on the type of its argument. A full explanation of Algorithm \ref{alg:sngd} can be found in Appendix \ref{app:algorithms_sngd}.

\begin{algorithm}
    \caption{SNGD with EF surrogate}
    \label{alg:sngd}
    \begin{algorithmic}
        \Require objective $f : \Theta \rightarrow \mbb{R}$
        \Require parameter mapping $g : \tilde{\Theta} \rightarrow \Theta$
        \Require initial surrogate parameters $\tilde\theta_0 \in \tilde\Theta$
        \Require step size schedule $\{\epsilon_t \in \mbb{R}_+ : t = 0, 1, ... \}$
        
        \State $\tilde f_\text{dual}(.) := f(g(\texttt{dualparams}(.)))$
        \State $t \leftarrow 0$
        \While{not converged}
            \State $\tilde{\nabla} \leftarrow \texttt{grad}[\tilde f_\text{dual}](\texttt{dualparams}(\tilde\theta_t))$
            \State $\tilde\theta_{t+1} \leftarrow \tilde\theta_t - \epsilon_t\tilde{\nabla}$
            \State $t \leftarrow t + 1$
        \EndWhile \\
        \Return{$g(\tilde\theta_t)$}
    \end{algorithmic}
\end{algorithm}

The domains of valid natural and mean parameters of an EF are open convex sets. It is possible that finite-length NGD steps in these parameters can move outside of the valid domain. These steps can be corrected by using a backtracking linesearch, an approach taken in previous work on natural gradients \citep{khan2017ccvi}. However, we find that when this is required, the backtracking typically occurs so infrequently as to have a relatively small impact on performance.

\subsection{Auxiliary Parameters}
\label{subsec:method_auxiliaryparams}

In some cases, a chosen surrogate's parameters, $\tilde\theta$, may not be sufficient to fully specify $\theta$. To handle such cases, we can generalise SNGD by augmenting $\tilde\theta$ with an additional vector of parameters, $\lambda \in \Lambda$, that are \emph{not} optimised with natural gradients. Our reparameterised objective is then given by 
\begin{align}
    \tilde{f}(\tilde{\theta}, \lambda) = f\bs(g(\tilde\theta, \lambda)\bs)\label{eqn:ftilde_auxiliary}.
\end{align}
We optimise \eqref{eqn:ftilde_auxiliary} by updating $\tilde\theta$ using natural gradients as before, and $\lambda$ using standard gradient-based techniques, either jointly or in alternation. For example, we can apply the sequence of updates
\begin{align}
    \tilde{\theta}_{t+1} &= \tilde{\theta}_t - \epsilon_t [\tilde{F}(\tilde{\theta}_t)]^{-1}\nabla_{\tilde\theta}\tilde{f}(\tilde\theta_t, \lambda_t), \label{eqn:sngd_auxiliary_theta_update} \\
    \lambda_{t+1} &= \lambda_t - \varepsilon_t\nabla_{\lambda}\tilde{f}(\tilde\theta_t, \lambda_t), \label{eqn:sngd_auxiliary_lambda_update}
\end{align}
where $\epsilon_t$, $\varepsilon_t$ are the step sizes for $\tilde\theta$, $\lambda$, respectively at time $t$. Update \eqref{eqn:sngd_auxiliary_lambda_update} corresponds to GD in $\lambda$, but we may equally use any other first-order optimiser.

Revisiting our example from Section \ref{subsec:method_sngd}, if $q$ is a multivariate Student's $t$ distribution, but now with \emph{unknown} degrees of freedom $\nu$, we can again directly map from $\tilde\theta$, which contains the location and scale matrix parameters of $\tilde q$, to the mean and covariance of $q$. $\nu$ has no analogue in the multivariate normal distribution however, and so we can capture this with $\lambda = (\nu)$. $g$ can then simply be the identity map on $\tilde\Theta\times\Lambda = \Theta = \mathbb{R}^d\times\mathbb{S}_{++}^d\times \mathbb{R}_+$.\footnote{In practice it would be preferable to choose $\lambda = (\log\nu)$, with $g$ adjusted accordingly, so that optimisation of $\lambda$ is unconstrained.}

This extension expands the set of distributions that can be targeted by SNGD. Pseudocode for this extension is given in Appendix \ref{app:algorithms_sngd_auxiliaryparams}, with examples appearing in the experiments of Sections \ref{subsec:experiments_skewelliptical} and \ref{subsec:experiments_ellipticalcopula}.

%% file: results.tex
\label{sec:experiments}

\begin{table*}[t]
  \caption{Summary of the surrogate-target pairs appearing in this paper, with links to the relevant sections. For the sake of readability we have omitted the conventional multivariate prefix from distribution names, as this information is conveyed by the support.  Parameter mappings for these examples are found in Appendix \ref{app:mappings}.}
  \label{tbl:models}
  \centering
  \begin{tabular}{llllc}
    \toprule
    \multicolumn{2}{c}{\textbf{TARGET}} & \multicolumn{2}{c}{\textbf{SURROGATE}} \\
    \cmidrule(r){1-2} \cmidrule(r){3-4}
    \textbf{DISTRIBUTION} & \textbf{SUPPORT} & \textbf{DISTRIBUTION} & \textbf{SUPPORT} & \textbf{SECTION} \\
    \midrule
    Negative binomial& $\mbb{N}$ & Gamma  & $\mbb{R}_+$ & \ref{subsec:experiments_negbin} \\
    Negative binomial mixture & $\mbb{N}$ & Gamma mixture model & $\mbb{R}_+\times\{1, ..., k\}$ & \ref{subsec:experiments_mixture} \\
    % Student's $t$ & $\mbb{R}^d$ & Normal & $\mbb{R}^d$ & \ref{subsec:experiments_mvt} \\
    Skew-normal & $\mbb{R}^d$ & Normal & $\mbb{R}^d$ & \ref{subsec:experiments_skewelliptical} \\
    Skew-normal mixture & $\mbb{R}^d$ & Normal mixture model & $\mbb{R}^d\times\{1, ..., k\}$ & \ref{subsec:experiments_mixture} \\
    Skew-$t$ & $\mbb{R}^d$ & Normal & $\mbb{R}^d$ & \ref{subsec:experiments_skewelliptical} \\
     % & $\mbb{R}^d$ & MCEF Student's $t$ & $\mbb{R}^d\times\mbb{R}_+$ & \ref{subsec:experiments_skewelliptical} \\
     % & $\mbb{R}^d$ & MCEF Skew-Normal & $\mbb{R}^d\times\mbb{R}$ & \ref{subsec:experiments_skewelliptical} \\
    Elliptical copula & $[0,1]^d$ & Zero-mean normal & $\mbb{R}^d$ & \ref{subsec:experiments_ellipticalcopula} \\
    % Kent & $S^2$ & Normal & $\mbb{R}^2$ & \ref{subsec:experiments_kentdistribution} \\
    % Generalised Dirichlet & $ \Delta^{d-1}$ & Dirichlet & $\Delta^{d-1}$ & \ref{subsec:experiments_skewelliptical} \\
    \bottomrule
  \end{tabular}
\end{table*}

In this section we present a number of experiments demonstrating the utility of SNGD. Natural gradient methods have been used in a variety of settings, but perhaps most extensively for MLE \citep{amari1998natgrad,bernacchia2018natgrad,martens2015kfac,ren2019natgrad,roux2007topmoumoute} and VI \citep{hoffman2013svi,khan2017ccvi,khan2017van,khan2018vogn,lin2020mcef,salimbeni2018natgrad}. We therefore used a variety of MLE and VI tasks as test cases.

In MLE the goal is to find parameters $\theta$ of distribution $q$ that maximise the (log) likelihood of observed data $\{x_i\}_{i=1}^n$ under $q$. That is, the objective function for MLE tasks is of the form
\begin{align}
    f(\theta) &= -\sum_{i=1}^n\log q_\theta(x_i) .\label{eqn:mle_objective}
\end{align}
In VI we are given a generative model $p(x)p(\mathcal{D}\mid x)$, and observed data $\mathcal{D}$, and the goal is to find parameters $\theta$ of distribution $q$ that minimise the Kullback-Leibler (KL) divergence from $q$ to the posterior $p(x|\mathcal{D})$. The objective function for VI tasks is of the form
\begin{align}
    f(\theta) &= -\mathbb{E}_{q_\theta(x)}\bigg[\log\frac{p(\mathcal{D},x)}{q_\theta(x)}\bigg]. \label{eqn:vi_objective}
\end{align}
It can be shown that minimising \eqref{eqn:vi_objective} is equivalent to maximising a lower bound on the log marginal likelihood: $\log p(\mathcal{D})$. In the VI experiments of this section, gradients of \eqref{eqn:vi_objective} were estimated by applying the reparameterisation trick \citep{kingma2014aevb} to the target distribution.

In each experiment, we compared SNGD with a number of baselines. For tasks that were small scale, and had objectives that could be computed deterministically, we compared SNGD with two baselines: GD and BFGS \citep{nocedal2006opt}. In these cases, we used exact line search to determine step sizes for each method, in order to compare search direction independent of hyperparameter settings. 
For larger tasks, or where only stochastic estimates of $f$ were available (minibatched or VI experiments), we used Adam as the standard baseline, and used grid search to determine the best hyperparameters for each method.

Further details of all experiments are given in Appendix \ref{app:experiment_details}. A summary of surrogate-target pairs used for SNGD in these experiments can be found in Table \ref{tbl:models}, and details of parameter mappings can be found in Appendix \ref{app:mappings}.

We display average training curves for each experiment, with error bars corresponding to 2 standard errors, computed across 10 random seeds. In Appendix \ref{app:experiment_results} we plot Pareto frontiers, showing the best performance achieved by \emph{any} hyperparameter setting, as a function of both iteration count and wall-clock time.

In this section we use $n$ to denote the number of training data points, and $d$ for the dimensionality of the random variable under $q$. For brevity, we omit the conventional multivariate prefix from distribution names in this section, as this is implicit for $d > 1$.

\subsection{Negative Binomial Distribution}
\label{subsec:experiments_negbin}

The negative binomial distribution is widely used for modelling discrete data \citep{fisher1941negbin,lloydsmith2005superspreading,lloydsmith2007negbin,orooji2021countregression,kendall2023covid}. Maximum likelihood estimates of the negative binomial parameters are not available in closed form, and must be found numerically.

Let $\theta = (r, s)$ be the parameters of the negative binomial distribution with probability mass function
\begin{align}
    q(x) &= \binom{x+r-1}{x}(1-s)^xs^r.\label{eqn:negbin_pmf}
\end{align}
The gamma distribution can be seen as a continuous analogue of the negative binomial, and as it is a EF, it is a natural choice of surrogate when targeting the negative binomial with SNGD.

Let $\tilde q$, therefore, be the gamma distribution with parameters $\alpha$, $\beta$ and probability density function
\begin{align}
    \tilde q(x) &= \frac{\beta^\alpha}{\Gamma(\alpha)}x^{\alpha-1}\exp^{-\beta x}.
\end{align}
\citet{guenther1972negbin} used the gamma distribution to approximate the CDF of the negative binomial, using the mapping $(r, s) = (\alpha/(1-\beta), \beta)$. Let $g$ be the equivalent mapping, but defined in terms of \textit{mean} parameters of $\tilde q$. To ensure that $r$ and $s$ correspond to valid negative binomial parameters, we restrict $\Theta$ to the subset of gamma mean parameters for which $\beta < 1$.\footnote{It can be shown that $\Theta$ remains an open convex set.} Finally, let $f$ be defined as \eqref{eqn:mle_objective}, where $\{x_i\}_{i=1}^n$ is the dataset of \citet{fisher1941negbin}, consisting of counts of ticks observed on a population of sheep ($n$=82, $d$=1). 

Using these definitions, we applied Algorithm \ref{alg:sngd}, modified to use a line search as discussed at the start of this section. We compared SNGD with our standard baselines: GD and BFGS. Additionally, this task was sufficiently small that we were able to compare with NGD under $q$ by estimating and inverting the Fisher directly.

Figure \ref{fig:negbin_experiment} shows that SNGD outperformed all of the baselines. Remarkably, this included NGD under $q$; this surprising result is explored in detail in Appendix \ref{app:ngdcomparison}. There, we show that this was true even when NGD was performed with respect to the reparameterisation defined by $g$, implying that this is not simply an artefact of SNGD using a more effective parameterisation; rather, it is the change of parameterisation \emph{and} distribution that results in such rapid convergence.

\begin{figure}[ht]
    \centering
    \begin{subfigure}[t]{.49\linewidth}
        \centering
        \includegraphics[height=.9\linewidth]{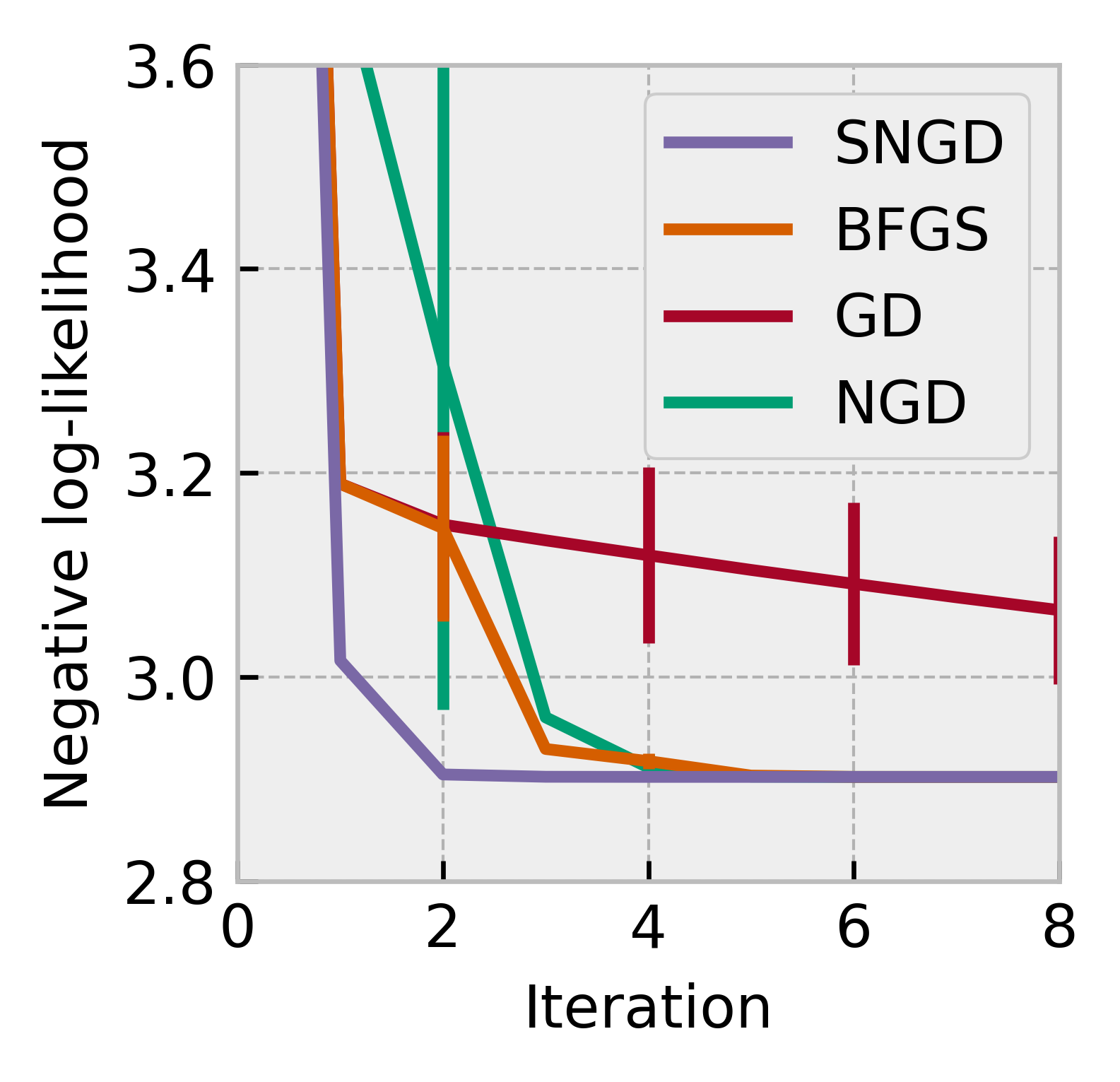}
    \end{subfigure}
    \hfill
    \begin{subfigure}[t]{.49\linewidth}
        \centering
        \includegraphics[height=.9\linewidth]{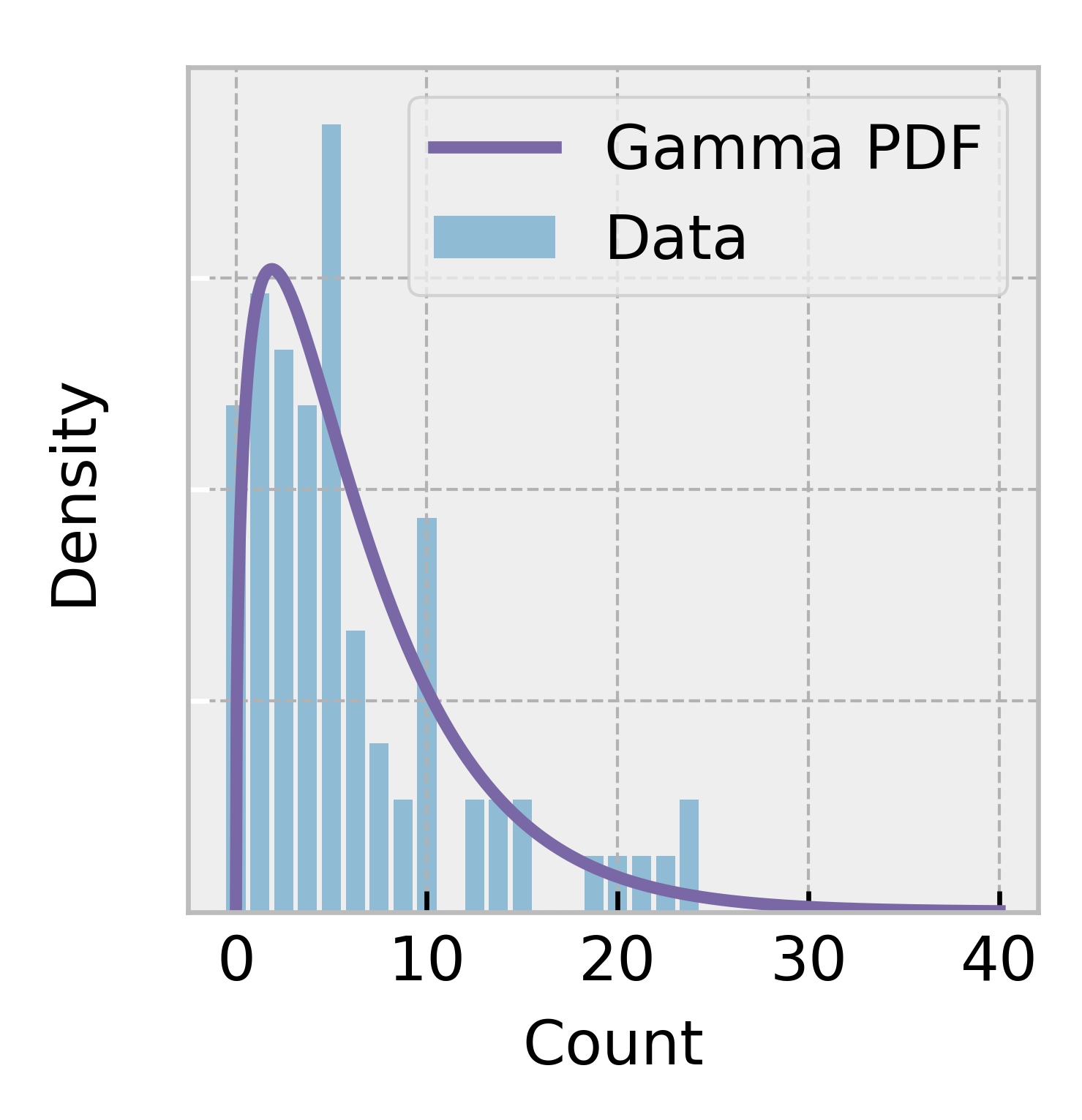}
    \end{subfigure}
    \caption{Negative binomial MLE on the sheep dataset. (left) Training curves. (right) Histogram of the observed counts, overlaid with the PDF of the gamma surrogate of SNGD at convergence.}
    \label{fig:negbin_experiment}
\end{figure}

\subsection{Skew-Elliptical Distributions}
\label{subsec:experiments_skewelliptical}

\begin{figure}[t]    
    \begin{subfigure}[t]{.49\linewidth}
        \centering
        \includegraphics[height=.9\linewidth]{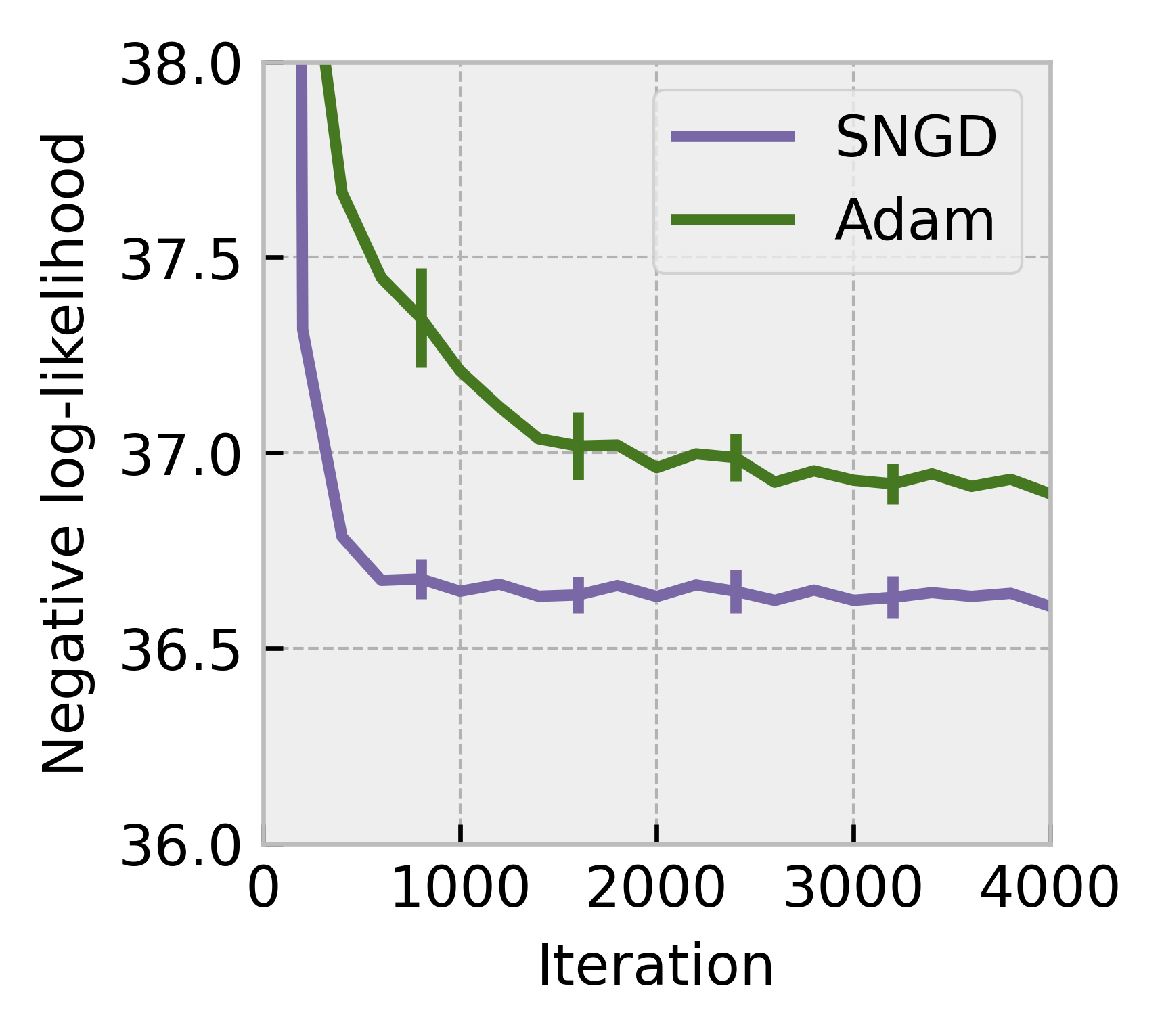}
    \end{subfigure}
    \hfill
    \begin{subfigure}[t]{.49\linewidth}
        \centering
        \includegraphics[height=.9\linewidth]{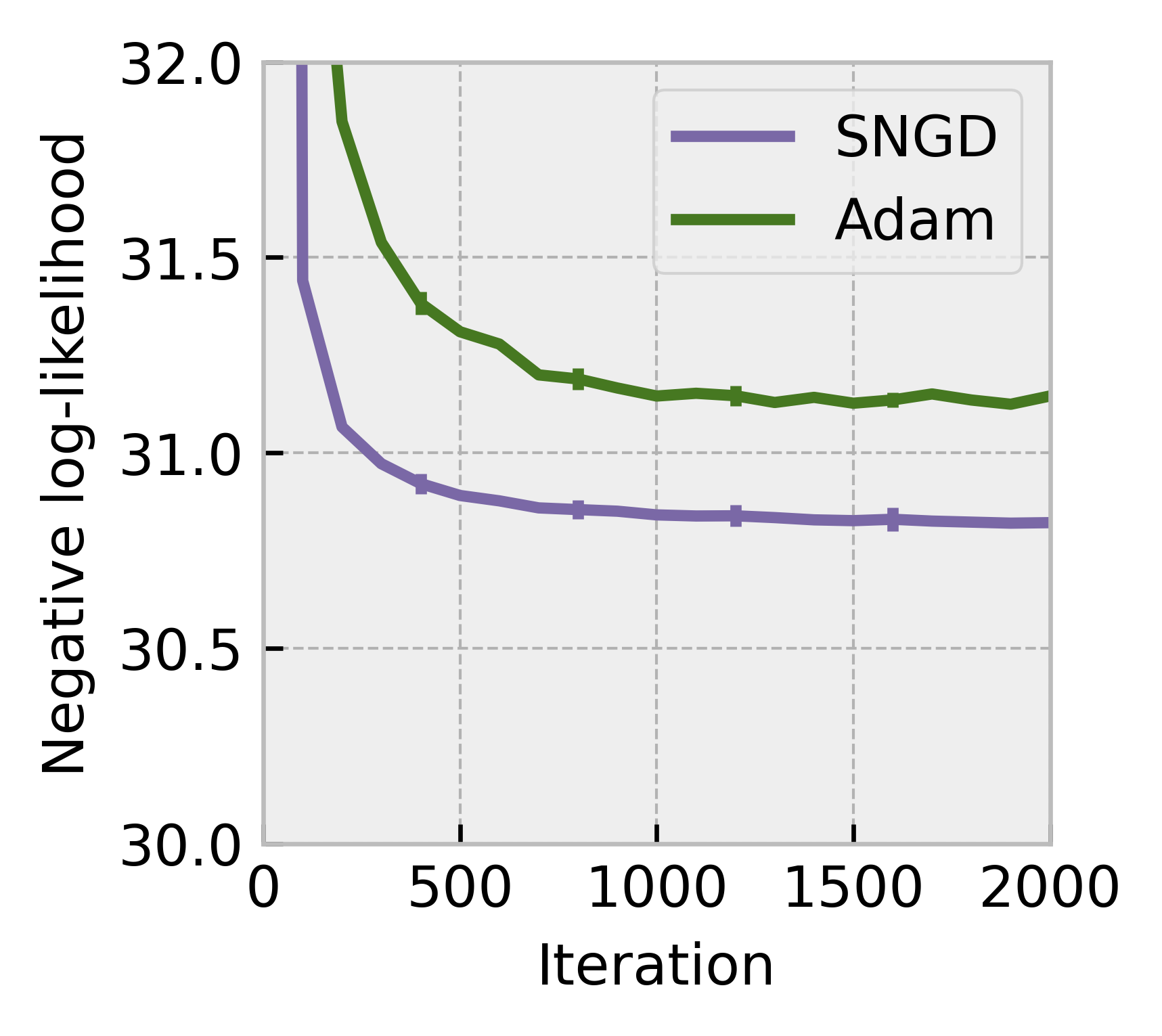}
    \end{subfigure}
    \caption{Training curves for MLE on the miniboone dataset ($n$=32,840, $d$=43) using (left) skew-normal, and (right) skew-$t$ distributions.}
    \label{fig:skewelliptical_mle_trainingcurves}
\end{figure}

The skew-elliptical distributions are a flexible family of unimodal multivariate distributions, allowing for features such as asymmetry and heavy tails \citep{azzalini2013skewnorm}. In this section we consider in particular the skew-normal and skew-$t$ distributions. The skew-normal can be viewed as an asymmetric generalisation of the normal distribution. Similarly, the skew-$t$ can be viewed as an asymmetric generalisation of the Student's $t$ distribution.

The normal distribution is a natural surrogate for these distributions. We used it as such in a number of experiments, mapping its mean and covariance to parameters playing similar roles in the target distributions, with any other parameters captured by $\lambda$, as outlined in Section \ref{subsec:method_auxiliaryparams}.

We performed two tasks on real data in our experiments. In the first, the skew-elliptical distributions were used for density estimation on the UCI miniboone dataset ($n$=32,840, $d$=43) \citep{roe2010miniboone}, fitting their parameters via MLE. In the second, they were used as approximate posteriors in VI for a Bayesian logistic regression model on the UCI covertype dataset \citep{blackard1998covertype}; in this task we used a small subsample of observations in order to retain some uncertainty in the posterior ($n$=500, $d$=53). We also performed a further MLE experiment using high-dimensional synthetic data ($n$=10,000, $d$=1,000).

Note that the numbers of free parameters in these optimisations were $O(d^2)$ due to the covariance-like parameters of the target distributions. The skew-normal distribution had 1,032, 1,537, and 502,500 parameters in the miniboone, covertype, and synthetic experiments, respectively. For the skew-$t$ distribution those numbers were 1,033, 1,537, and 502,501.

Figures \ref{fig:skewelliptical_mle_trainingcurves}, \ref{fig:skewelliptical_vi_trainingcurves} and  \ref{fig:skewelliptical_synthetic_trainingcurves} show training curves from the miniboone MLE, covertype VI, and synthetic MLE experiments, respectively. In all cases SNGD significantly outperformed Adam. In the particular case of the skew-normal VI experiment, it was also possible to apply the natural gradient method of \citet{lin2020mcef} based on \emph{minimal conditional exponential family} (MCEF) distributions, and so we included this as an additional baseline; its performance was virtually identical to that of SNGD, however, we note that this method can itself be viewed as applying SNGD, a point we discuss further in Section \ref{sec:relatedwork}.

\begin{figure}[t]    
    \begin{subfigure}[t]{.49\linewidth}
        \centering
        \includegraphics[height=.9\linewidth]{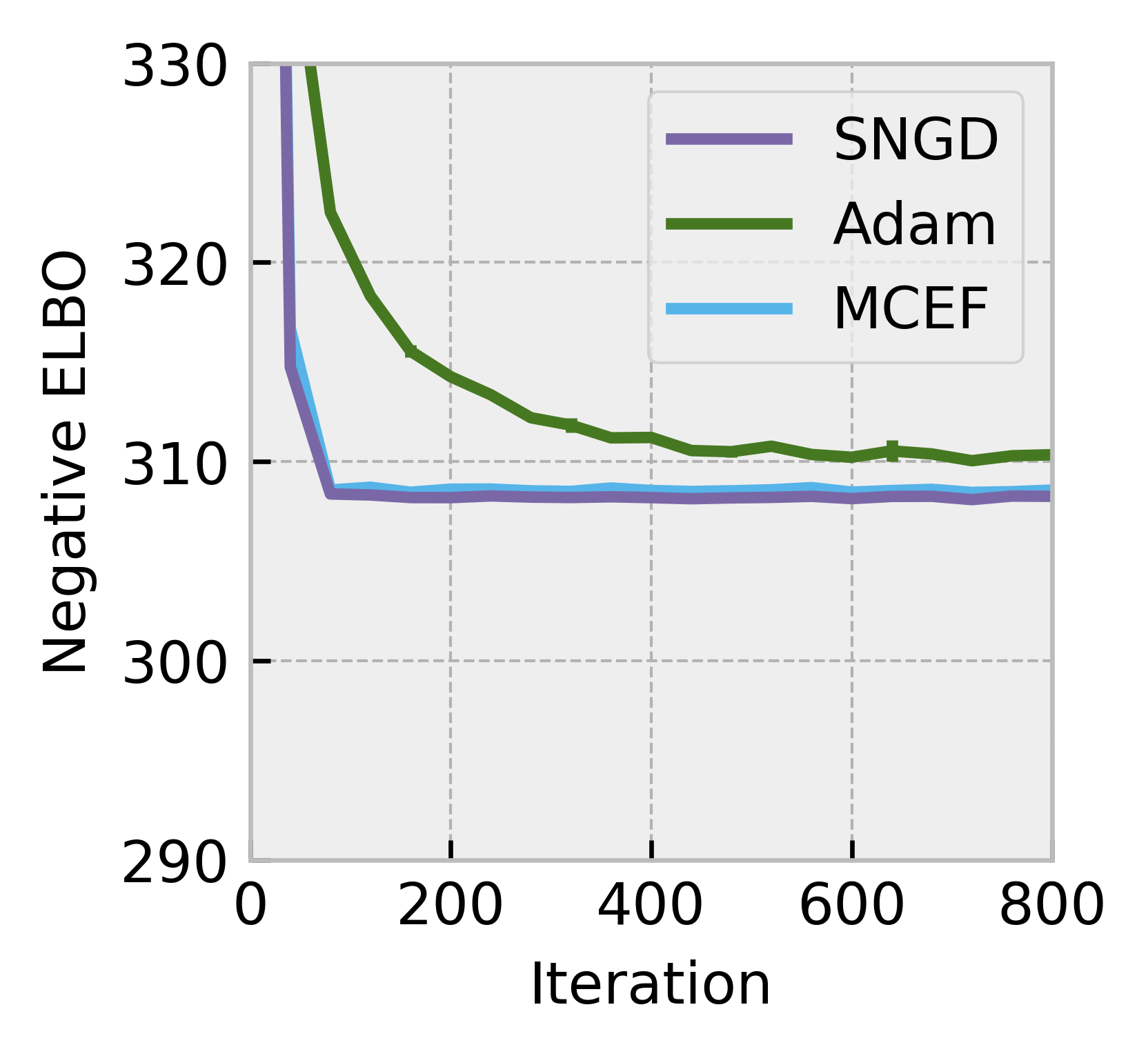}
    \end{subfigure}
    \hfill
    \begin{subfigure}[t]{.49\linewidth}
        \centering
        \includegraphics[height=.9\linewidth]{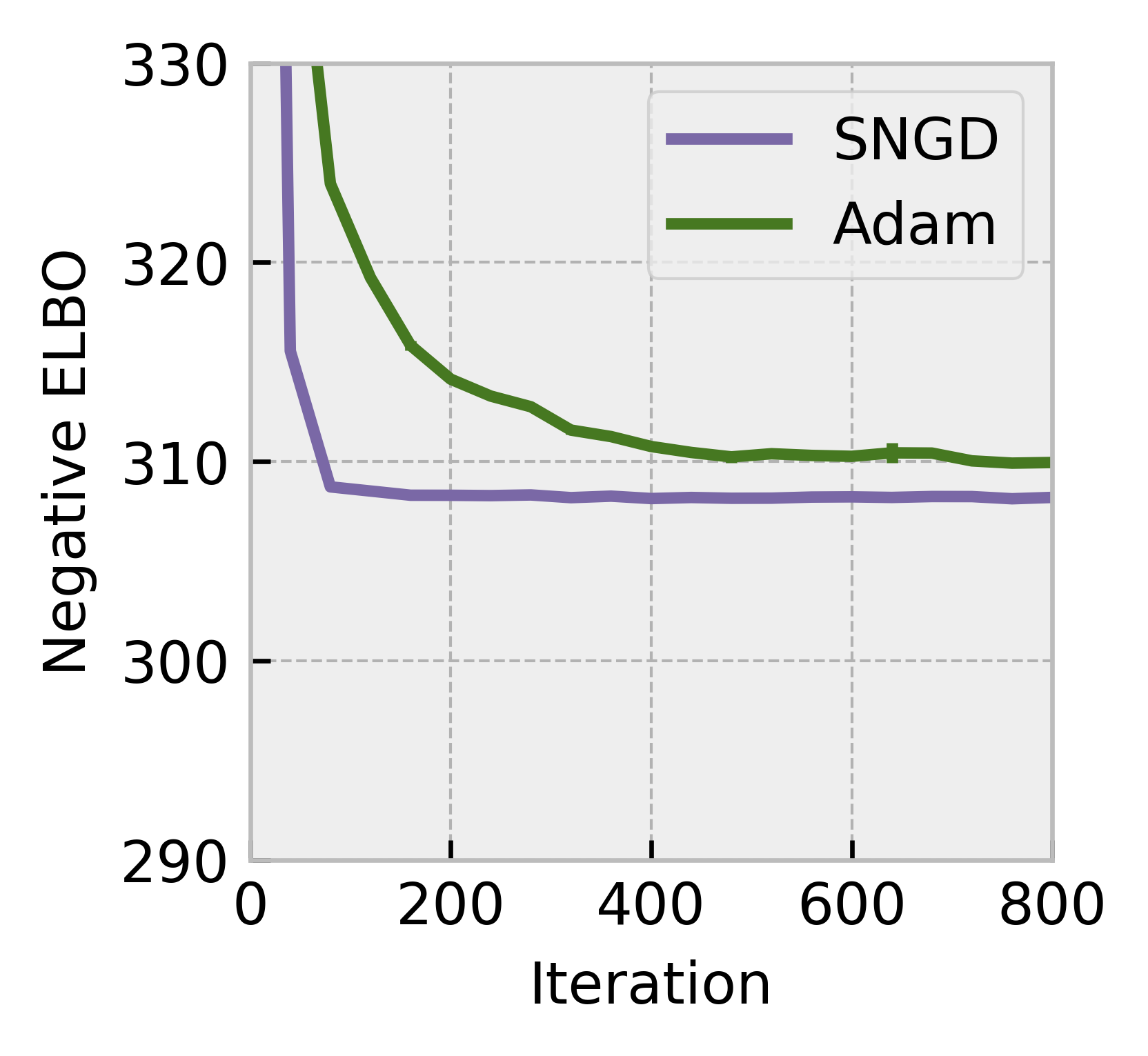}
    \end{subfigure}
    \caption{Training curves for Bayesian logistic regression VI on the covertype dataset ($n$=500, $d$=53) using (left) skew-normal, and (right) skew-$t$ approximations. MCEF corresponds to the natural gradient VI method of \citet{lin2020mcef}.}
    \label{fig:skewelliptical_vi_trainingcurves}
\end{figure}

\begin{figure}[t]    
    \begin{subfigure}[t]{.49\linewidth}
        \centering
        \includegraphics[height=.9\linewidth]{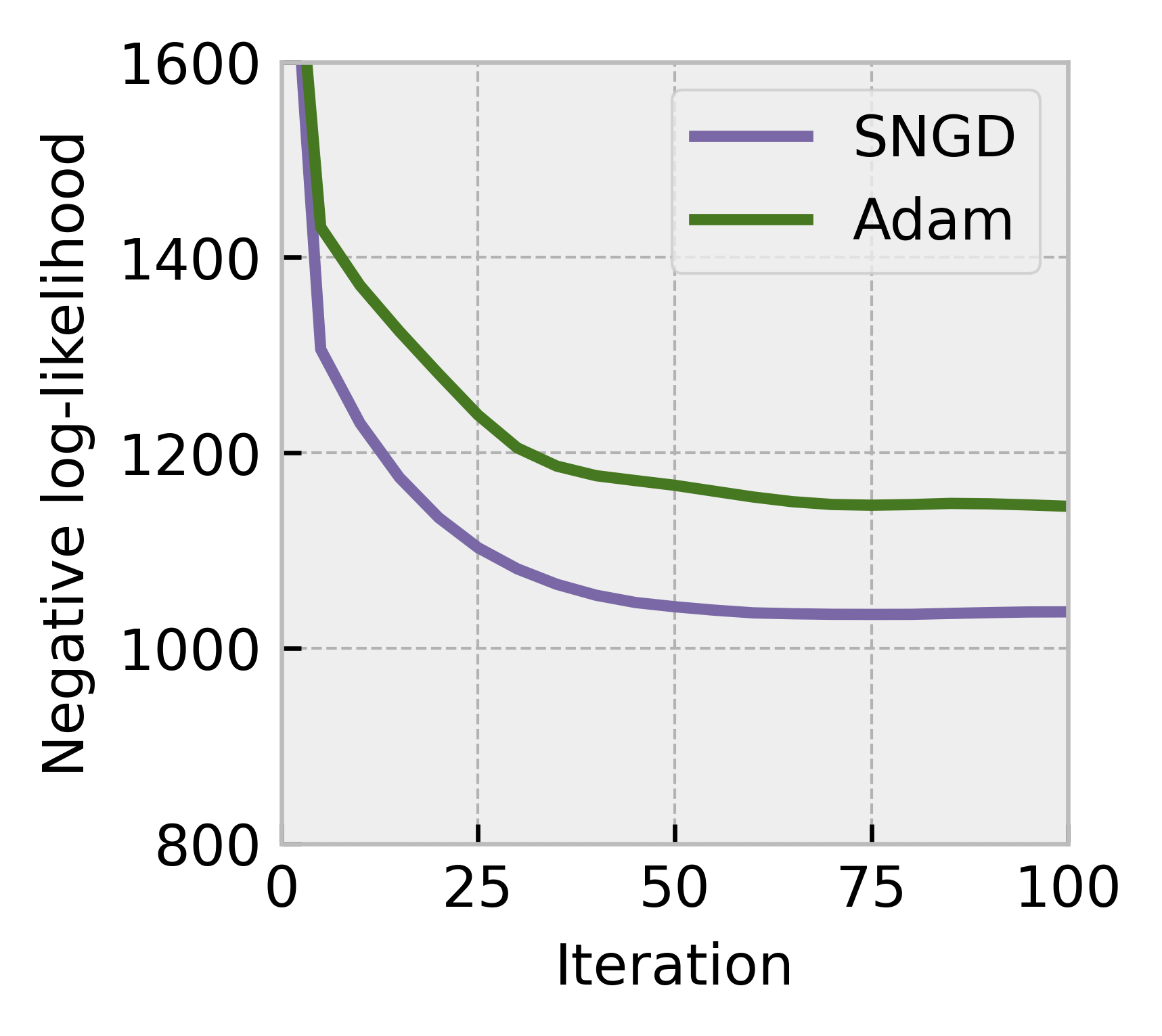}
    \end{subfigure}
    \hfill
    \begin{subfigure}[t]{.49\linewidth}
        \centering
        \includegraphics[height=.9\linewidth]{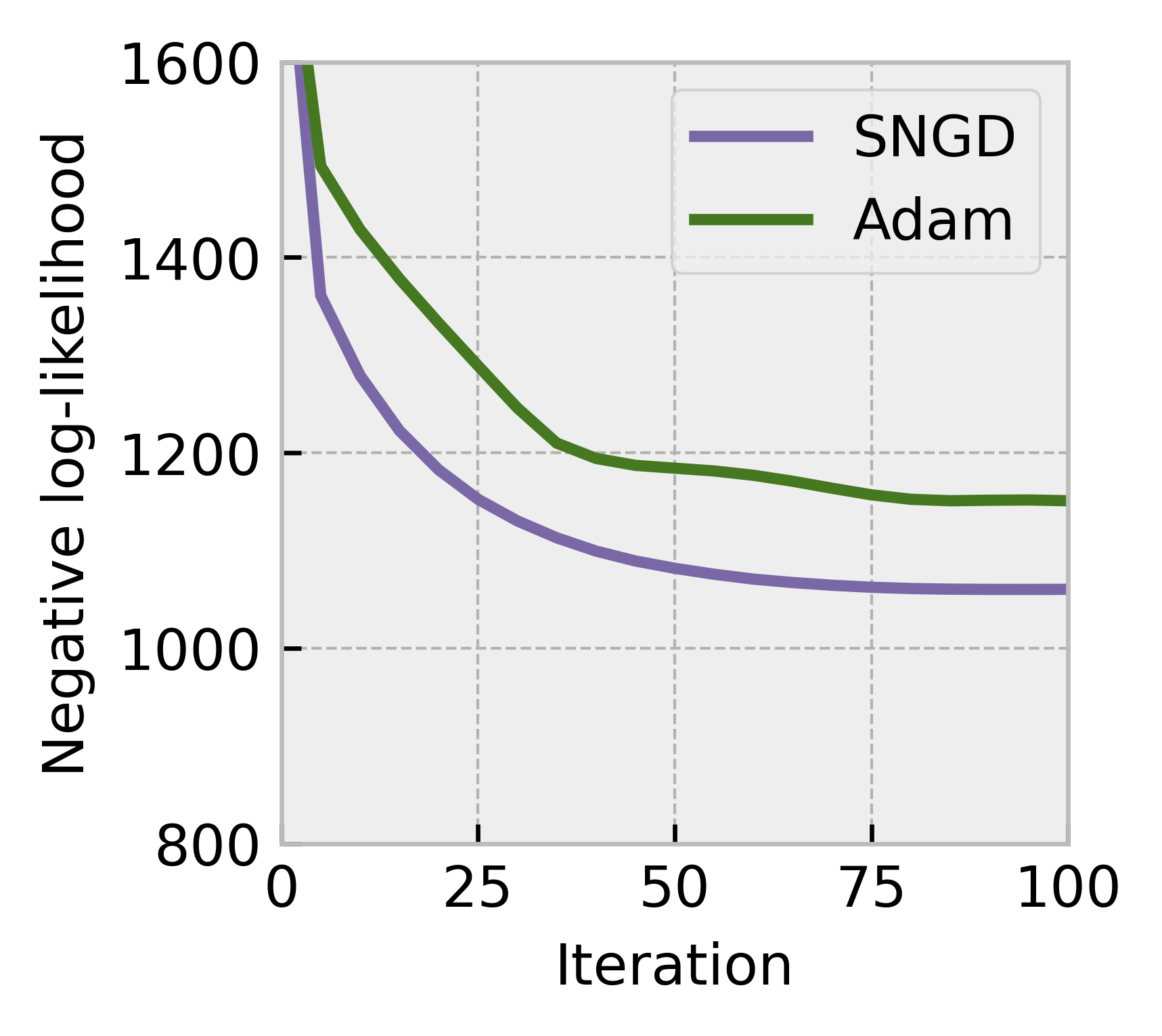}
    \end{subfigure}
    \caption{Training curves for MLE on a synthetic dataset ($n$=10,000, $d$=1,000) using (left) skew-normal, and (right) skew-$t$ distributions.}
    \label{fig:skewelliptical_synthetic_trainingcurves}
\end{figure}

\subsection{Elliptical Copulas}
\label{subsec:experiments_ellipticalcopula}

Copula models define a distribution on the unit hypercube $[0,1]^d$, for which each of the marginals is uniform on $[0,1]$. If $x \in [0,1]^d$ is distributed according to copula $q$, and $y_i = P_i^{-1}(x_i)$, where $P_i^{-1}$ is the inverse CDF of a distribution $q_i$, then $y_i$ will have marginal distribution $q_i$. The $y_i$ will be dependent in general, with dependence structure determined by $q(x)$.

Copula models are widely used in finance for modelling high-dimensional variables, in part because they allow marginal distributions to be modelled separately from the dependence structure \citep{dewick2022copulas}.

The class of elliptical copulas are defined as those copulas which can be used to generate elliptical distributions \citep{frahm2003copulas}. Common examples are the Gaussian and $t$ copulas. Elliptical copulas are parameterised by a correlation matrix $R$, as well as a (possibly empty) set of additional parameters.

A simple way to target elliptical copulas with SNGD is to use a zero-mean normal surrogate, mapping its correlation matrix to $R$. Any additional parameters may then be captured by $\lambda$, as outlined in Section \ref{subsec:method_auxiliaryparams}.

\begin{figure}[t]
    \centering
    \begin{subfigure}[t]{.49\linewidth}
        \centering
        \includegraphics[height=.9\linewidth]{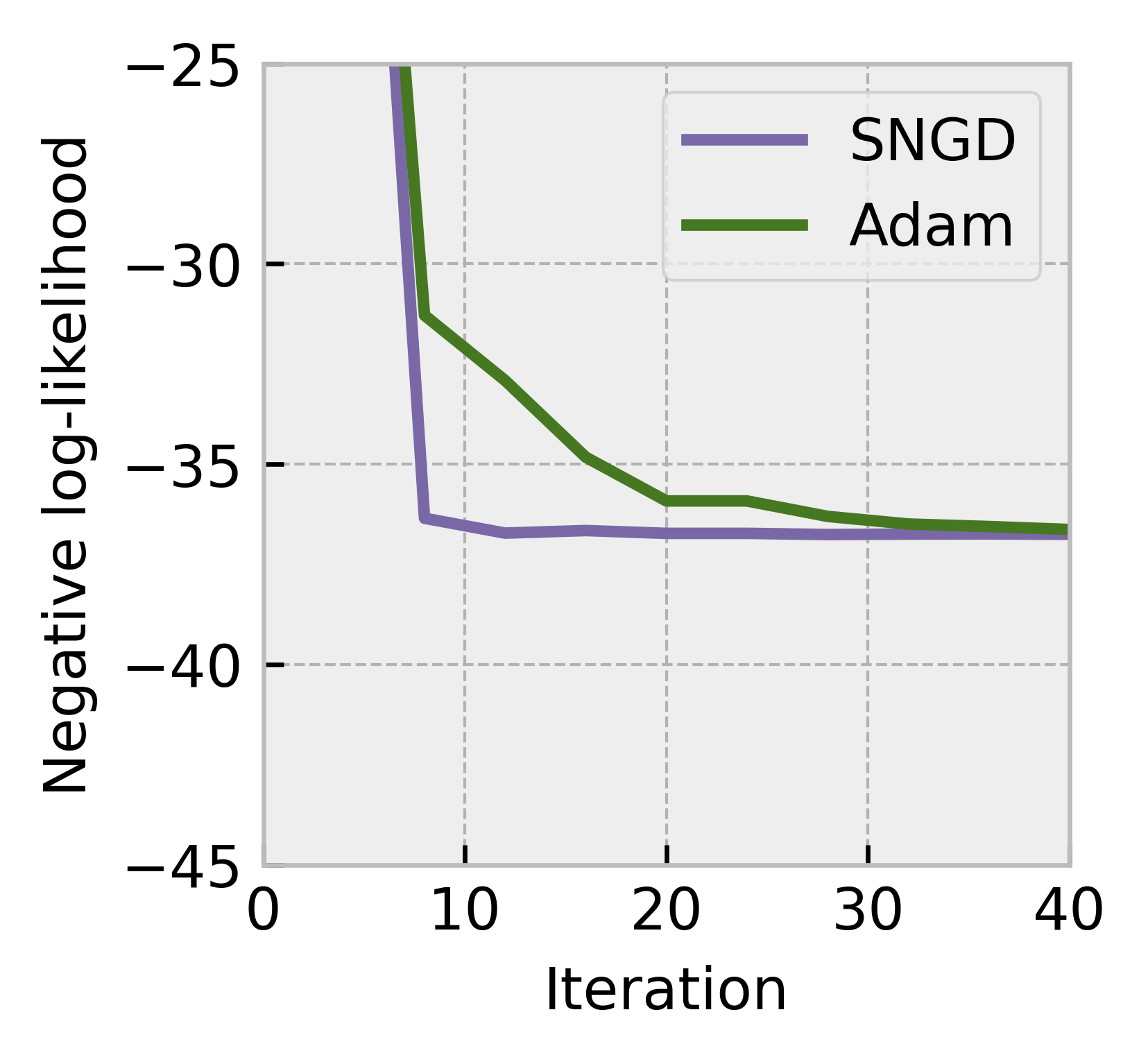}
    \end{subfigure}
    \hfill
    \begin{subfigure}[t]{.49\linewidth}
        \centering
        \includegraphics[height=.9\linewidth]{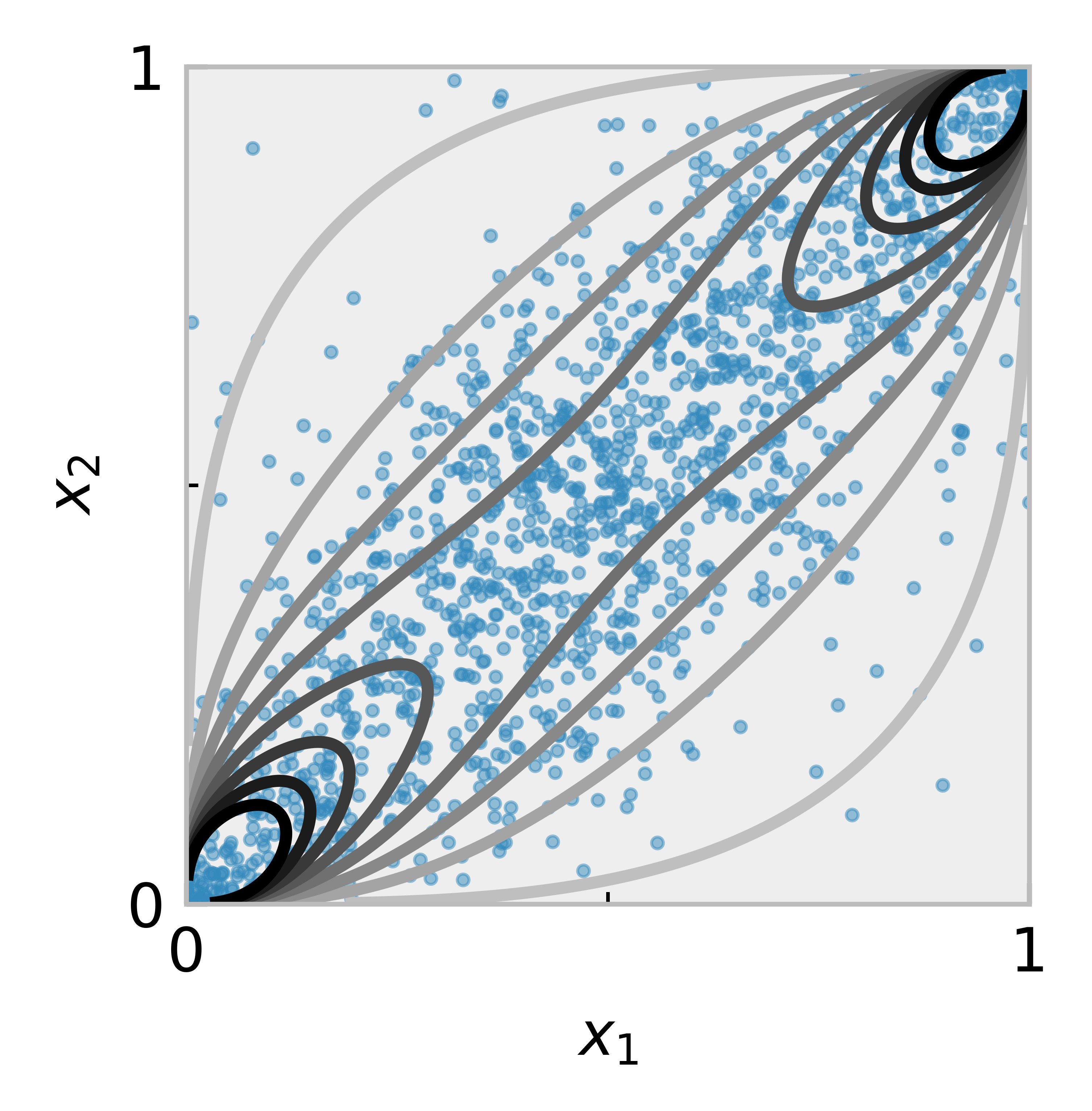}
    \end{subfigure}
    \caption{$t$-copula MLE using 5 years of daily stock return data ($n$=1,515, $d$=93). (left) Training curves. (right) Contours of a 2D marginal density from the fitted copula, overlaid with the training data.}
    \label{fig:copula_experiment}
\end{figure}

We used this approach to perform MLE of the $t$-copula on 5 years of daily stock returns from the FTSE 100 universe ($n$=1,515, $d$=93), with marginals estimated as (univariate) Student's $t$ distributions. In this experiment $q$ had 4,279 free parameters. Figure \ref{fig:copula_experiment} shows that SNGD converged significantly faster than Adam.

\subsection{Mixture Distributions}
\label{subsec:experiments_mixture}

A mixture distribution expresses a complicated density as a convex combination of $k$ simpler densities:
\begin{align}
    q(x) &= \sum_{i=1}^k \pi_i q_{\theta_i}(x), \label{eqn:mixturedistribution}
\end{align}
where $\pi \in \Delta^{k-1}$, and $q_{\theta_i}$ is known as the $i$-th \emph{component} distribution. For each mixture, we can also define a corresponding mixture \emph{model}, a joint distribution, with density
\begin{align}
    q(z, x) &= \pi_z q_{\theta_z}(x), \label{eqn:mixturemodel}
\end{align}
where the mixture component identity, $z \in \{1, ..., k\}$, is treated as a random variable. Note then, that $q(x) = \sum_{i=1}^k q(z = i, x)$. That is, a mixture model has a mixture as its \emph{marginal}. A mixture of EFs is not in general itself an EF. However, it can be shown that the corresponding mixture model (joint distribution) \textit{is} an EF (see Appendix \ref{app:efmixture}).

When an EF is an appropriate surrogate for applying SNGD to a given target family, it is therefore straightforward to consider the extension to mixtures of that family by using an EF mixture \emph{model} as a surrogate. We emphasise that in doing so, the surrogate distribution has support over more variables ($z$ and $x$) than the target distribution (just $x$).

In Section \ref{subsec:experiments_negbin} we demonstrated the use of a gamma surrogate for optimising negative binomial distribution parameters with SNGD. It is therefore straightforward to use a gamma mixture model to target a negative binomial mixture. As an experiment, we performed MLE of a 5 component negative binomial mixture, using a dataset consisting of the number of daily COVID-19 hospital admissions in the UK over a 3 year period ($n$=1,120, $d$=1). Figure \ref{fig:negbinmix_trainingcurves} shows that SNGD significantly outperformed our baselines.

\begin{figure}[t]
    \centering
    \begin{subfigure}[t]{.49\linewidth}
        \centering
        \includegraphics[height=.9\linewidth]{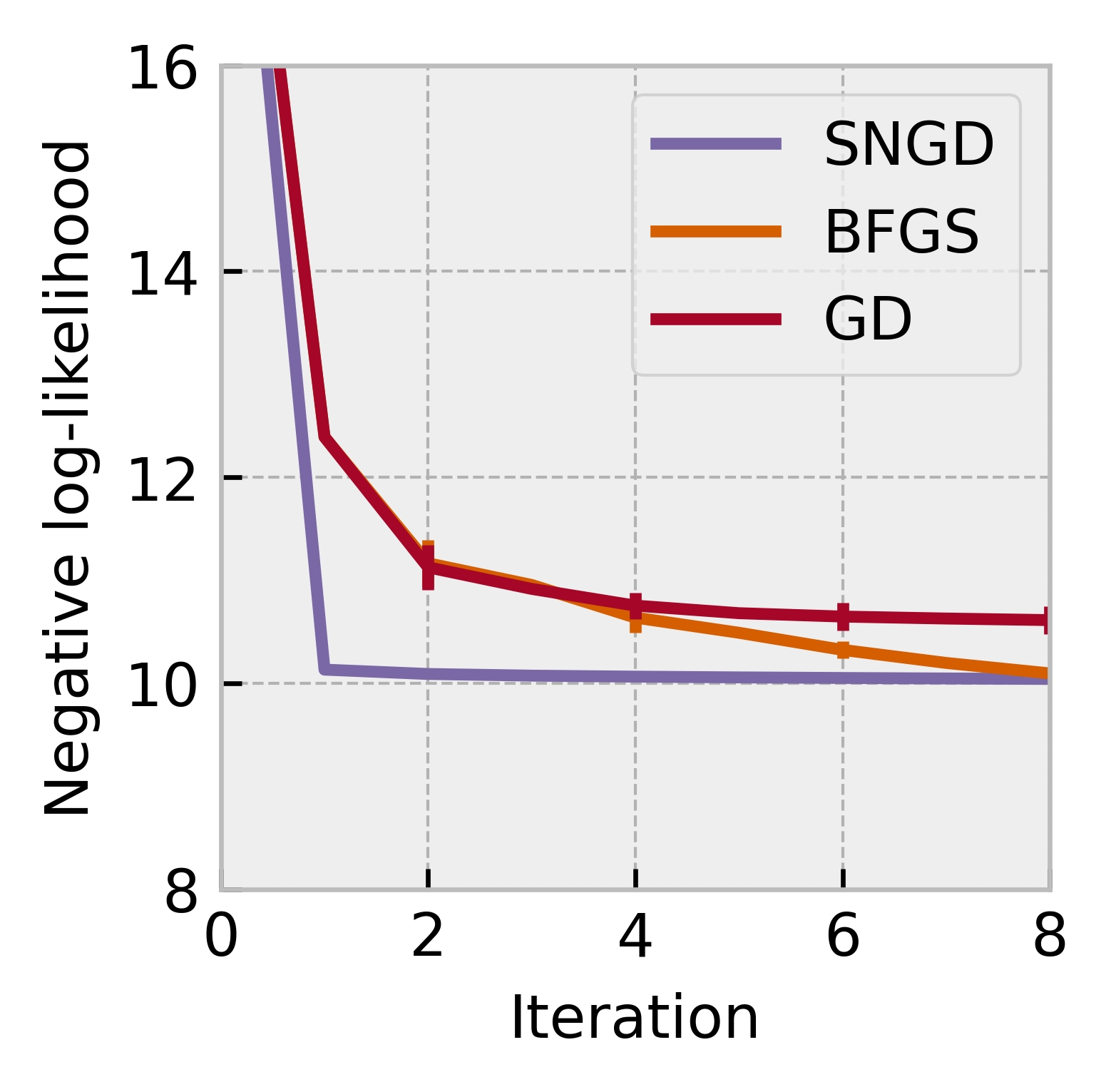}
    \end{subfigure}
    \hfill
    \begin{subfigure}[t]{.49\linewidth}
        \centering
        \includegraphics[height=.9\linewidth]{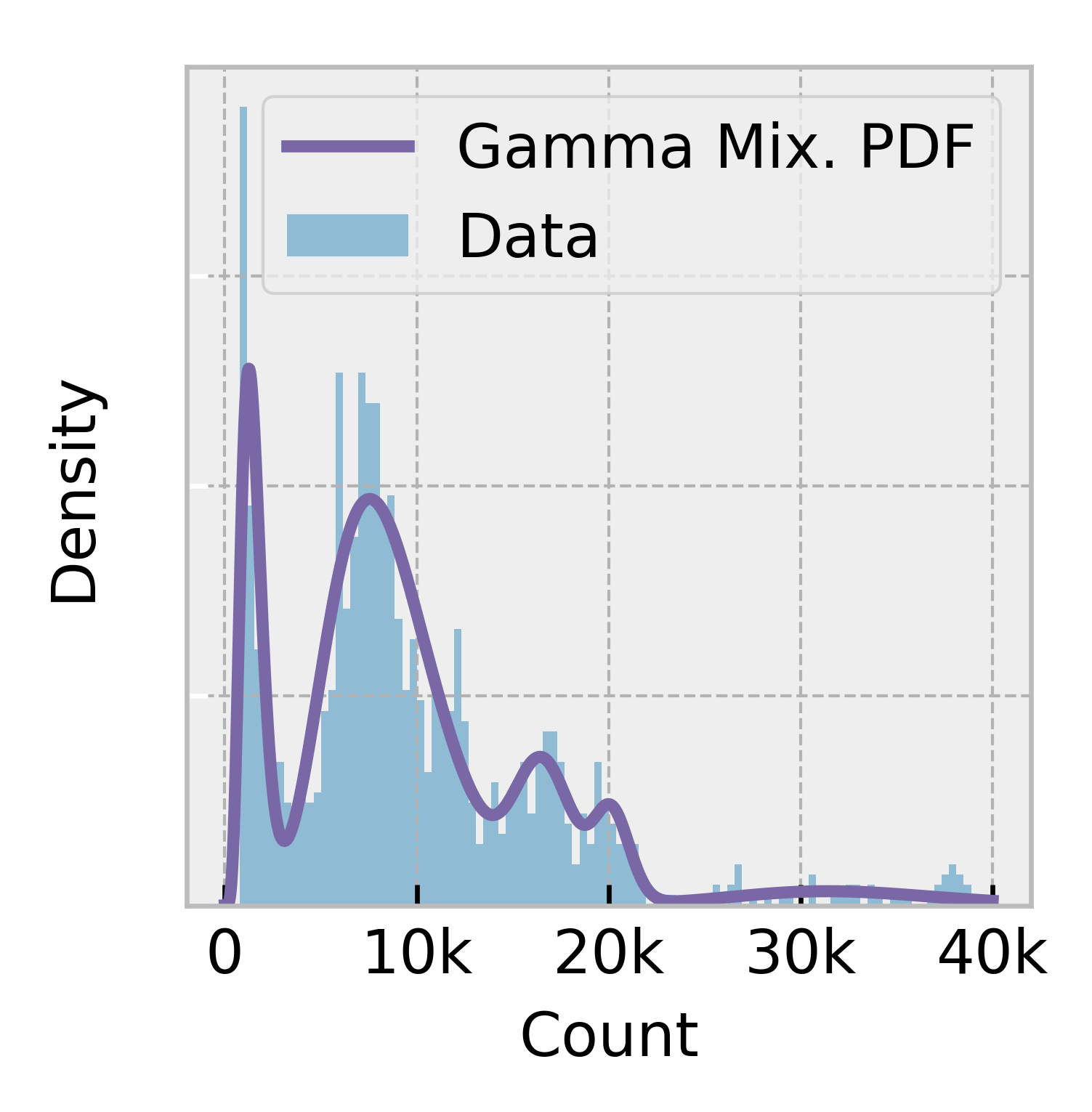}
    \end{subfigure}
    \caption{Negative binomial mixture MLE ($k$=5) using the number of daily COVID hospital admissions in the UK over a 3 year period ($n$=1,120, $d$=1). (left) Training curves. (right) Histogram of the observed counts, overlaid with the PDF of (a marginal of) the gamma mixture model surrogate of SNGD at convergence.}
    \label{fig:negbinmix_trainingcurves}
\end{figure}

As a further example, similarly, we can use a normal mixture model as the surrogate for a skew-normal mixture. Figure \ref{fig:skewnormalmix_posterior} shows the qualitative improvement a skew-normal mixture can offer over a normal mixture on the Bayesian logistic regression VI task of \citet{murphy2013ml} ($n$=60, $d$=2). Figure \ref{fig:skewnormalmix_trainingcurves} shows training curves for this task, as well as for the UCI covertype VI task of Section \ref{subsec:experiments_skewelliptical}. In the latter, $q$ had 4,613 and 7,689 free parameters for $k$=3, 5, respectively.

\begin{figure}[ht]
    \centering
    \addtolength{\tabcolsep}{-4pt}
    \begin{tabular}{cccc}
    {\includegraphics[width=.23\linewidth]{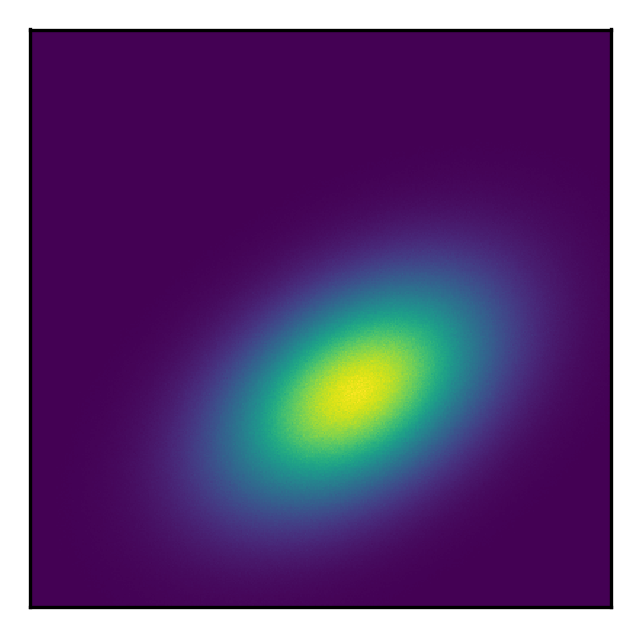}} &
    {\includegraphics[width=.23\linewidth]{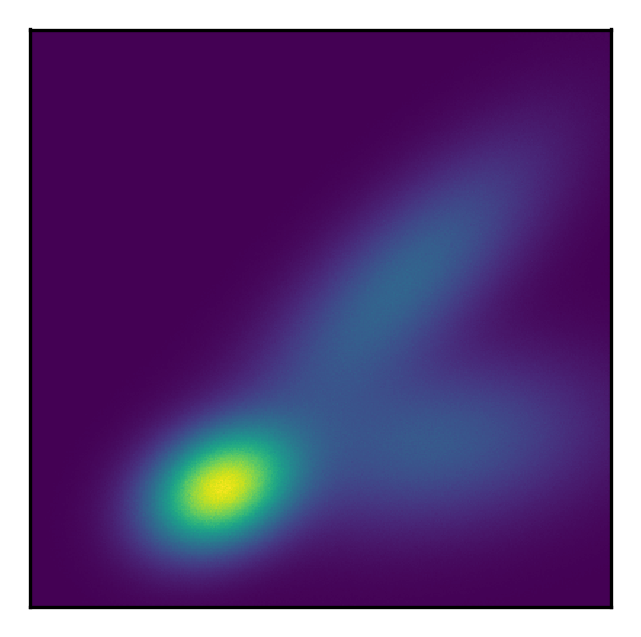}} &
    {\includegraphics[width=.23\linewidth]{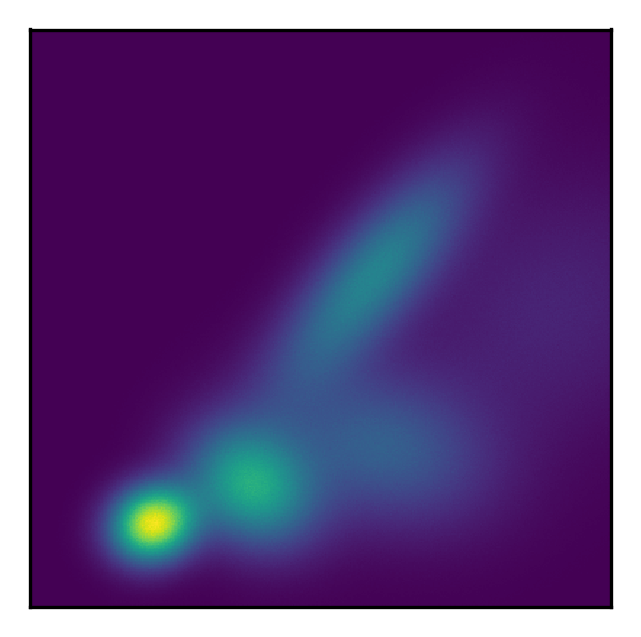}} &
    {\includegraphics[width=.23\linewidth]{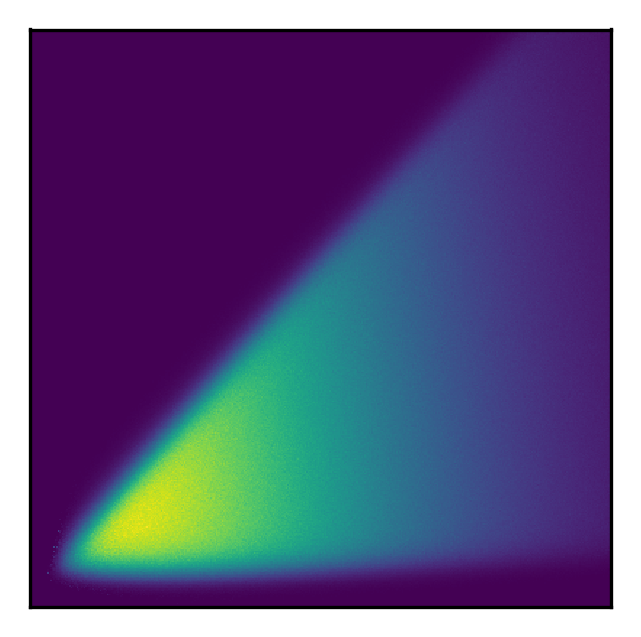}} \\
    {\includegraphics[width=.23\linewidth]{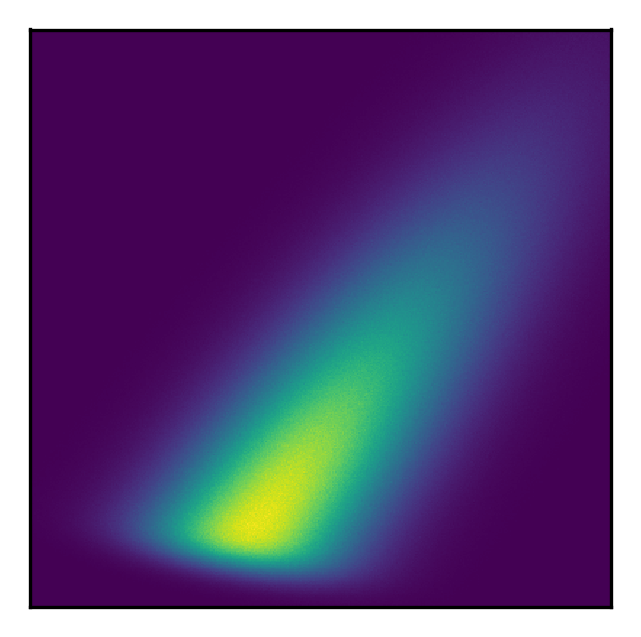}} &
    {\includegraphics[width=.23\linewidth]{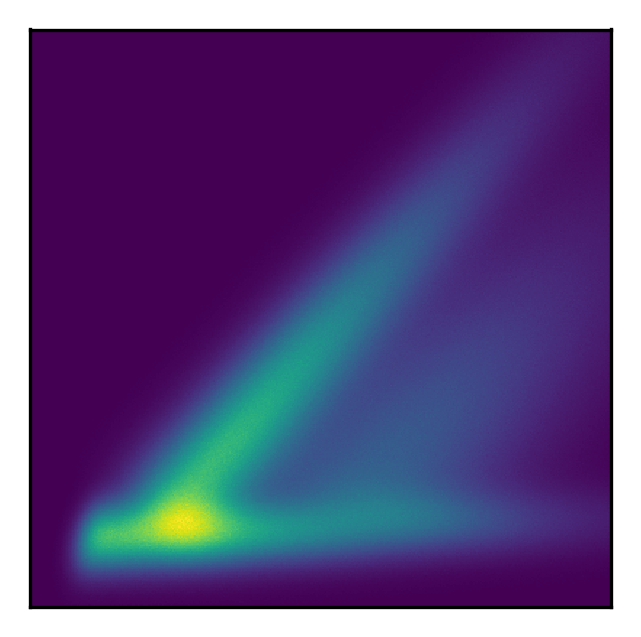}} &
    {\includegraphics[width=.23\linewidth]{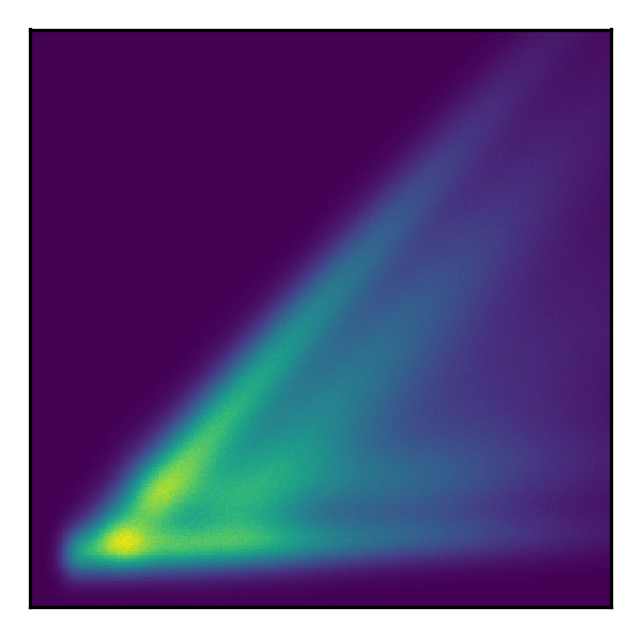}} &
    {\includegraphics[width=.23\linewidth]{figures/girolami2d-mcmc.png}} \\
    $k=1$ & $k=3$ & $k=5$ & GT
    \end{tabular}
    \addtolength{\tabcolsep}{4pt}
    \caption{Comparison of (top) normal and (bottom) skew-normal mixture posterior approximations in the Bayesian logistic regression VI task of \citet{murphy2013ml}. Ground truth (GT) generated using MCMC samples.}
    \label{fig:skewnormalmix_posterior}
\end{figure}

\begin{figure}[ht]
    \centering
    \begin{subfigure}[t]{.49\linewidth}
        \centering
        \includegraphics[height=.9\linewidth]{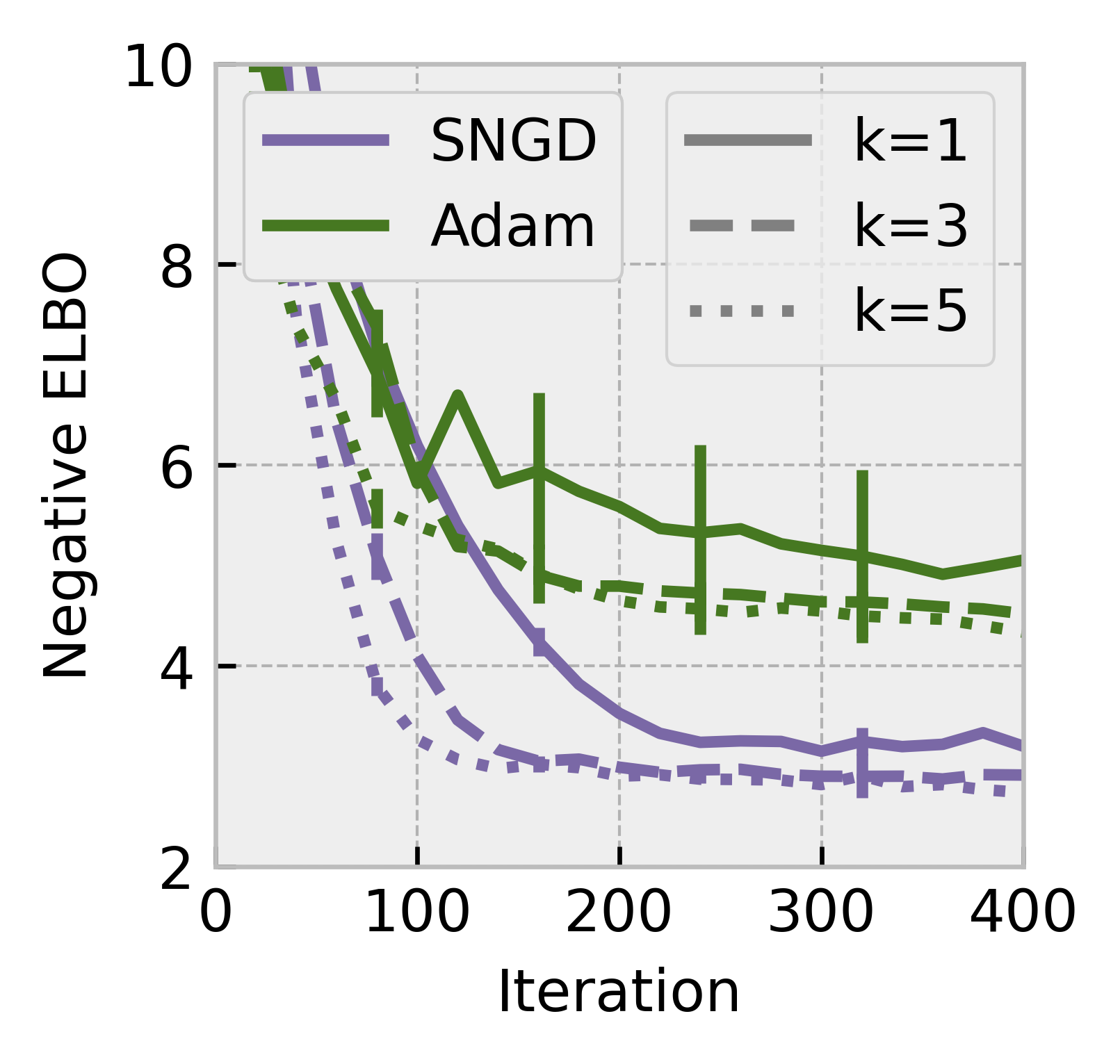}
    \end{subfigure}
    \hfill
    \begin{subfigure}[t]{.49\linewidth}
        \centering
        \includegraphics[height=.9\linewidth]{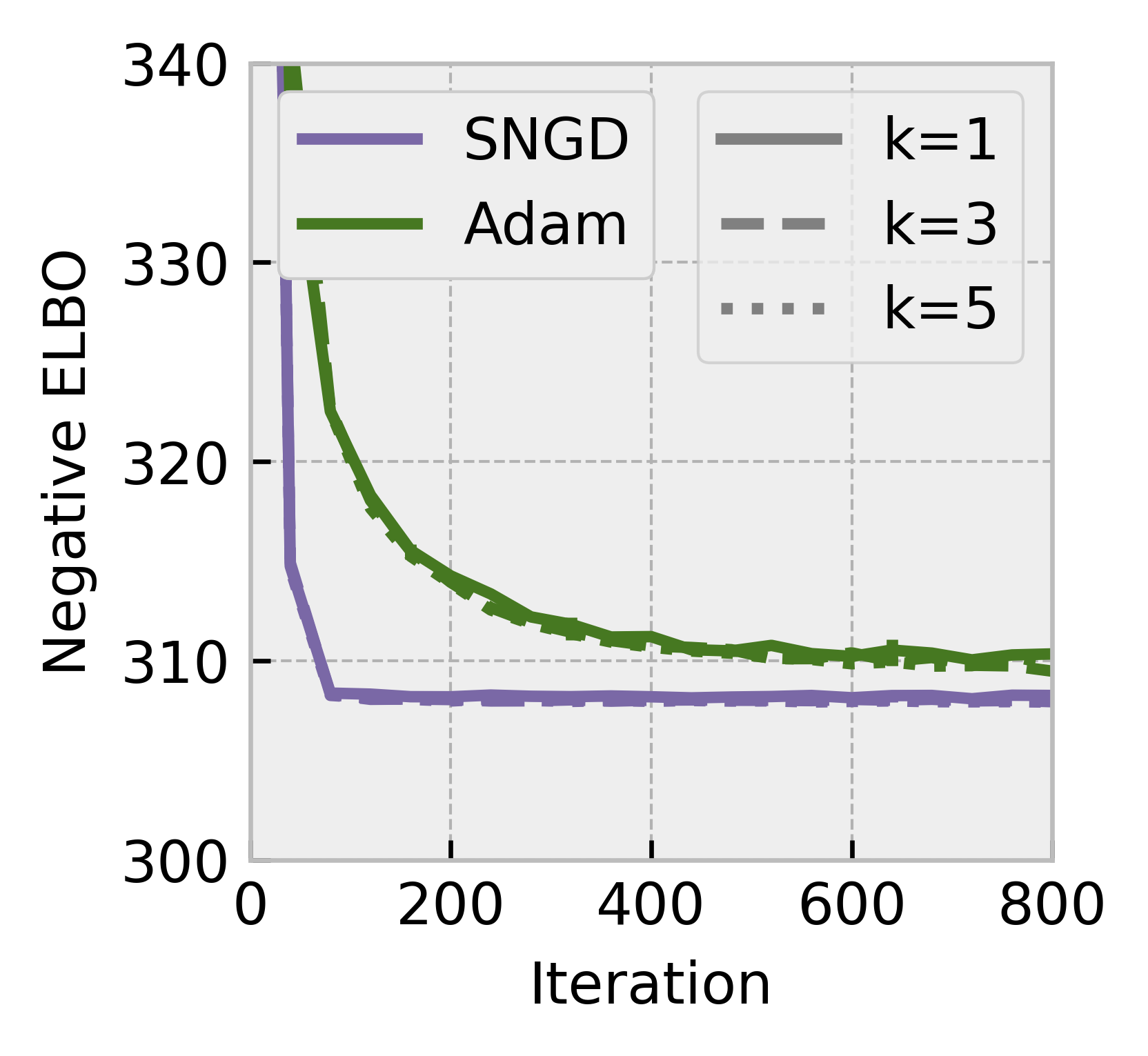}
    \end{subfigure}
    \caption{Training curves for Bayesian logistic regression VI with a skew-normal mixture approximation. (left) the toy dataset of \citet{murphy2013ml} ($n$=60, $d$=2), and (right) UCI covertype ($n$=500, $d$=53). $k$ is the number of mixture components.}
    \label{fig:skewnormalmix_trainingcurves}
\end{figure}

%% file: relatedwork.tex
\label{sec:relatedwork}

Several existing methods can be interpreted as examples of the general technique presented in this paper.

\citet{lin2020mcef} introduced the class of MCEF distributions, and showed that natural gradients with respect to a particular parameterisation of MCEF distributions can be efficiently computed using an identity analogous to \eqref{eqn:ef_natgrad_natparams}. They also showed that several standard distributions can be expressed as \emph{marginals} of MCEF distributions, including the multivariate Student's $t$ and multivariate skew-normal distributions. They then used NGD with respect to the MCEF joint as a surrogate for its marginal in a number of VI tasks. This mismatch between the surrogate and target distributions (our terminology) was not explicitly discussed by \citet{lin2020mcef}, possibly because the two were so closely related (by marginalisation). SNGD is more general in that it makes no such assumption about the relationship between surrogate and target distributions. In particular, it is not restricted to targets that can be expressed as the marginal of a MCEF (or any other) distribution.

The stochastic natural gradient expectation propagation (SNEP) algorithm of \citet{hasenclever2017} solves a saddle-point optimisation problem, with the inner optimisation being over a set of parameters $\{\eta_i\}$ known as the \textit{site parameters}. The site parameters jointly parameterise a set of distributions for which computing natural gradients is not straightforward. SNEP instead treats each $\eta_i$ as if it were the natural parameter vector of a EF distribution in its own right, and performs NGD in the dual (mean parameter) space. These (pseudo)distributions acted as surrogates for the more problematic set of distributions that were the ultimate targets of the optimisation.

\citet{hernandez2014copula} devised a fixed-point iteration scheme for optimising the correlation matrix parameter of an elliptical copula. Interestingly, although their procedure was motivated in an entirely different way, the resulting updates for $R$ are identical to those performed by SNGD when applied in the manner described in Section \ref{subsec:experiments_ellipticalcopula}.

In supervised learning, the goal is to model the conditional density $q(y|x)$ given training data $\mathcal{D} = \{(x_i, y_i)\}_{i=1}^n$. In this setting the Fisher matrix is usually defined with respect to the distribution $p_\mathcal{D}(x)q_\theta(y|x)$, where $p_\mathcal{D}(x)$ is the empirical distribution of $x$ under $\mathcal{D}$. In practice, when the data set is large, the Fisher must be computed with respect to some subset of the training inputs $\mathcal{S}$, such as the current minibatch. This is equivalent to using $p_\mathcal{S}(x)q_\theta(y|x)$ as a surrogate for $p_\mathcal{D}(x)q_\theta(y|x)$, where $p_\mathcal{S}(x)$ is the empirical distribution of $x$ in $\mc{S}$. In the context of supervised deep learning, \citet{ren2019natgrad} showed that exact natural gradients with respect to $p_\mathcal{S}(x)q_\theta(y|x)$ can be computed efficiently when $|\mc{S}|$ is small.

Several works have attempted to handle intractability of the natural gradient in other ways. Some methods are based on structured approximations to the Fisher, that allow inverse vector products to be computed efficiently \citep{heskes2000kfac,martens2015kfac,george2018kfac}. \citet{garcia2023fishleg} instead approximated the inverse Fisher matrix directly by expressing it in terms of the solution to an optimisation problem.
The widely used Adam optimiser employs a moving average of squared gradients in a diagonal approximation to the Fisher \citep{kingma2014adam}, although it has been cautioned that this is more accurately viewed as an approximation to the \textit{empirical} Fisher \citep{kunstner2019empiricalfisher}.

%% file: discussion.tex
\label{sec:discussion}

In this work we proposed a novel technique for optimising functions of probability distribution parameters: reframing the objective as an optimisation with respect to a \textit{surrogate} distribution for which computing natural gradients is easy, and performing optimisation in that space. We found several existing methods that can be interpreted as applying this technique, and proposed a new method based on EF surrogates. We demonstrated that our method is able to converge rapidly on a variety of MLE and VI tasks. We believe our method can be readily applied to distributions outside of the set of examples given here. We also expect that new methods can be found by applying the more general technique that motivated our method.

The main limitation of our method is the need to find a suitable surrogate and reparameterisation for a given target distribution. In some cases there is an obvious candidate, such as when an EF distribution could serve as an approximation for the target, or when the target can be viewed as an EF that has been warped or transformed in some way. However, for most of the examples in this paper, the target and surrogate distributions do not even have support over the same space (see Table \ref{tbl:models}). Our choices of surrogates were largely guided by the intuition that $\tilde\theta$ should have similar local effects on KL divergences in $q$ than it does in $\tilde q$. However, finding more prescriptive or systematic methods for choosing surrogates would be beneficial, and may be an interesting direction for future work.

%% file: acknowledgements.tex
We would like to thank Runa Eschenhagen, Jihao Andreas Lin, and James Townsend for providing valuable feedback. Jonathan So is supported by the University of Cambridge Harding Distinguished Postgraduate Scholars Programme. Richard E. Turner is supported by Google, Amazon, ARM, Improbable and EPSRC grant EP/T005386/1.

%% file: appendix_gradientnotation.tex
\label{app:gradientnotation}
In this section we introduce our notation for gradients and related quantities. We largely follow the notation of \citet{bertsekas1997nonlinear}, with one addition. For function $f: \mathbb{R}^m \rightarrow \mathbb{R}$, the gradient at $x$, assuming all partial derivatives exist, is given by
\begin{align}
    \nabla f(x) &= \left[{\pdv{f(x)}{x_1}}, {\pdv{f(x)}{x_2}}, \dots, {\pdv{f(x)}{x_m}}\right]^\top.
\end{align}
Note, in particular, that this implies $\nabla f(g(y))$ is the gradient of $f$ evaluated at $x = g(y)$. The Hessian of $f$, denoted $\nabla^2 f(x)$, is the matrix with entries given by
\begin{align}
    \left[\nabla^2 f(x)\right]_{ij} &= \pdv{f(x)}{x_i}{x_j}.
\end{align}

If $f: \mathbb{R}^{m+n}$ is a function of $x\in\mathbb{R}^m$ and $y\in\mathbb{R}^n$, then
\begin{align}
    \nabla_x f(x, y) &= \left[{\pdv{f(x, y)}{x_1}}, {\pdv{f(x, y)}{x_2}}, \dots, {\pdv{f(x, y)}{x_m}}\right]^\top
\end{align}
with $\nabla_y f(x, y)$ defined similarly.

When $f: \mathbb{R}^m \rightarrow \mathbb{R}^n$ is a vector valued function, the \emph{gradient matrix} of $f$, denoted $\nabla f(x)$, is the transpose of the Jacobian of $f$. That is, the matrix with $i$-th column equal to the gradient of $f_i$, the $i$-th component of $f$:
\begin{align}
    \nabla f(x) &= \left[\nabla f_1(x)\ ...\ \nabla f_n(x)\right].
\end{align}
Finally (our addition), where $\nabla$ is immediately followed by a bracketed expression, we use this to denote the gradient of an anonymous function, with definition given by the bracketed expression, and gradient taken with respect to the subscript, e.g.
\begin{align}
    \nabla_x\bs(2f(x)\bs) &= 2\nabla f(x),
\end{align}
and in these cases (only), evaluation of the gradient at, e.g. $x = g(y)$, is denoted
\begin{align}
    \nabla_x\bs(2f(x)\bs)\big\rvert_{g(y)}.
\end{align}

%% file: appendix_validity.tex
\label{app:validity}

In this appendix we show that under certain conditions, optimising $\tilde f$ is equivalent to optimising $f$ in the following sense: finding a local minimiser of $\tilde f$ also gives us a local minimiser for $f$ (Proposition \ref{thm:criticalpoints_localminimum}), and all local minima of $f$ are attainable through $\tilde f$ (Proposition \ref{thm:attanability}). Furthermore, we show that $\tilde f$ does not have any non-strict saddle points that are not also present at the corresponding points in $f$ (Proposition \ref{thm:criticalpoints_strictsaddlepoint}, with support from Propositions \ref{thm:criticalpoints_localmaximum} and \ref{thm:criticalpoints_saddlepoint}).

Note that the results derived here are more general than those that are summarised in Section \ref{subsec:method_equivalence}. The conditions stated in that section are sufficient to cover all of the examples appearing in this paper. In this appendix we use notation mirroring the method of Section \ref{subsec:method_sngd}, but the results apply equally to the extension of Section \ref{subsec:method_auxiliaryparams} if we replace $\tilde\Theta$ with the product manifold $\tilde\Theta\times\Lambda$ below.

Let $\Theta$ and $\tilde{\Theta}$ be differentiable manifolds of dimension $i$ and $j$ respectively, where $i \le j$.  Let $g: \tilde{\Theta} \rightarrow \Theta$ be a twice differentiable submersion on $\tilde{\theta}$, with $g(\tilde\Theta) = \Theta$. Let $f : \Theta \rightarrow \mbb{R}$ be twice continuously differentiable, and define $\tilde{f} = f \circ g$.

\begin{proposition}
    \label{thm:criticalpoints_localminimum}
    $\tilde f$ has a local minimum at $\tilde{\theta}^*$ if and only if $f$ has a local minimum at $\theta^* = g(\tilde{\theta}^*)$
\end{proposition}
\begin{proof}
    First we prove the statement: $f$ has a local minimum at $\theta^* = g(\tilde\theta^*) \Rightarrow \tilde f$ has a local minimum at $\tilde\theta^*$.

    From the definition of a local minimum, there exists a neighbourhood of $\theta^*$, $\mc V$ such that
    \begin{align}
        f(\theta^*) &\le f(\theta)\ \forall\ \theta \in \mc V.
    \end{align}
    Because $g$ is a continuous map, the preimage of $\mc V$, $\mc U = g^{-1}(\mc V)$, is an open set which by construction contains $\tilde\theta^*$, and is therefore a neighbourhood of $\tilde\theta^*$.  The result then follows
    \begin{align}\begin{split}
        f(\theta^*) &\le f(\theta)\ \forall\ \theta \in \mc V \\
        \Rightarrow  f\bs(g(\tilde\theta^*)\bs) &\le f\bs(g(\tilde\theta)\bs)\ \forall\ \tilde\theta \in \mc U \\
        \Rightarrow \tilde{f}(\tilde\theta^*) &\le \tilde{f}(\tilde\theta)\ \forall\ \tilde\theta \in \mc U
    \end{split}\end{align}

    Finally, we prove the statement: $\tilde f$ has a local minimum at $\tilde\theta^* \Rightarrow f$ has a local minimum at $\theta^* = g(\tilde\theta^*)$.

    From the definition of a local minimum, there exists a neighbourhood of $\tilde\theta^*$, $\mc U$ such that
\begin{align}
    \tilde{f}(\tilde\theta^*) &\le \tilde{f}(\tilde\theta)\ \forall\ \tilde\theta \in \mc  U
\end{align}
By assumption $g$ is a submersion and therefore a continuous open map.  Because $\mc U$ is an open set containing $\tilde\theta^*$, $\mc V = g(\mc U)$ must also be an open set containing $\theta^*$.  Finally, then
\begin{align}\begin{split}
    \tilde{f}(\tilde\theta^*) &\le \tilde{f}(\tilde\theta)\ \forall\ \tilde\theta \in \mc 
 U \\
    \Rightarrow f\bs(g(\tilde\theta^*)\bs) &\le f\bs(g(\tilde\theta)\bs) \forall\ \tilde\theta \in \mc U \\
    \Rightarrow f(\theta^*) &\le f(\theta) \forall\ \theta \in \mc V
\end{split}\end{align}
\end{proof}

\begin{proposition}
    \label{thm:attanability}
    For any local minimiser $\theta^*$ of $f$, $\exists\tilde{\theta}^* \in \tilde{\Theta}$ that is a local minimiser of $\tilde{f}$ s.t. $\theta^* = g(\tilde{\theta}^*)$
\end{proposition}
\begin{proof}
    $g(\tilde{\Theta}) = \Theta \Rightarrow \exists\tilde{\theta}^* \in \tilde{\Theta}$ s.t. $\theta^* = g(\tilde{\theta}^*)$. The result then follows from Proposition \ref{thm:criticalpoints_localminimum}.
\end{proof}

\begin{proposition}
    $\tilde f$ has a local maximum at $\tilde{\theta}$ if and only if $f$ has a local maximum at $\theta = g(\tilde{\theta})$
    \label{thm:criticalpoints_localmaximum}
\end{proposition}
\begin{proof}
    This follows from Proposition \ref{thm:criticalpoints_localminimum} by symmetric arguments.
\end{proof}

\begin{proposition}
    \label{thm:criticalpoints_saddlepoint}
    $\tilde f$ has a saddle point at $\tilde{\theta}$ if and only if $f$ has a saddle point at $\theta = g(\tilde{\theta})$
\end{proposition}
\begin{proof}
    First, we prove the following statement: $\tilde f$ has a critical point at $\tilde{\theta} \Leftrightarrow f$ has a critical point at $\theta = g(\tilde{\theta})$.

    Let $\tilde\theta$ and $\theta = g(\tilde\theta)$ be represented as co-ordinates for some charts at those points, with $f$, $g$, $\tilde f$ defined similarly. Then the statement about critical points can be expressed as follows
    \begin{align}
        \nabla\tilde f(\tilde\theta) = \mathbf{0}\ \Leftrightarrow\ \nabla f\bs(g(\tilde\theta)\bs) = \mathbf{0}
    \end{align}
    where $\mathbf{0}$ is a vector of zeros. To show that $\nabla\tilde f(\tilde\theta) = \mathbf{0}\ \Rightarrow\ \nabla f\bs(g(\tilde\theta)\bs) = \mathbf{0}$, we have
    \begin{align}
        \nabla\tilde f(\tilde\theta) &= \nabla g(\tilde\theta)\nabla f\bs(g(\tilde\theta)\bs) 
    \end{align}
    where $\nabla g(\tilde\theta)$ is the transposed Jacobian of $g$ (see Appendix \ref{app:gradientnotation} for an explanation of this notation) and has full column rank due to $g$ being a submersion, and so:
    \begin{align}
        \nabla g(\tilde\theta)\nabla f\bs(g(\tilde\theta)\bs) = \mathbf{0} &\Rightarrow \nabla f\bs(g(\tilde\theta)\bs) = \mathbf{0}.
    \end{align}
    For the other direction, $\nabla f\bs(g(\tilde\theta)\bs) = \mathbf{0} \Rightarrow \nabla\tilde f(\tilde\theta) = \mathbf{0}$, trivially,
    \begin{align}\begin{split}
        \nabla\tilde f(\tilde\theta) &= \nabla g(\tilde\theta)\nabla f\bs(g(\tilde\theta)\bs)   \\
        &= \nabla g(\tilde\theta)\mathbf{0} \\
        &= \mathbf{0}.
    \end{split}\end{align}

    Having proven correspondence of critical points, we can proceed to prove the proposition statement in two parts.
    
    First, if $\tilde f$ has a saddle point at $\tilde{\theta}$, then $f$ must have a critical point at $\theta = g(\tilde{\theta})$.  However, from Propositions \ref{thm:criticalpoints_localminimum} and \ref{thm:criticalpoints_localmaximum}, we know that this cannot be a local minimum or local maximum, and hence must be a saddle point.

    Second, in the other direction, if $f$ has a saddle point at $\theta = g(\tilde{\theta})$, then $\tilde f$ has a critical point at $\tilde{\theta}$, which by similar reasoning must also be a saddle point.
\end{proof}

\begin{proposition}
    \label{thm:criticalpoints_strictsaddlepoint}
    If $\tilde f$ has a non-strict saddle point at $\tilde{\theta}$ then $f$ has a non-strict saddle point at $\theta = g(\tilde{\theta})$
\end{proposition}
\begin{proof}
    From Proposition \ref{thm:criticalpoints_saddlepoint}, we know that if $\tilde f$ has a non-strict saddle point at $\tilde\theta$, then $f$ must have a saddle point at $\theta = g(\tilde\theta)$. It remains to be proven that the saddle point at $\theta$ must be a non-strict saddle point. We do this by contradiction. Let us assume that the saddle point of $f$ at $\theta = g(\tilde\theta)$ is strict.

    Let $\tilde\theta$, $\theta$ be represented as co-ordinates for some charts at those points, with $f$, $g$, $\tilde f$ defined similarly. Furthermore, let $H_f = \nabla^2\tilde f(\tilde\theta)$, $H_{\tilde f} = \nabla^2 f\bs(g(\tilde\theta)\bs)$ be the Hessians of $f$ and $\tilde f$, at $\tilde\theta$ and $\theta = g(\tilde\theta)$, respectively. Then,
    \begin{align}\begin{split}
        H_{\tilde f} &= \nabla g(\tilde\theta)H_f \nabla g(\tilde\theta)^\top + \sum_{k=1}^j [\nabla f\bs(g(\tilde\theta)\bs)]_k \nabla^2g_k(\tilde\theta) \\
        &= \nabla g(\tilde\theta) H_f \nabla g(\tilde\theta)^\top
    \end{split}\end{align}
    where $\nabla^2g_k(\tilde\theta)$ is the Hessian of the $k$-th component of $g$. The second equality follows from $\theta$ being a critical point of $f$.

    A strict saddle point is a saddle point for which there is at least one direction of strictly negative curvature, and so given the assumption that $\theta = g(\tilde\theta)$ is a strict saddle point of $f$, $\exists v \in \mbb{R}^i$ such that $v^\top H_f v < 0$. Let $\tilde v = A^\top v$, where $A$ is any left inverse of $\nabla g(\tilde\theta)$, then
    \begin{align}
    \begin{split}
        \tilde v^\top H_{\tilde f}\tilde v &= (A^\top v)^\top \nabla g(\tilde\theta) H_f\nabla g(\tilde\theta)^\top(A^\top v) \\
        &= v^\top H_f v \end{split}\\
        &< 0,
    \end{align}
    implying that $\tilde f(\tilde\theta)$ is a strict saddle point, a contradiction, and so we conclude that the saddle point at $f(\theta)$ cannot be strict.
\end{proof}

%% file: appendix_algorithms.tex
\subsection{SNGD WITH EF SURROGATES}
\label{app:algorithms_sngd}

In Algorithm \ref{alg:sngd}, restated with line numbers below, we provide pseudocode for an implementation of SNGD when $\tilde q$ is an EF distribution, and $\tilde\theta$ are either natural or mean parameters of that family. We assume the existence of an autodiff operator $\texttt{grad}$, which takes as input a real-valued function, and returns another function for computing its gradient. Note that when $f$ cannot be computed deterministically, such as in VI or minibatch settings, we assume $\texttt{grad}$ returns a function that provides unbiased stochastic estimates of the gradient. In our VI experiments, where gradients had to be taken through samples, we used the reparameterisation trick \citep{kingma2014adam}, applied to the target distribution.

We also assume the existence of an overloaded function $\texttt{dualparams}$, which converts from mean to natural parameters of the EF, or vice-versa, depending on the type of its argument.\footnote{This is purely for convenience, as it allow us to describe a single implementation handling both parameterisations.} That is, in the notation of Section \ref{subsec:background_ef}, $\texttt{dualparams}$ resolves to either $\mu(.)$ or $\eta(.)$ as appropriate. We note that each $\texttt{dualparams}$ pair only needs to be defined once for each EF, and is not dependent on e.g. the target distribution or loss function, meaning that if these are supplied as part of a software library, the end user is only required to supply $f$, $g$ and $\tilde\theta_0$. Algorithm \ref{alg:sngd} also assumes a given step size schedule, but can easily be extended to incorporate line search or other methods for choosing $\epsilon_t$.

\newcounter{tempAlgorithmCounter}
\setcounter{tempAlgorithmCounter}{\value{algorithm}}
\setcounter{algorithm}{0}

\begin{algorithm}[ht]
    \caption{SNGD with EF surrogate}
    \begin{algorithmic}[1]
        \Require objective $f : \Theta \rightarrow \mbb{R}$
        \Require parameter mapping $g : \tilde{\Theta} \rightarrow \Theta$
        \Require initial surrogate parameters $\tilde\theta_0 \in \tilde\Theta$
        \Require step size schedule $\{\epsilon_t \in \mbb{R}_+ : t = 0, 1, ... \}$
        \State $\tilde f_\text{dual}(.) := f(g(\texttt{dualparams}(.)))$ 
        \State $t \leftarrow 0$
        \While{not converged}
            \State $\tilde{\nabla} \leftarrow \texttt{grad}[\tilde f_\text{dual}](\texttt{dualparams}(\tilde\theta_t))$
            \State $\tilde\theta_{t+1} \leftarrow \tilde\theta_t - \epsilon_t\tilde{\nabla}$
            \State $t \leftarrow t + 1$
        \EndWhile \\
        \Return{$g(\tilde\theta_t)$}
    \end{algorithmic}
\end{algorithm}

\setcounter{algorithm}{\value{tempAlgorithmCounter}}

On line $1$ of Algorithm \ref{alg:sngd} we define a reparameterisation of $\tilde f$ in terms of the dual parameters.  That is, if $\tilde\theta$ are natural parameters, $\tilde f_\text{dual}$ is a function of mean parameters, and vice-versa. Line $4$ computes the natural gradient, given by equation \eqref{eqn:ef_natgrad_natparams} or \eqref{eqn:ef_natgrad_meanparams}, using automatic differentiation of the function $\tilde f_\text{dual}$. Note that \emph{both} overloads of $\texttt{dualparams}$ are called: one inside the auto-differentiated function $\tilde f_\text{dual}$, and the other outside (to compute its argument). It is often possible for the inner conversion to be elided; for example, the user can supply a function $g_\text{dual}$ (instead of $g$) that can perform the map from dual parameters to $\theta$ directly, more efficiently than the composition $g \circ \texttt{dualparams}$. For example, this is often the case when $f$ depends on covariance-like parameters, and the composition $g \circ \texttt{dualparams}$ would otherwise involve inverting a matrix twice (a no-op).

If $\tilde\theta$ are mean parameters of $\tilde q$, then the computational overhead of Algorithm \ref{alg:sngd} (relative to GD in $f$) is approximately $\texttt{cost}[\eta] + 3\times\texttt{cost}[g\circ\mu]$, where $\texttt{cost}$ returns the cost of its function argument. Similarly, if $\tilde\theta$ are natural parameters, then the overhead will be approximately $\texttt{cost}[\mu] + 3\times\texttt{cost}[g\circ\eta]$. The factor of 3 in the second term in each case results from taking gradients through $g \circ \texttt{dualparams}$; however, as discussed above, it is often the case the composition is almost zero cost, in which case this term can be largely eliminated by implementing $g_\text{dual}$ directly.

We have assumed here that we can take gradients through either $\mu(.)$ or $\eta(.)$ efficiently, depending on the choice of parameterisation. For $\mu(.)$ this is true by assumption for tractable families. For some families the `reverse' map $\eta(.)$, is not available in closed form, but can be efficiently computed using an iterative optimisation procedure \citep{minka2000dirichlet,minka2002gamma}. In such cases, we can use implicit differentiation techniques to efficiently compute gradients \citep{christianson1994fixedpoints,blondel2022implicitdiff}.

\subsection{SNGD WITH EF SURROGATES AND AUXILIARY PARAMETERS}
\label{app:algorithms_sngd_auxiliaryparams}

In Algorithm \ref{alg:sngd_additionalparams} we provide pseudocode for the extension of Section \ref{subsec:method_auxiliaryparams}, in which we augment $\tilde\theta$ with auxiliary parameters $\lambda$.  $\tilde\theta$ are optimised using natural gradients, whereas $\lambda$ are optimised with standard first-order methods.  Algorithm \ref{alg:sngd_additionalparams} uses GD with a fixed learning rate schedule for $\lambda$, but the extension to any first-order optimiser is straightforward.

The structure of Algorithm \ref{alg:sngd_additionalparams} is largely the same as that of Algorithm \ref{alg:sngd}. One notable difference is that $\tilde f_\text{dual}$ now has 2 arguments, and so the call to $\texttt{grad}$ on line 4 returns a function that returns a 2-tuple of gradients, one for each argument.

\begin{algorithm}[ht]
    \caption{SNGD with EF surrogate and auxiliary parameters}
    \label{alg:sngd_additionalparams}
    \begin{algorithmic}[1]
        \Require objective $f : \Theta \rightarrow \mbb{R}$
        \Require parameter mapping $g : \tilde{\Theta} \rightarrow \Theta$
        \Require initial surrogate parameters $\tilde\theta_0 \in \tilde\Theta$
        \Require initial auxiliary parameters $\lambda_0 \in \Lambda$
        \Require $\theta$ step size schedule $\{\epsilon_t \in \mbb{R}_+ : t = 0, 1, ... \}$
        \Require $\lambda$ step size schedule $\{\varepsilon_t \in \mbb{R}_+ : t = 0, 1, ... \}$
        \State $\tilde f_\text{dual}(\_1, \_2) := f(g(\texttt{dualparams}(\_1), \_2))$
        \State $t \leftarrow 0$
        \While{not converged}
            \State $(\tilde{\nabla}_\theta, \nabla_\lambda) \ \leftarrow \texttt{grad}[\tilde f_\text{dual}](\texttt{dualparams}(\theta_t), \lambda_t)$
            \State $\theta_{t+1} \leftarrow \theta_t - \epsilon_t\tilde{\nabla}_\theta$
            \State $\lambda_{t+1} \leftarrow \lambda_t - \varepsilon_t\nabla_\lambda$
            \State $t \leftarrow t + 1$
        \EndWhile \\
        \Return{$g(\theta_t, \lambda_t)$}
        \end{algorithmic}
\end{algorithm}

%% file: appendix_experimentdetails.tex
\label{app:experiment_details}

In this appendix we provide additional details about the experiments presented in the main paper. The code for these experiments is available at \url{https://github.com/cambridge-mlg/sngd}.

We repeated all experiments with 10 different random seeds. In all cases this led to different random parameter initialisations. For VI experiments, this also seeded randomness in the Monte Carlo samples, and for experiments with minibatching, it also seeded randomness in the minibatch sampling. The mean and standard errors as displayed in the training curves and Pareto frontier plots were computed over the 10 runs. For VI experiments we used the reparameterisation trick to estimate gradients for each method \citep{kingma2014aevb}.

In experiments that were small scale and had objectives that could be computed deterministically, namely MLE experiments using the sheep and COVID datasets, we used exact line search to determine step sizes for each of the methods being tested. This allowed us to compare the search direction of each method without confounding results with choice of hyperparameter settings.

For all other experiments we chose hyperparameters using a grid search. The training curves in the main paper correspond to the `best' hyperparameter settings for each method. The best setting was considered to be that which had the best average (across time steps) worst case (over random seeds) value of the evaluation metric (negative elbo or negative log-likelihood as appropriate). Although this choice is somewhat arbitrary, we found that it consistently chose settings with training curves that closely resembled those that we considered best for each method. In Appendix \ref{app:experiment_results} we provide Pareto frontier plots which incorporate \emph{all} of the hyperparameter settings tried, which qualitatively are very similar to the training curves in the main paper.

For SNGD we chose $\tilde\theta$ to be mean parameters of $\tilde q$ for all MLE tasks, and natural parameters for all VI tasks. These choices were motivated by the results of Appendix \ref{app:efnatgrads}, and we found them to consistently perform better than alternatives.

All experiments were executed on a 76-core Dell PowerEdge C6520 server, with 256GiB RAM, and dual Intel Xeon Platinum 8368Q (Ice Lake) 2.60GHz processors. Each individual optimisation run was locked to a single dedicated core. Implementations were written in JAX \citep{frostig2018jax}.

Next we provide details specific to each task featured in the experiments. In the list below, $n$ denotes the number of training observations, and $d$ denotes the dimensionality of the distribution that is being optimised in the task.

\paragraph{Sheep}($n$=82, $d$=1) Taken from a seminal work on the negative binomial distribution by \citet{fisher1941negbin}, this dataset consists of the number of ticks observed on each member of a population of sheep. The task for this dataset was MLE of the negative binomial distribution.
\paragraph{UCI miniboone}($n$=32,840, $d$=43) Taken from the MiniBooNE experiment at Fermilab, this dataset consists of a number of readings that can be used to classify observations as either electron or muon neutrinos \citep{roe2010miniboone}.\footnote{Licensed under a Creative Commons Attribution 4.0 International (CC BY 4.0) license.} We follow the pre-processing of \citet{papamakarios2017}. Using this dataset we performed MLE of skew-normal and skew-$t$ distributions. We used 32,840 of the observations for training, and the remaining $3,648$ observations for evaluation. We used a minibatch size of $256$ for each method.
\paragraph{UCI covertype}($n$=500, $d$=53) This dataset classifies the forest cover type of $581,024$ pixels, based on 53 cartographic variables \citep{blackard1998covertype}.\footnote{Licensed under a Creative Commons Attribution 4.0 International (CC BY 4.0) license.} We used the `binary scale' preprocessing of \citet{chang2011libsvm}. The task was to perform VI in a Bayesian logistic regression model, with regularisation parameter 1.0. We found that using anything close to the full number of observations resulted in degenerate posteriors with virtually zero uncertainty, obviating the need for variational inference, and so we used a randomly chosen subset of $500$ observations for our experiments. We used skew-normal, skew-$t$ and skew-normal mixture distributions as approximate posteriors. All methods used $20$ Monte Carlo samples for training and $1000$ for evaluation.
\paragraph{Synthetic skew-normal}($n$=10,000, $d$=1,000) We generated synthetic data from a $d$-dimensional skew-normal distribution with parameters $(\xi, \Omega, \eta)$, where $\xi$ and $\eta$ were drawn from $\mathcal{N}(\mathbf{0}, I_d)$. We chose $\Omega = d^{-1}W^\top W + 10^{-4}I_d$, where the components of $W \in \mathbb{R}^{d\times d}$ were drawn from independent standard normal distributions. The task was to perform MLE of a skew-normal distribution on this dataset.
\paragraph{Synthetic skew-$t$}($n$=10,000, $d$=1,000) We generated synthetic data from a $d$-dimensional skew-$t$ distribution with parameters $(\xi, \Omega, \eta, \nu)$. $\xi$, $\Omega$ and $\eta$ were generated in the same manner as the synthetic skew-normal task as detailed above, with $\nu$ chosen to be 10. The task was to perform MLE of a skew-$t$ distribution on this dataset.
\paragraph{FTSE 100 stock returns}($n$=1,515, $d$=93) This dataset consists of daily stock price returns from 2017/01/01 to 2022/12/31 for the subset of FTSE 100 stocks that were members of the index during the entire period.\footnote{Downloaded from the Bloomberg Terminal.} We first fitted univariate Student's $t$ distributions to each dimension independently, and then transformed the data by converting each observed variable to its quantile value under the marginal distribution. The task was then MLE of a $t$ copula on the quantile values.
\paragraph{COVID hospital admissions}($n$=1,120, $d$=1) This dataset consists of the number of daily COVID hospital admissions in the UK from 2020/4/1 to 2023/5/1.\footnote{Downloaded from \url{https://coronavirus.data.gov.uk} and licensed under the Open Government License v3.0.} The task for this dataset was MLE of a 5-component negative binomial mixture.
\paragraph{Synthetic 2D logistic regression}($n$=30, $d$=2) This synthetic logistic regression dataset was generated using the same procedure as \citet{murphy2013ml}. The task was VI in a Bayesian logistic regression model, with regularisation parameter 1. We used a skew-normal mixture approximate posterior, with $20$ Monte Carlo samples for training and $1,000$ for evaluation.

Now we provide details particular to the distributions being optimised in the tasks above. Note that the correspondence between tasks (above) and target distributions (below) is many to many: some target distributions were applied to more than one task, and some tasks were used for several target distributions. The parameter mappings used for SNGD are given in Table \ref{tbl:mapppings} in Appendix \ref{app:mappings}; we do not repeat them here unless additional explanation is required. When the parameter mappings make use of the auxiliary parameter extension of Section \ref{subsec:method_auxiliaryparams}, we used Adam to optimise $\lambda$.

With our baseline methods, when a target distribution required a positive definite covariance matrix parameter, we tried two different parameterisations; covariance square-root (e.g. $\Sigma = W^\top W)$ and precision square-root (e.g. $\Sigma^{-1} = W^\top W)$ \citep{salimbeni2018natgrad}. This parameterisation choice was determined by a hyperparameter which we included in our grid search. Further parameterisation details are given below.

\paragraph{Negative binomial} The PMF of the negative binomial with parameters $\theta = (r, s)$ is given by \eqref{eqn:negbin_pmf}. We initialised negative binomial parameters for all methods by drawing $s \sim \text{Uniform}(0.05, 0.95)$, $r \sim \text{Gamma}(6.25, 1.25)$.  With GD, BFGS and NGD we used a log parameterisation of $r$, and a logit parameterisation of $s$. For SNGD we chose $\tilde\theta$ to be mean parameters of $\tilde q$.
In order to ensure that $g(\tilde\theta) \in \Theta\ \forall\ \tilde\theta \in \tilde\Theta$, we require $\beta < 1\ \forall\ \tilde\theta \in \tilde\Theta$, therefore we chose $\tilde\Theta$ as the subset of $\mathcal{M}$ (the mean domain of $\tilde q$) for which $\beta < 1$. It can be shown that this remains an open convex set.
\paragraph{Skew-normal} The PDF of the multivariate skew-normal with parameters $\theta = (\xi, \Omega, \eta)$ is given by $q_\theta(x) = 2\mathcal{N}_d(x;\xi, \Omega)\Phi(\eta^\top(x-\xi))$, where $\mathcal{N}_d(.)$ is the PDF of the $d$-dimensional normal distribution, and $\Phi(.)$ is the CDF of the standard normal distribution.\footnote{Equation (5.1) of \citet{azzalini2013skewnorm}.} We used random initialisations of $\xi \sim \mathcal{N}(\mathbf{0}, 0.01^2I_d)$, $\eta \sim \mathcal{N}(\mathbf{0}, 0.01^2I_d)$, with $\Omega$ initialised to $I_d$.
\paragraph{Skew-$t$} The PDF of the multivariate skew-$t$ with parameters $\theta = (\xi, \Omega, \eta, \nu)$ is given by $q_\theta(x) = 2\mathcal{T}_d(x; \xi, \Omega, \nu)\Psi(\alpha^\top x\sqrt{(\nu + d)/(\nu + x^\top\bar\Omega^{-1} x)}; \nu + d)$, where $\mathcal{T}_d$ is the PDF of the $d$-dimensional Student's $t$, $\Psi(.)$ is the CDF of the (univariate) Student's $t$, $\omega = (\Omega\odot I_d)^{1/2}$, $\bar\Omega = \omega^{-1}\Omega\omega^{-1}$, and $\alpha = \omega\eta$.\footnote{Equation (6.24) of \citet{azzalini2013skewnorm}.} We used random initialisations of $\xi \sim \mathcal{N}(\mathbf{0}, 0.01^2I_d)$, $\eta \sim \mathcal{N}(\mathbf{0}, 0.01^2I_d)$, with $\Omega$ and $\nu$ initialised to $I_d$ and $50$ respectively. For all methods we parameterised $\nu$ as $\log(\nu-2)$.
\paragraph{$t$ copula} The $d$-dimensional $t$ copula with parameters $\theta = (R, \nu)$ has PDF $q_\theta(x) = \mathcal{T}_d(z; 0, R)/\prod_i \mathcal{T}(z_i; 0, 1)$, where $\mathcal{T}$ is the PDF of the univariate Student's $t$. We initialised $R$ to $\texttt{corr}(I + W^\top W)$, and $\nu$ to $50$, where $W_{ij} \sim \mathcal{N}(0, .01^2)$, and $\texttt{corr}(\Sigma) = (\Sigma\odot I_d)^{-\frac{1}{2}}\Sigma(\Sigma\odot I_d)^{-\frac{1}{2}}$ projects a covariance matrix to its implied correlation matrix. For all methods we parameterised $\nu$ as $\log(\nu-2)$.
\paragraph{Negative binomial mixture} The negative binomial mixture with $k$ components and parameters $\theta = (\pi_i, r_i, s_i)_{i=1}^k$ has PMF given by \eqref{eqn:mixturedistribution}, where $q_{\theta_i}(x)$ is the negative binomial PMF given by \eqref{eqn:negbin_pmf} with parameters $\theta_i = (r_i, s_i)$. We initialised $\pi_i$ to $k^{-1}$ for $i = 1, \dots, k$, and the component negative binomials were initialised to have mean $m_i$ and variance $v_i$, with $m_i \sim \text{Uniform}(0.1x_\text{max}, 0.9x_\text{max})$ and $v_i = m_io_i$ where $o_i^{-1} \sim \text{Uniform}(0.001, 0.02)$ and $x_\text{max}$ is the maximum value observed in the training data. With GD and BFGS, we used a softmax parameterisation of the mixture probabilities, a log parameterisation of $r_i$, and a logit parameterisation of $s_i$.
\paragraph{Skew-normal mixture}  The skew-normal mixture with $k$ components and parameters $\theta = (\pi_i, \xi_i, \Omega_i, \nu_i)_{i=1}^k$ has PDF given by \eqref{eqn:mixturedistribution}, where $q_{\theta_i}(x)$ is the skew-normal PDF with parameters $\theta_i = (\xi_i, \Omega_i, \nu_i)$. We initialised $\pi_i$ to $k^{-1}$ for $i = 1, \dots, k$, with the remaining parameters initialised equivalently to the (non-mixture) skew-normal case. With Adam, we used a softmax parameterisation of the mixture probabilities.

%% file: appendix_mappings.tex
\label{app:mappings}

In Table \ref{tbl:mapppings} we provide parameter mappings for all of the examples appearing in this paper. For convenience, we express the mappings in terms of the standard parameterisation of $\tilde q$. For the method described in Section \ref{subsec:method_sngd}, $\tilde\theta$ would be the mean or natural parameters corresponding to the parameters stated here, with $g$ similarly adjusted. Note that in Table \ref{tbl:mapppings} we follow convention by using $\mu$ to denote the mean of a normal distribution, whereas in the main paper it refers to the mean \textit{parameters} (expected sufficient statistics) of an EF distribution. See Appendix \ref{app:experiment_details} for target distribution definitions.

{
\setlength{\tabcolsep}{3.8pt}
\begin{table}[h]
% \begin{sidewaystable*}[t]
  \caption{Example surrogate-target parameter mappings}
  \label{tbl:mapppings}
  \centering
  \begin{tabular}{llcccc}
    \toprule
    \textbf{TARGET} & \textbf{SURROGATE} & $\theta$ & $\tilde{\theta}$ & $\lambda$ & $g(\tilde{\theta}, \lambda)$ \\
    \midrule
    Neg. bin. & Gamma & $r, s$ & $\alpha, \beta$ & & $\alpha/(1-\beta^{-1}), \beta^{-1}$  \\
    Neg. bin. mix. & Gamma mix. model & $(\pi_i, r_i, s_i)$ & $(\pi_i, \alpha_i, \beta_i)$ & & $(\pi_i, \alpha_i/(1-\beta_i^{-1}), \beta_i^{-1})$ \\
    % Student's $t$ & Normal & $\mu, \Sigma, \nu$ & $\mu, \Sigma$ & $\nu$ & $\mu, \Sigma, \nu$ \\
    Skew-normal & Normal & $\xi, \Omega, \eta$ & $\mu, \Sigma$ & $\eta$ & $\mu, \Sigma, \eta$ \\
    Skew-normal mix. & Normal mix. model & $(\pi_i, \xi_i, \Omega_i, \eta_i)$ & $(\pi_i, \mu_i, \Sigma_i)$ & $(\eta_i)$ & $(\pi_i, \mu_i, \Sigma_i, \eta_i)$ \\
    Skew-$t$ & Normal & $\xi, \Omega, \eta, \nu$ & $\mu, \Sigma$ & $\eta, \log(\nu-2)$ & $\mu, \Sigma, \eta, \nu$ \\
     % & MCEF Skew-Normal & & $\xi, \Omega, \eta$ & $\nu$ & $\xi, \Omega, \eta, \nu$ \\
     % & MCEF Student's $t$ & & $\mu, \Sigma, \nu$ & $\eta$ & $\mu, \Sigma, \eta, \nu$ \\
    Elliptical copula & Zero-mean normal & $R, \dots$ & $\Sigma$ & $\dots$ & $\text{corr}(\Sigma), \dots \footnotemark$ \\
    % Elliptical Copula & Zero-mean Low-Rank Normal & $R, ...$ & $A, v$ & $...$ & $\text{corr}(A^\top A + \text{diag}(v)), ...$ \\
    % Student's $t$ Copula & Zero-mean MCEF Student's $t$ & $R, \nu$ & $\mu, \Sigma, \nu$ & & $\text{corr}(\Sigma), \nu$ \\
    \bottomrule
  \end{tabular}
\end{table}
}

\footnotetext{$\text{corr}(.)$ projects a covariance matrix to its implied correlation matrix. See Appendix \ref{app:experiment_details} for its definition.}
% \end{sidewaystable*}

%% file: appendix_results.tex
\label{app:experiment_results}

In this appendix we provide additional results for the experiments presented in the main paper. In several of the experiments (those not using line-search), the competing methods were dependent on hyperparameters which were chosen by grid search. For the training curves in the main paper, we used the heuristic method described in Appendix \ref{app:experiment_details} in order to choose the `best' settings for each method.

In order to provide a more complete comparison, in Figures \ref{fig:paretofrontiers_covtype_mvsn} through \ref{fig:paretofrontiers_covtype_mvsnmix} we present Pareto frontier plots that show the best (average) value of the chosen metric ($y$ axis) attained by \emph{any} learning rate setting for each method, as a function of both iteration count and wall-clock time ($x$ axis). The displayed error bars correspond to 2 standard errors, computed for the optimal setting at the corresponding point in time.

\begin{figure}[htbp]
    \centering
    \begin{minipage}[t]{0.49\textwidth}
        \begin{subfigure}[t]{.49\linewidth}
            \centering
            \includegraphics[height=.95\linewidth]{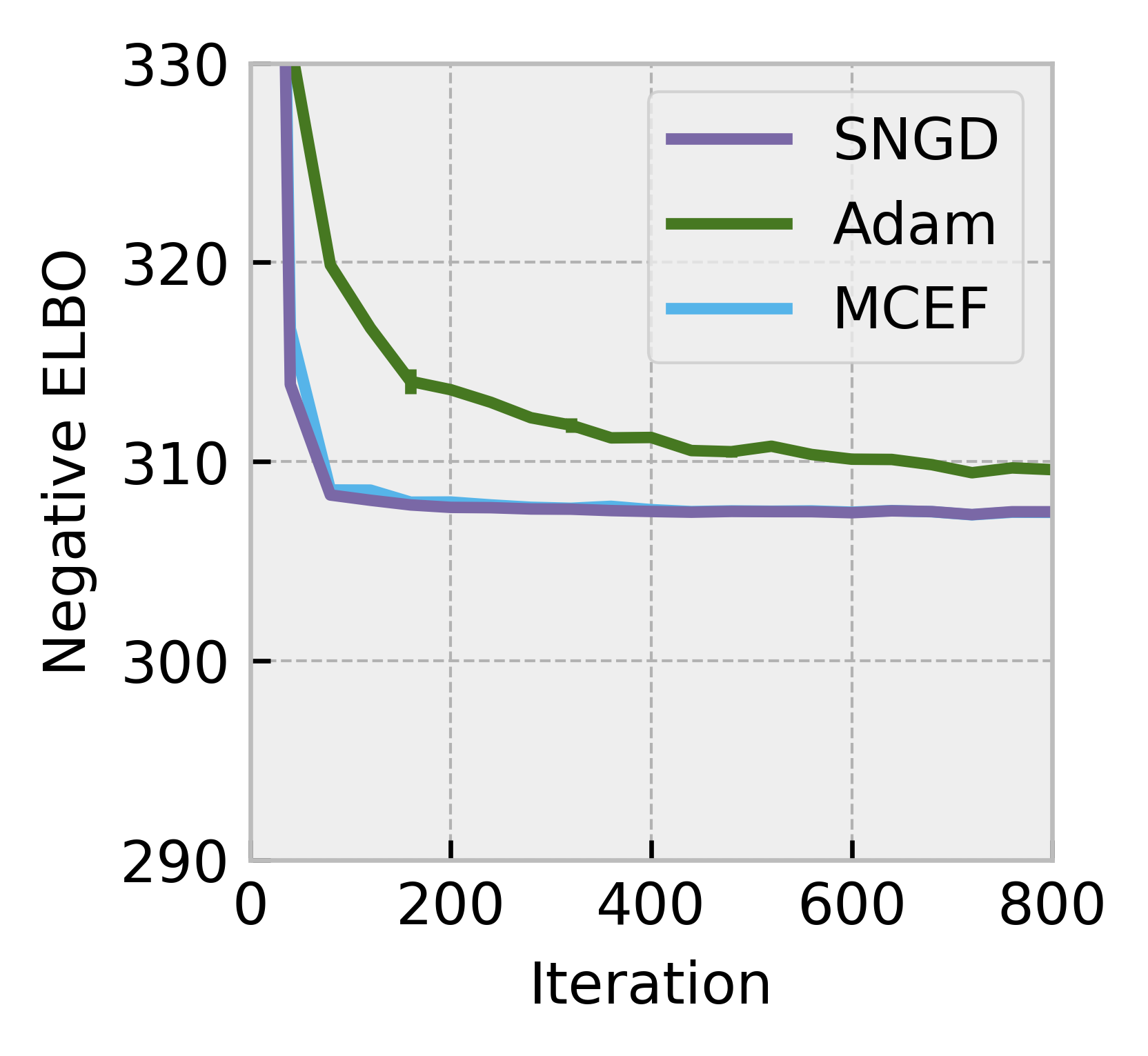}
        \end{subfigure}
        \hfill
        \begin{subfigure}[t]{.49\linewidth}
            \centering
            \includegraphics[height=.95\linewidth]{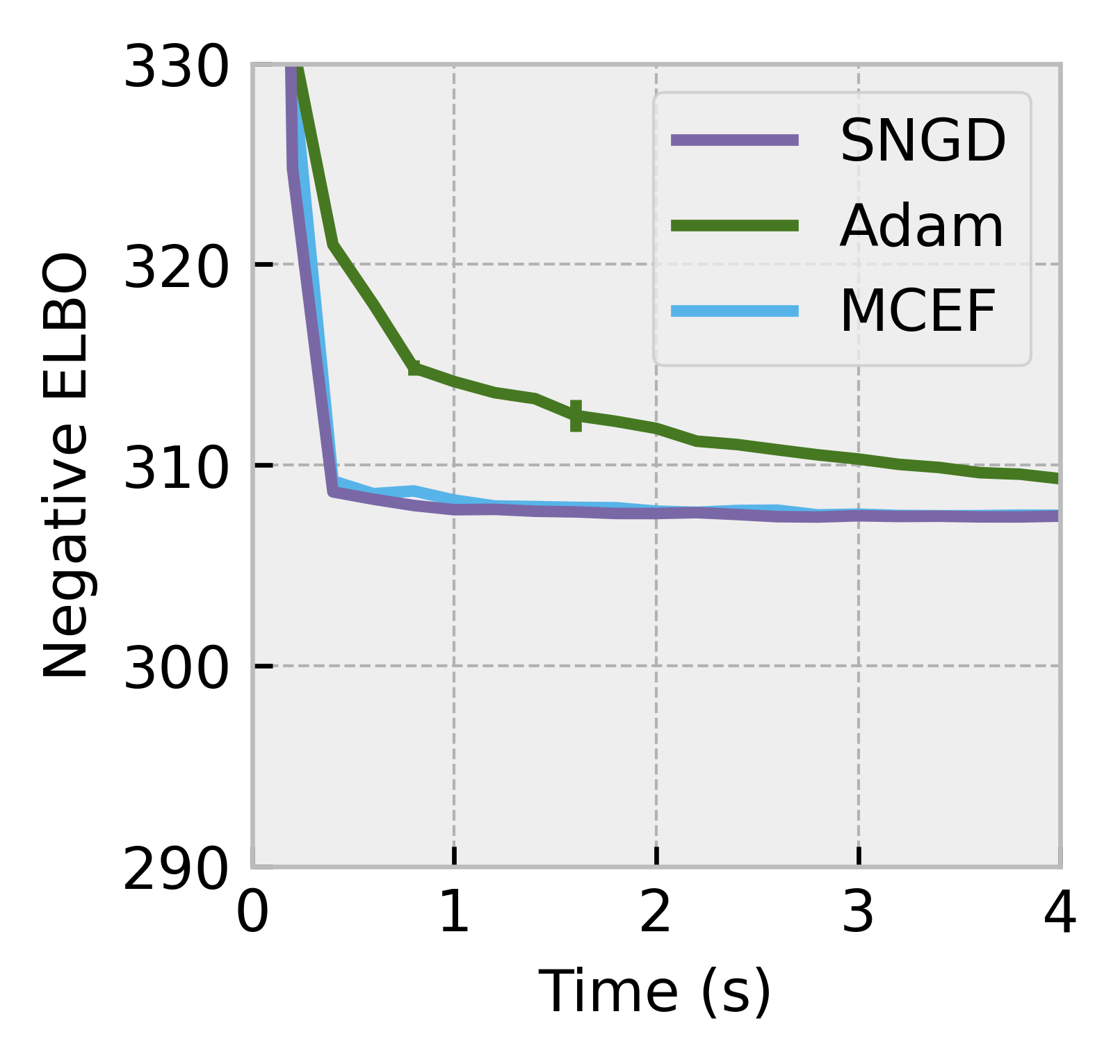}
        \end{subfigure}
        \caption{Pareto frontiers for Bayesian logistic regression VI on UCI covertype ($n$=500, $d$=53), with a skew-normal approximation.}
        \label{fig:paretofrontiers_covtype_mvsn}
    \end{minipage}
    \hfill % to add some horizontal spacing between the figures
    \begin{minipage}[t]{0.49\textwidth}
        \begin{subfigure}[t]{.49\linewidth}
            \centering
            \includegraphics[height=.95\linewidth]{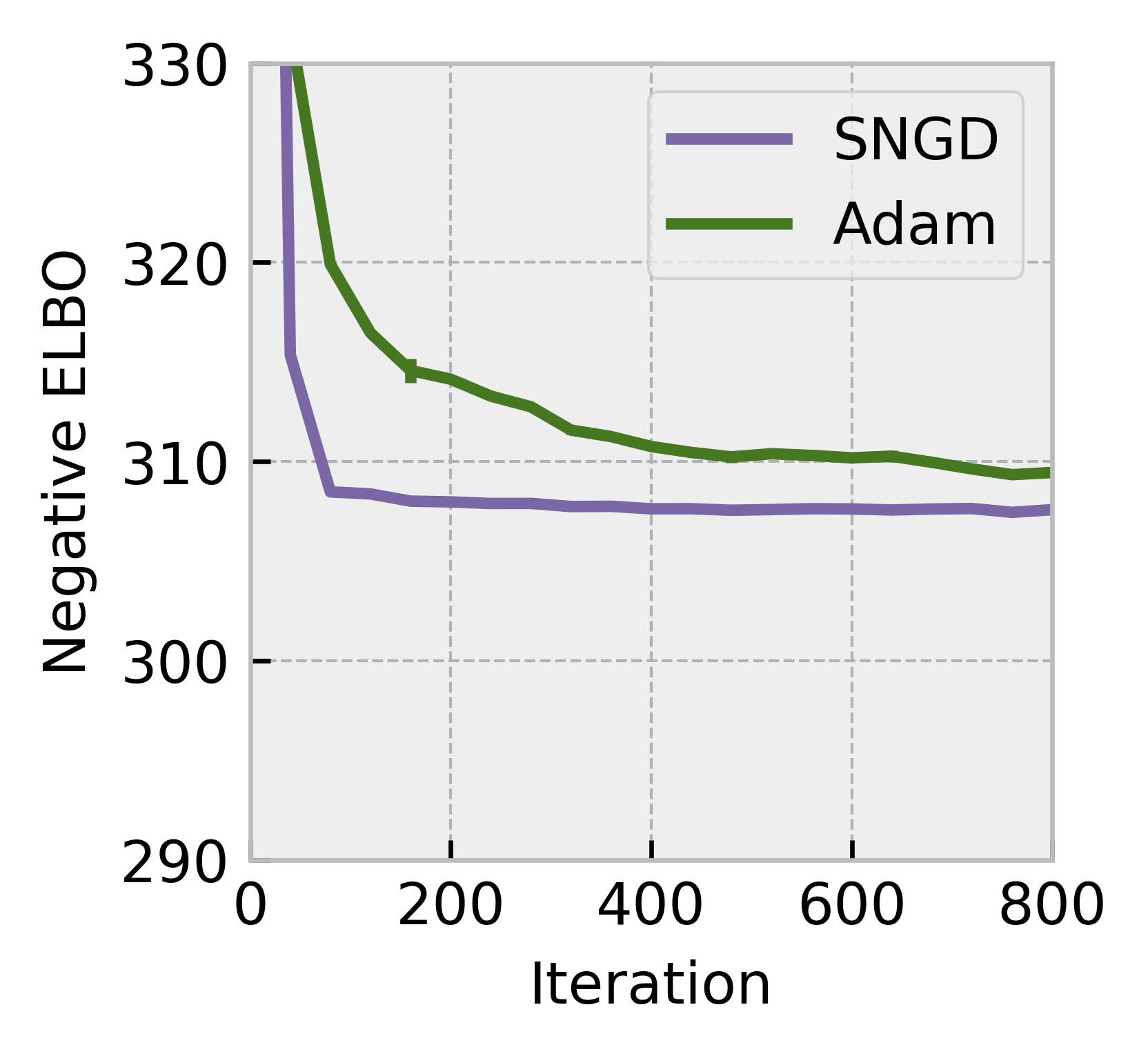}
        \end{subfigure}
        \hfill
        \begin{subfigure}[t]{.49\linewidth}
            \centering
            \includegraphics[height=.95\linewidth]{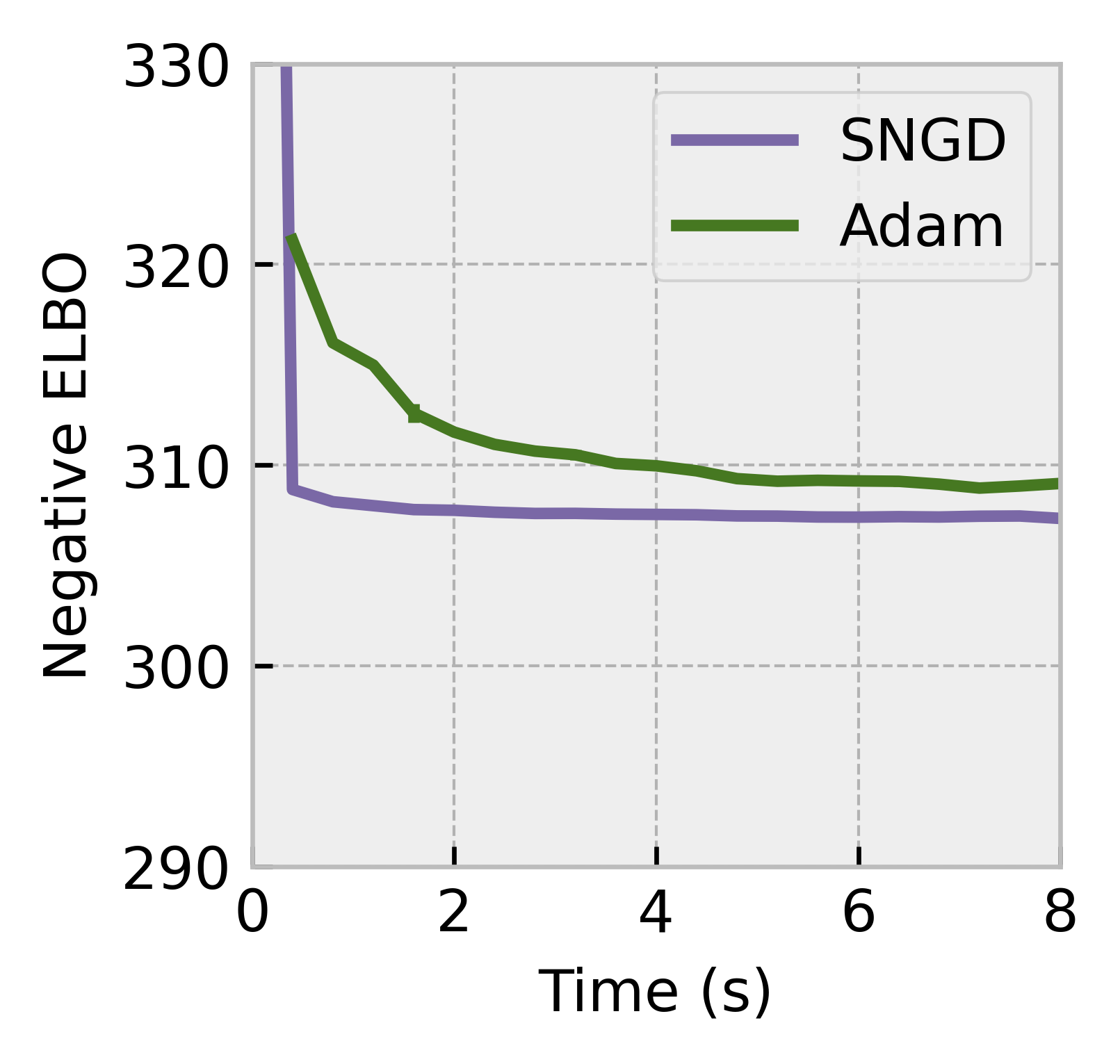}
        \end{subfigure}
        \caption{Pareto frontiers for Bayesian logistic regression VI on UCI covertype ($n$=500, $d$=53), with a skew-$t$ approximation.}
        \label{fig:paretofrontiers_covtype_mvst}
    \end{minipage}
\end{figure}

\begin{figure}[htbp]
    \centering
    \begin{minipage}[t]{0.49\textwidth}
        \begin{subfigure}[t]{.49\linewidth}
            \centering
            \includegraphics[height=.9\linewidth]{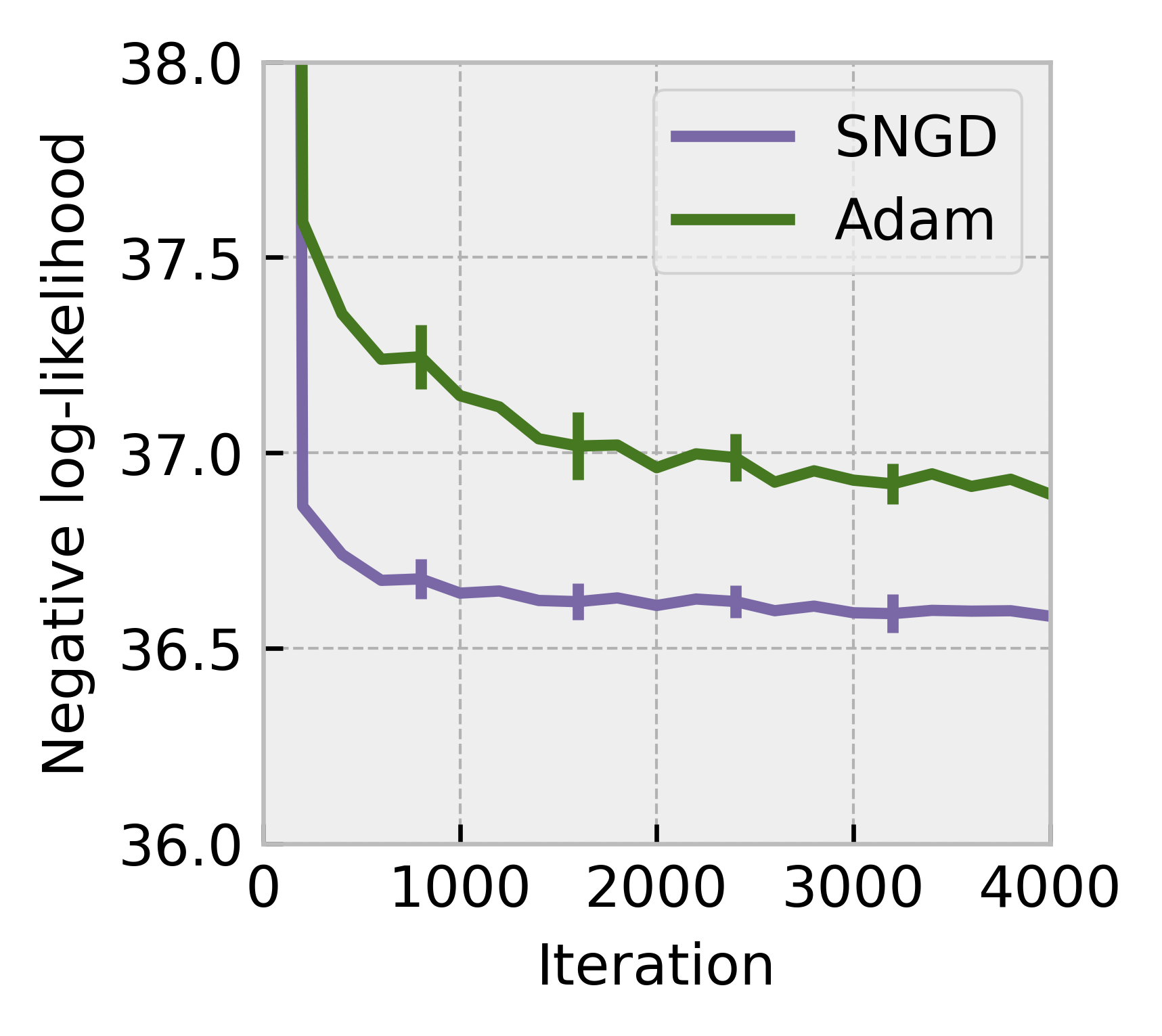}
        \end{subfigure}
        \hfill
        \begin{subfigure}[t]{.49\linewidth}
            \centering
            \includegraphics[height=.9\linewidth]{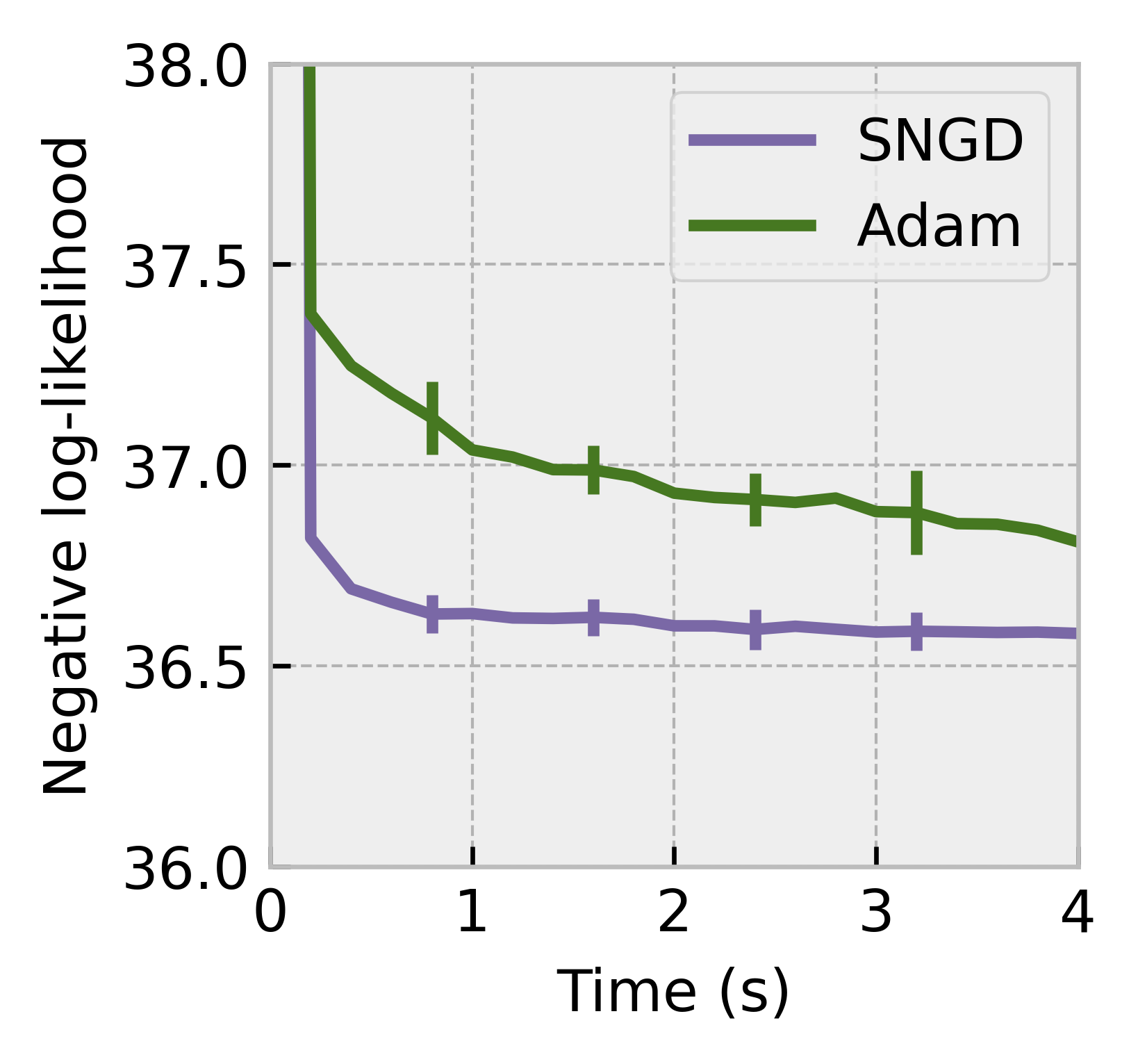}
        \end{subfigure}
        \caption{Pareto frontiers for skew-normal MLE on UCI miniboone ($n$=32,840, $d$=43).}
        \label{fig:paretofrontiers_miniboone_mvsn}
    \end{minipage}
    \hfill % to add some horizontal spacing between the figures
    \begin{minipage}[t]{0.49\textwidth}
        \begin{subfigure}[t]{.49\linewidth}
            \centering
            \includegraphics[height=.9\linewidth]{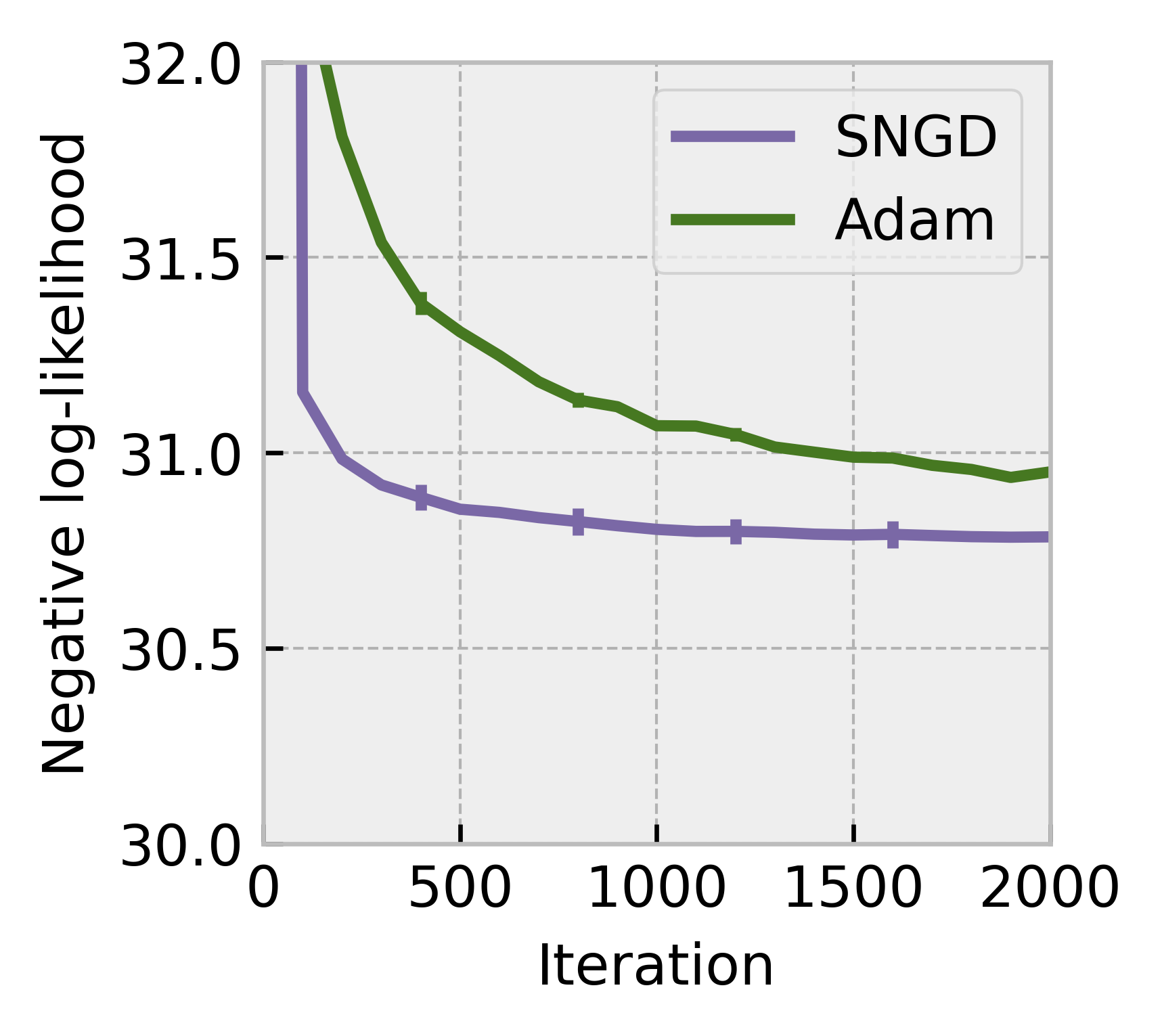}
        \end{subfigure}
        \hfill
        \begin{subfigure}[t]{.49\linewidth}
            \centering
            \includegraphics[height=.9\linewidth]{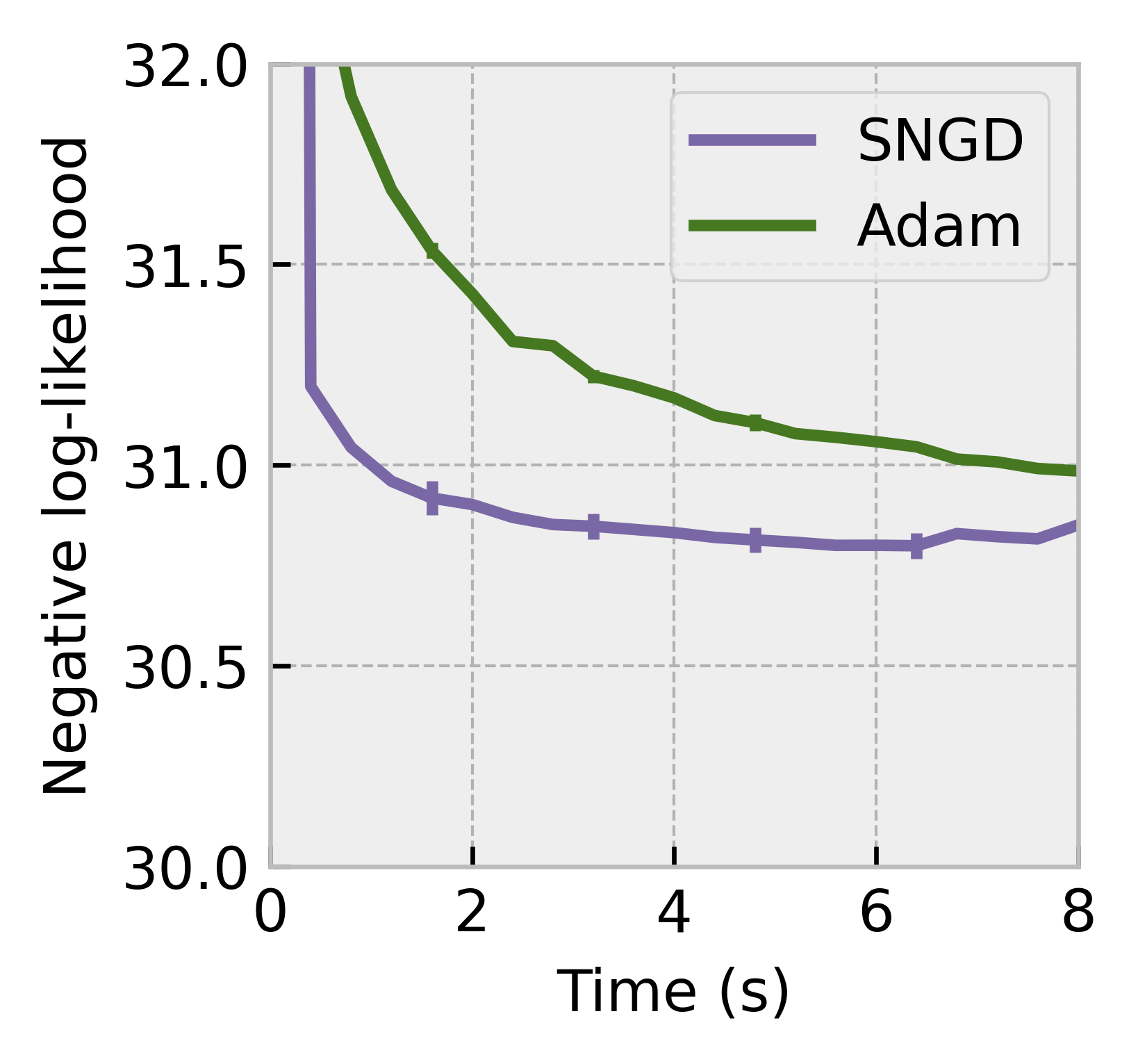}
        \end{subfigure}
        \caption{Pareto frontiers for skew-$t$ MLE on UCI miniboone ($n$=32,840, $d$=43).}
        \label{fig:paretofrontiers_miniboone_mvst}
    \end{minipage}
\end{figure}

\begin{figure}[htbp]
    \centering
    \begin{minipage}[t]{0.49\textwidth}
        \begin{subfigure}[t]{.49\linewidth}
            \centering
            \includegraphics[height=.9\linewidth]{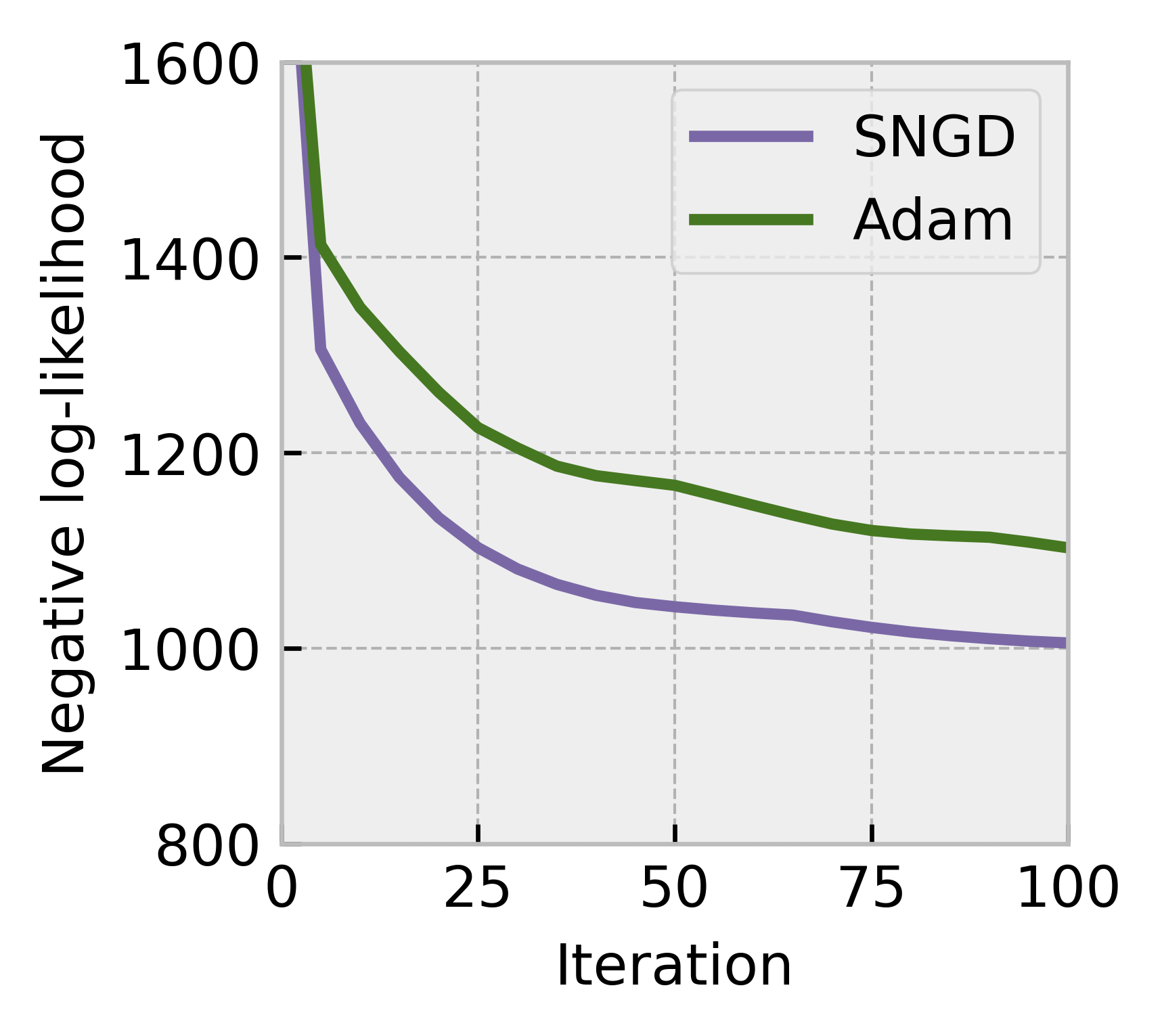}
        \end{subfigure}
        \hfill
        \begin{subfigure}[t]{.49\linewidth}
            \centering
            \includegraphics[height=.9\linewidth]{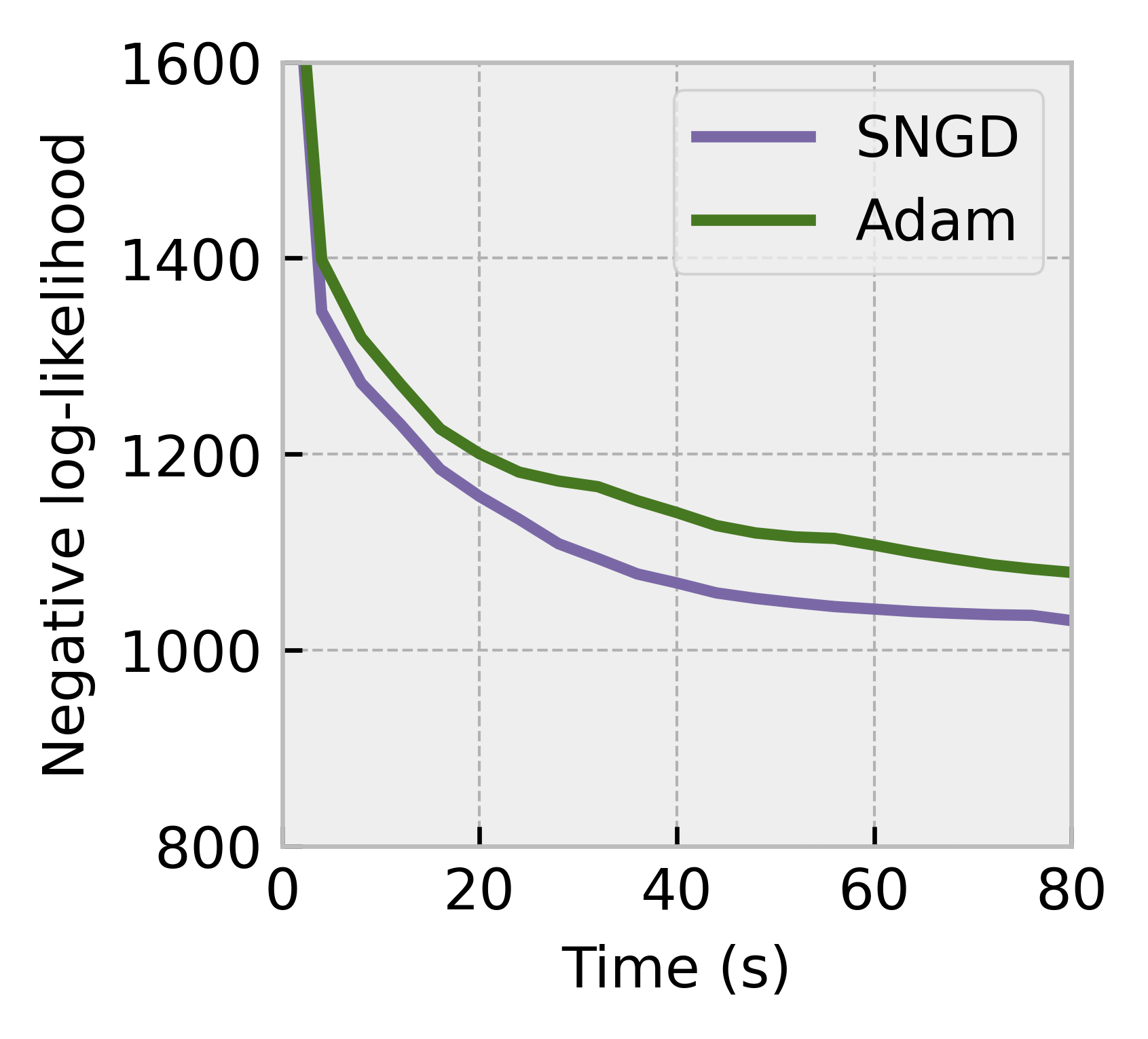}
        \end{subfigure}
        \caption{Pareto frontiers for skew-normal MLE on synthetic dataset ($n$=10,000, $d$=1,000).}
        \label{fig:paretofrontiers_synthetic_mvsn}
    \end{minipage}
    \hfill % to add some horizontal spacing between the figures
    \begin{minipage}[t]{0.49\textwidth}
        \begin{subfigure}[t]{.49\linewidth}
            \centering
            \includegraphics[height=.9\linewidth]{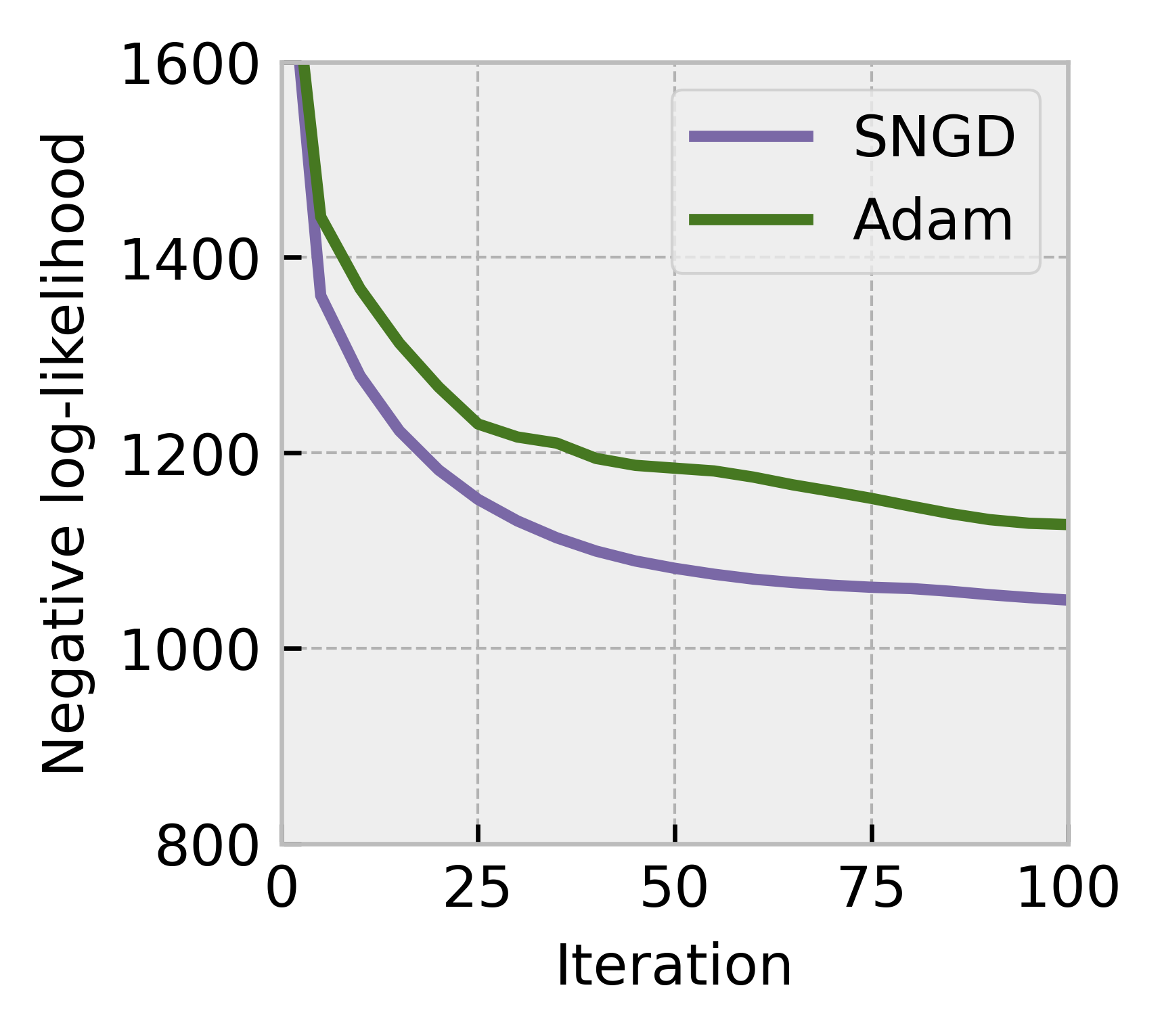}
        \end{subfigure}
        \hfill
        \begin{subfigure}[t]{.49\linewidth}
            \centering
            \includegraphics[height=.9\linewidth]{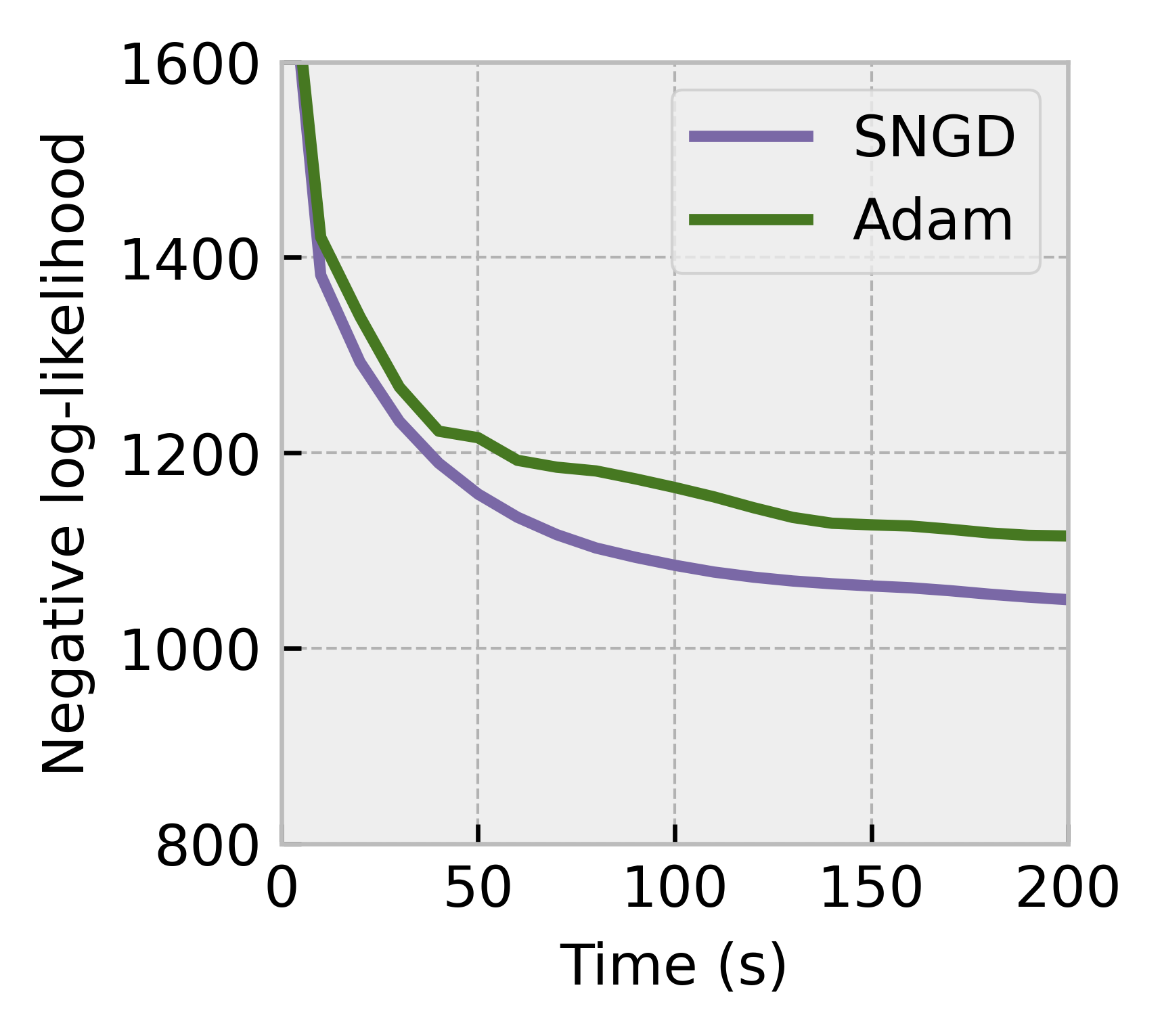}
        \end{subfigure}
        \caption{Pareto frontiers for skew-$t$ MLE on synthetic dataset ($n$=10,000, $d$=1,000).}
        \label{fig:paretofrontiers_synthetic_mvst}
    \end{minipage}
\end{figure}

\begin{figure}[htbp]
    \centering
    \begin{minipage}[t]{0.49\textwidth}
        \begin{subfigure}[t]{.49\linewidth}
            \centering
            \includegraphics[height=.9\linewidth]{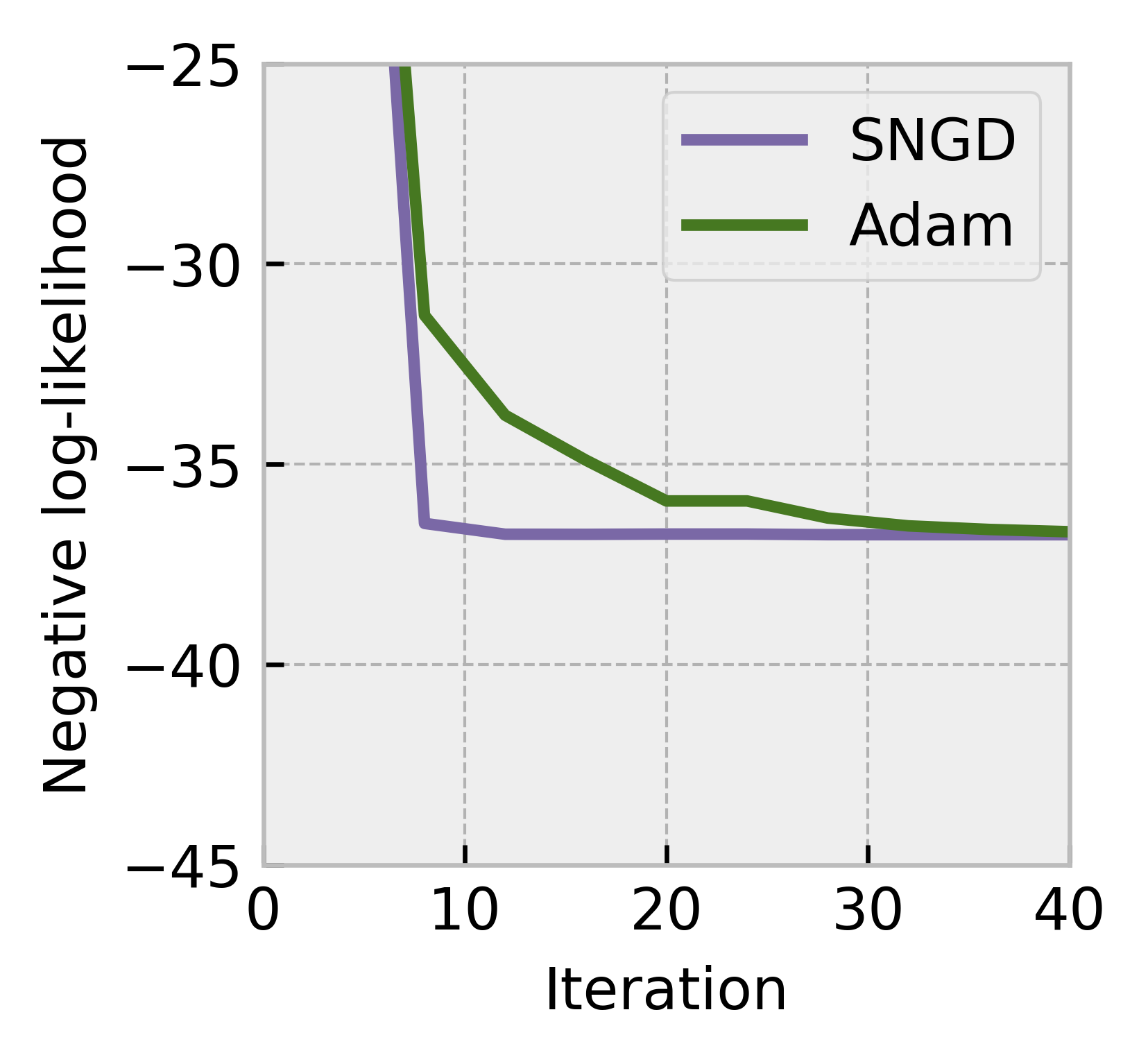}
        \end{subfigure}
        \hfill
        \begin{subfigure}[t]{.49\linewidth}
            \centering
            \includegraphics[height=.9\linewidth]{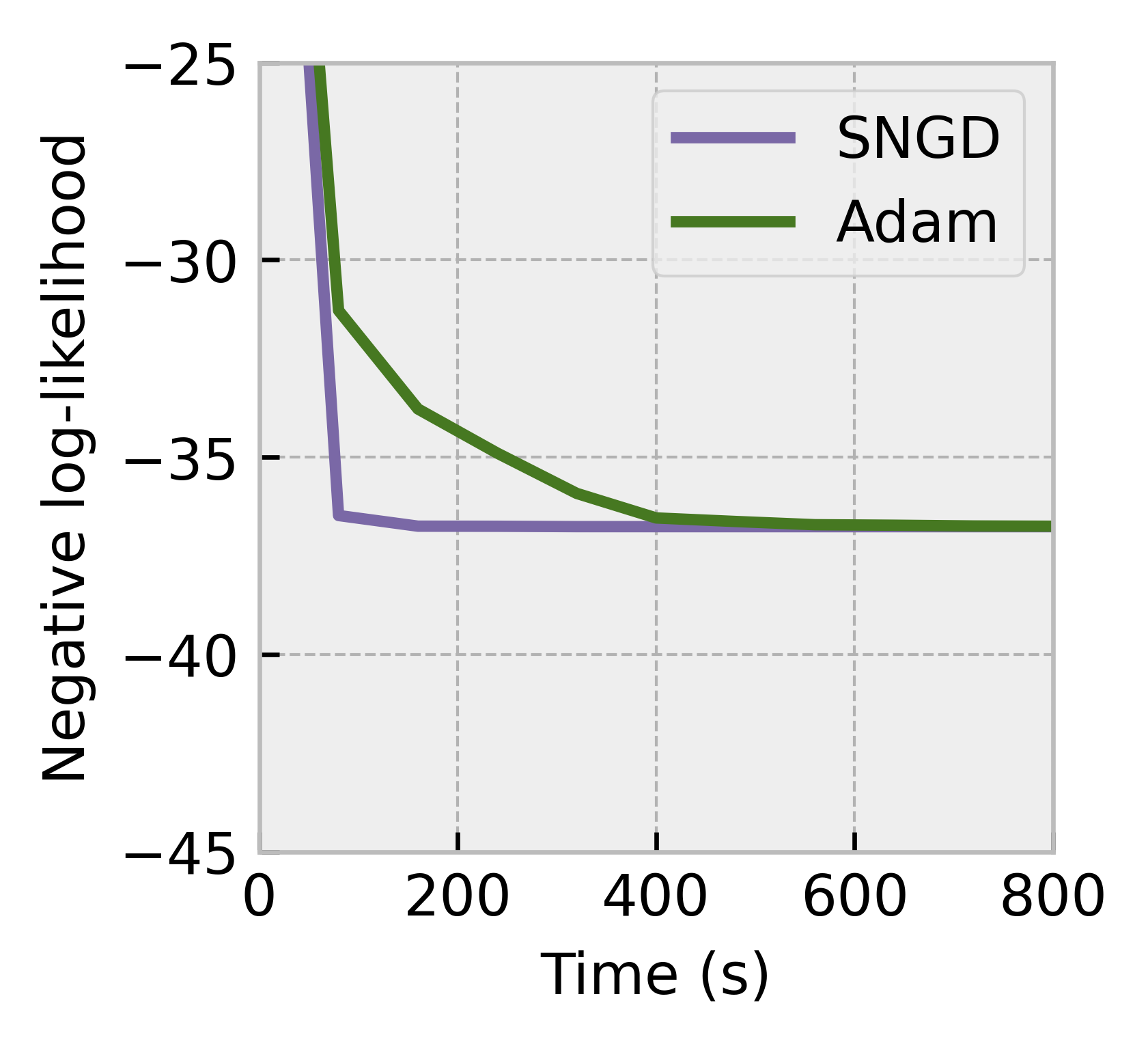}
        \end{subfigure}
        \caption{Pareto frontiers for $t$-copula MLE on 5 years of stock return data ($n$=1,515, $d$=93).}
        \label{fig:paretofrontiers_ftse100_tcopula}
    \end{minipage}
    \hfill % to add some horizontal spacing between the figures
    \begin{minipage}[t]{0.49\textwidth}
        \begin{subfigure}[t]{.49\linewidth}
            \centering
            \includegraphics[height=.9\linewidth]{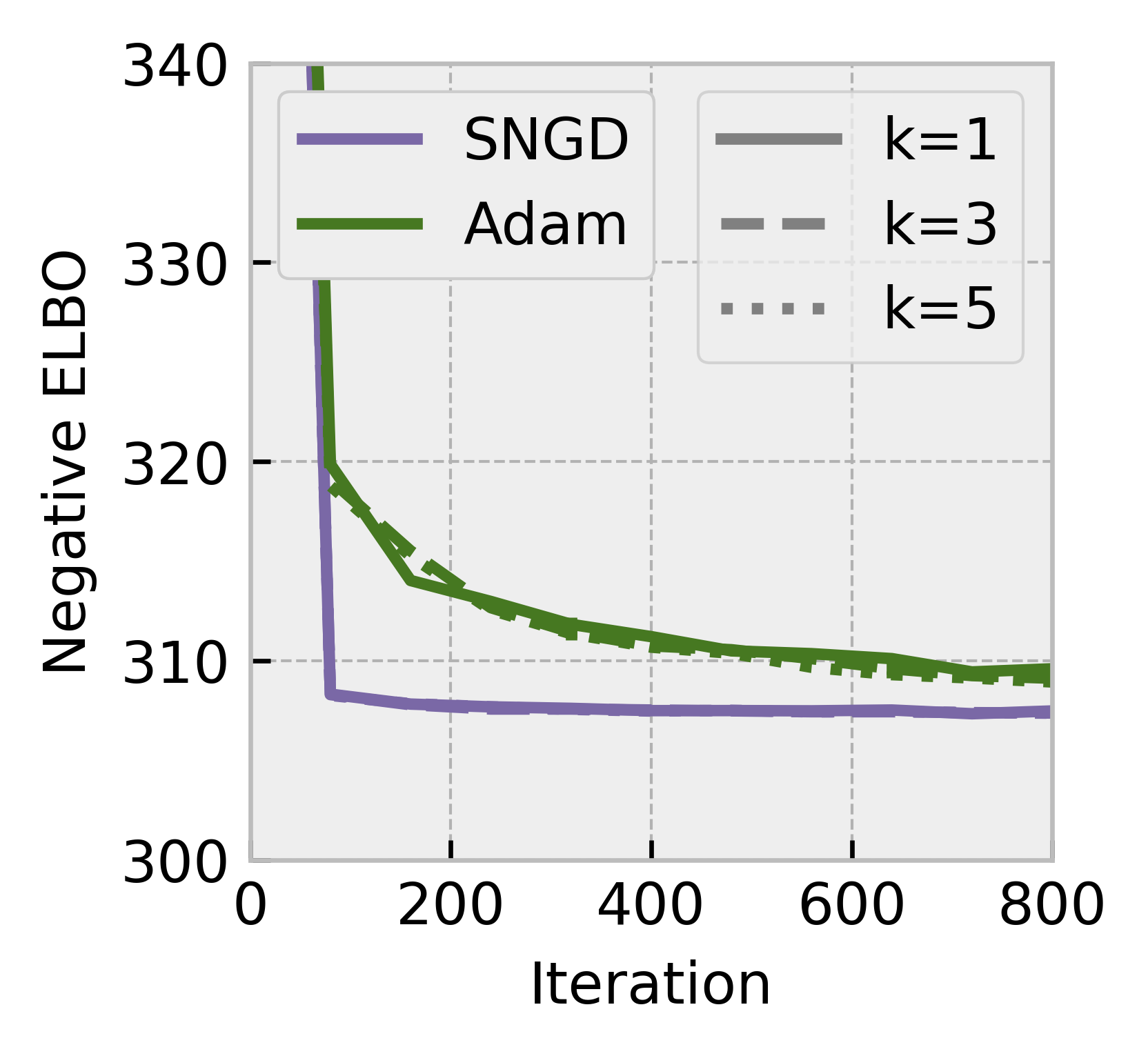}
        \end{subfigure}
        \hfill
        \begin{subfigure}[t]{.49\linewidth}
            \centering
            \includegraphics[height=.9\linewidth]{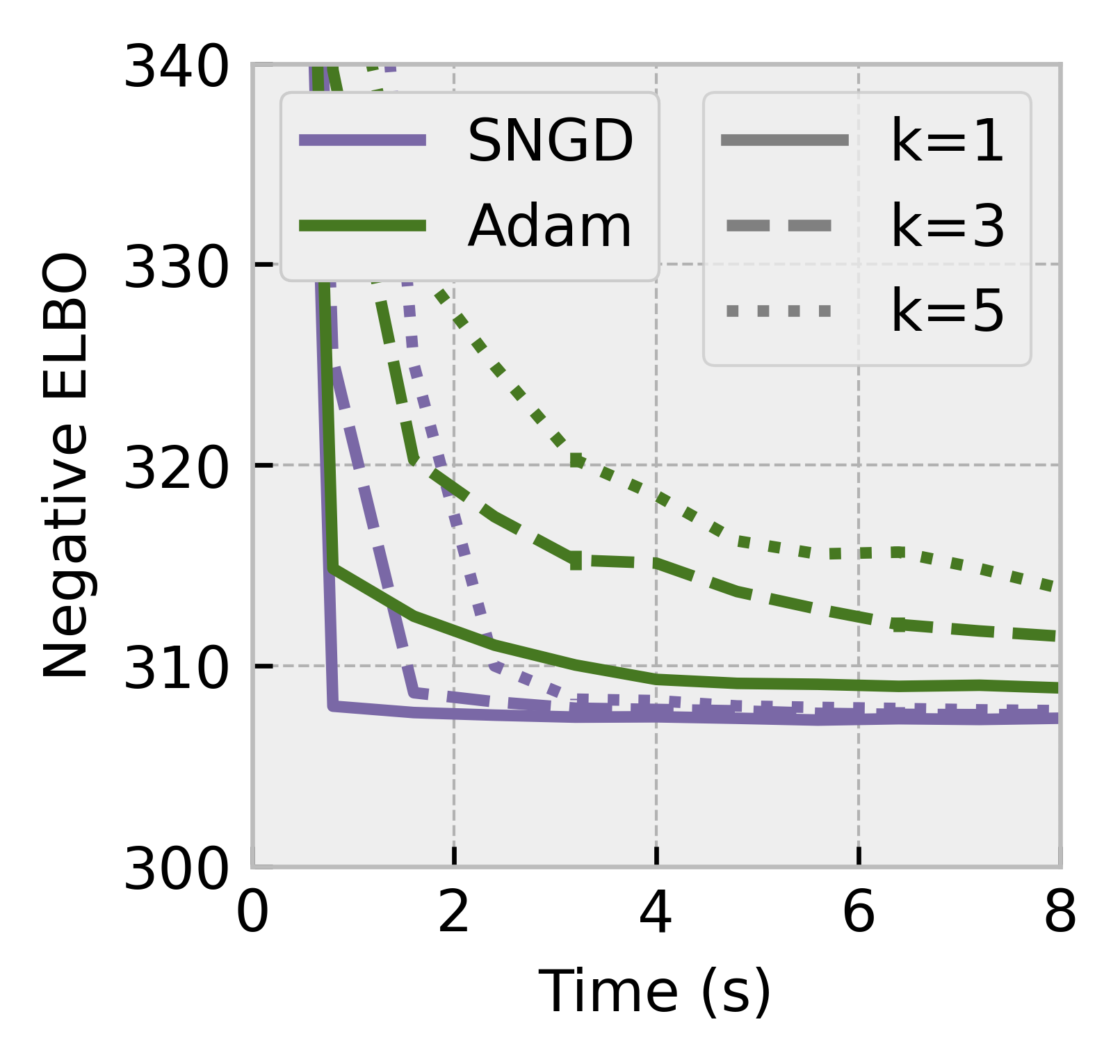}
        \end{subfigure}
        \caption{Pareto frontiers for Bayesian logistic regression VI on UCI covertype ($n$=500, $d$=53), with a skew-normal mixture approximation.}
        \label{fig:paretofrontiers_covtype_mvsnmix}
    \end{minipage}
\end{figure}

%% file: appendix_choosingsurrogates.tex
\label{app:choosingsurrogates}

In order to apply SNGD to a given target distribution $q$, we must choose a surrogate $\tilde q$, and reparameterisation $\theta = g(\tilde\theta)$. Appropriate choices here determine the effectiveness of SNGD.

One guiding principle that can be used is based on a view of NGD involving Kullback-Leibler (KL) divergences. The KL divergence from continuous\footnote{The discrete KL divergence is defined similarly, with summation replacing integration.} distribution $q$ to $p$ is defined as
\begin{align}
    \kl{q}{p} &= \int q(x)\log\frac{q(x)}{p(x)}\mathrm{d}x.
\end{align}
The NGD direction can be motivated by a result due to \citet{ollivier2017natgrads}, restated here in our notation.

\begin{proposition} Let $f$ be a smooth function on the parameter space $\Theta$. Let $\theta \in \Theta$ be a point where $\tilde\nabla f(\theta) = [F(\theta)]^{-1}\nabla f(\theta)$ does not vanish. Then, if
\begin{align}
    \delta &= -\frac{\tilde\nabla f(\theta)}{||\tilde\nabla f(\theta)||_F}
\end{align}
is the negative direction of the natural gradient of $f$ (with $||.||_F$ the Fisher norm), we have
\begin{align}
    \delta &= \lim_{\tau\rightarrow 0^+} \frac{1}{\sqrt{2\tau}}\argmin_{\delta' \in B(\tau)} f(\theta+\delta')
\end{align}
where $B(\tau) = \{\delta' : \kl{q_{\theta+\delta'}}{q_\theta} < \tau\}$.
\end{proposition}

In words, this says that NGD moves in the direction in parameter space that gives the greatest decrease in the objective within an infinitesimally small KL-ball around the current point.

\begin{figure}[t]
    \centering
    \includegraphics[width=.45\linewidth]{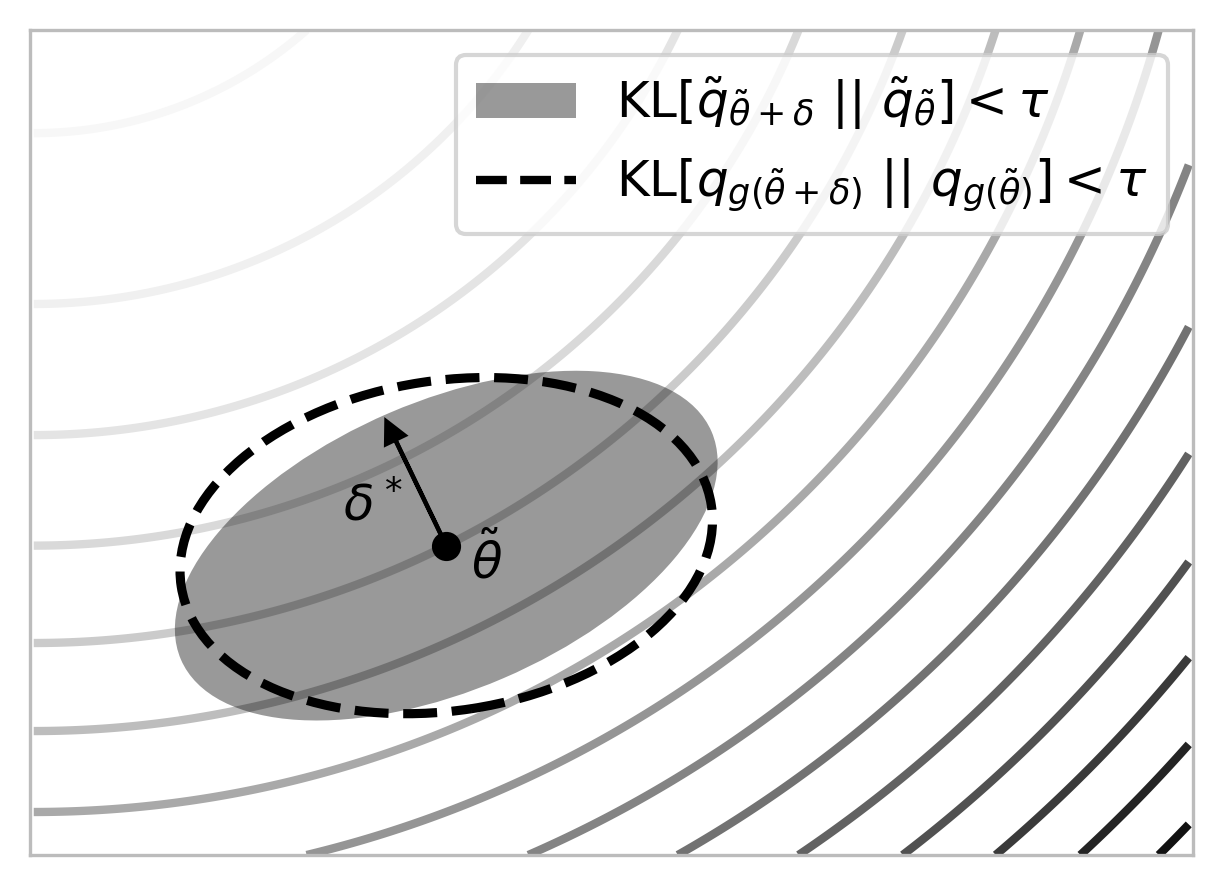}
    \caption{Illustration of SNGD. An infinitesimal KL ball (ellipse, in parameter space) around $\tilde\theta$ with respect to: $\tilde q$ (shaded), and $q$ (dashed). Contours are of $\tilde f(\tilde\theta)$. $\delta^*$ points to the minimum point in the shaded region; this is the SNGD direction. In general this will differ from the direction of NGD in $q$, which would instead point in the direction of the minimum point in the dashed region (not marked).}
    \label{fig:klball}
\end{figure}

This suggests that we would ideally like to find $\tilde q$ and $g$ such that the KL divergence between the surrogate $\tilde q_{\tilde\theta}$ and a perturbed version $\tilde q_{\tilde\theta+\delta}$ is identical to the KL divergence from the target $q_{g(\tilde\theta)}$ and its corresponding pertubation $q_{g(\tilde\theta+\delta)}$, for any small pertubation $\delta$. However, because this is only required to hold in some infinitesimally small neighbourhood of $\tilde\theta$, it is sufficient to satisfy
\begin{align}
    \nabla_\delta\big(\kl[\big]{q_{\tilde\theta+\delta}}{\tilde q_{\tilde\theta}}\big)\big\rvert_{\mathbf{0}} &= \nabla_\delta\big( \kl[\big]{q_{g(\tilde\theta+\delta)}}{q_{g(\tilde\theta)}}\big)\big\rvert_{\mathbf{0}}.\label{eqn:klequality}
\end{align}
If equality \eqref{eqn:klequality} holds then SNGD will move in the same direction in $\tilde\theta$ as `exact' NGD under (reparameterised) $q$. While it will not typically be possible to find tractable surrogates for which equality holds exactly, this observation motivates choosing $\tilde q$, $g$ for which it is approximately true. We illustrate this in Figure \ref{fig:klball}. This perspective also highlights the reason that $q$ and $\tilde q$ do not need to be distributions over the same space; all that matters is that the effect of (small) changes in $\tilde\theta$ on $\tilde q$ and $q$ is similar in a KL-divergence sense. If approximating NGD under $q$ is our goal, then this perspective can help guide us toward a choice of $\tilde q$ and $g$.\footnote{Approximating NGD under $q$ should not \emph{necessarily} be our ultimate goal. We elaborate on this in Appendix \ref{app:ngdcomparison}.}

As a concrete example, let $q$ be a mixture distribution, with PDF given by
\begin{align}
    q_\theta(x) &= \textstyle\sum_{i=1}^k \pi_i q_{\theta_i}(x),\label{eqn:mixturedistribution_repeated}
\end{align}
where $q_{\theta_i}$ are the component distributions, $\theta = (\pi_i, \theta_i)_{i=1}^k$, and $\pi\in\Delta^{k-1}$. Let us consider as a surrogate for $q_\theta$ the corresponding mixture \emph{model} (joint distribution) with PDF
\begin{align}
    \tilde q_{\tilde\theta}(z, x) &= \pi_z q_{\theta_z}(x),\label{eqn:mixturemodel_repeated}
\end{align}
where $\tilde\theta = (\pi_i, \theta_i)_{i=1}^k$, so that $g$ is simply the identity map, and $\sum_{i=1}^k \tilde q_{\tilde\theta}(z=i, x) = q_{\tilde\theta}(x) = q_{g(\tilde\theta)}(x)$. Using standard identities, we have
\begin{align}
    \begin{split}
    \nabla_\delta\big(\kl[\big]{\tilde q_{\tilde\theta+\delta}(z, x)}{\tilde q_{\tilde\theta}(z, x)}\big) &= \nabla_\delta\left(\kl[\big]{\tilde q_{\tilde\theta+\delta}(x)}{\tilde q_{\tilde\theta}(x)}
 + \mathbb{E}_{\tilde q_{\tilde\theta+\delta}(x)}\pmb{\big[}\kl[\big]{\tilde q_{\tilde\theta+\delta}(z|x)}{\tilde q_{\tilde\theta}(z|x)}\pmb{\big]}\right) \\
&= \nabla_\delta\left(\kl[\big]{q_{g(\tilde\theta+\delta)}(x)}{q_{g(\tilde\theta)}(x)} 
 + \mathbb{E}_{q_{\tilde\theta+\delta}(x)}\pmb{\big[}\kl[\big]{\tilde q_{\tilde\theta+\delta}(z|x)}{\tilde q_{\tilde\theta}(z|x)}\pmb{\big]}\right)
    \end{split} \\
        &\approx \nabla_\delta\left(\kl[\big]{ q_{g(\tilde\theta+\delta)}(x)}{q_{g(\tilde\theta)}(x)}\right)
\end{align}

It follows that using a mixture model joint, as a surrogate for its marginal, can be viewed as using an approximate metric which includes an additional term. This term imposes an additional penalty on directions in parameter space that affect the expected \textit{responsibility} distributions.\footnote{The responsibility of component $i$ for $x$ is the probability that $x$ was generated by component $i$ having observed $x$.}
In the experiments of Section \ref{subsec:experiments_mixture} we take this approximation one step further, and use EF mixture models as surrogates for mixtures with components that are \emph{not} EF distributions.

As an aside, we note that \citet{lin2020mcef} performed NGD with respect to a joint distribution in a VI objective for which the optimisation target was a \emph{marginal} of the NGD distribution. This mismatch was not explicitly discussed by \citet{lin2020mcef}. However, it is clear that this aspect of their method can be interpreted as an application of SNGD, and similar reasoning to that above shows that it implies making a specific approximation to the Fisher metric.

%% file: appendix_ngdcomparison.tex
\label{app:ngdcomparison}

\begin{figure}[ht]
    \centering
    \includegraphics[width=.45\linewidth]{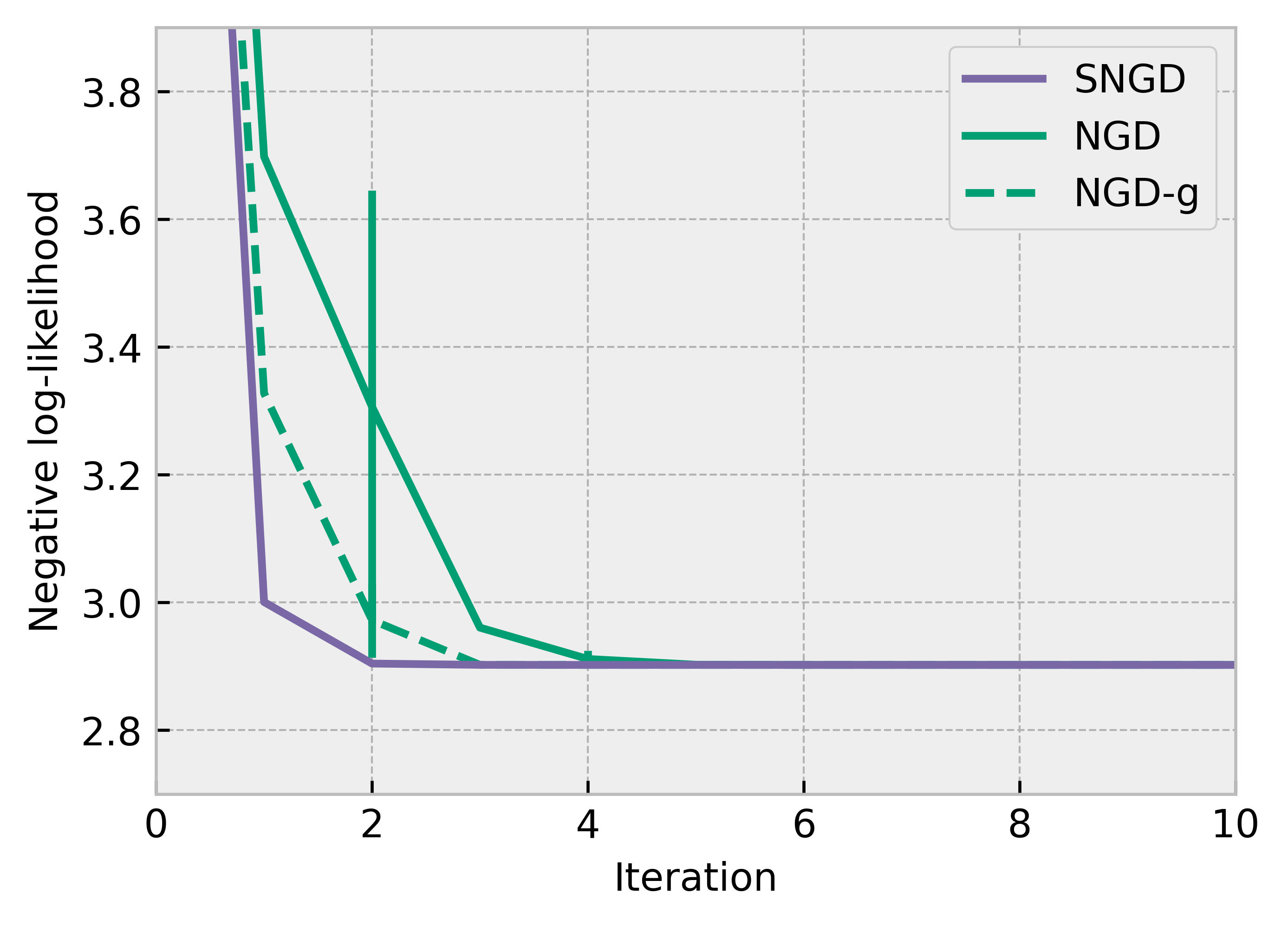}
    \caption{Training curves for negative binomial MLE on the sheep dataset ($n$=82, $d$=1), using surrogate natural gradient descent (SNGD) and natural gradient descent under $q$ with respect to both a standard parameterisation (NGD), and the reparameterisation defined by $g$ (NGD-g).}
    \label{fig:negbin_trainingcurves_ngd}
\end{figure}

In Appendix \ref{app:choosingsurrogates} we provided a motivation for SNGD that viewed it as approximating NGD under a reparameterisation of $q$. However, for the negative binomial MLE experiment of Section \ref{subsec:experiments_negbin}, we demonstrated that SNGD was actually able to \emph{outperform} NGD under $q$. In Figure \ref{fig:negbin_trainingcurves_ngd}, we show that this was true when NGD was performed with respect to both a standard parameterisation (also shown in Figure \ref{fig:negbin_experiment}), and the reparameterisation defined by $g(\tilde\theta)$. This implies that the outperformance of SNGD is not simply an artefact of the reparameterisation. This surprising result deserves extra scrutiny.

We begin by highlighting that in fact, SNGD \textit{is} NGD, but in the objective $\tilde f(\tilde\theta) = f(g(\tilde\theta))$ and with respect to distribution $\tilde q$. It is helpful therefore to consider the desirable properties of NGD:
\begin{enumerate}
    \item It is locally invariant to parameterisation of the distribution being optimised.
    \item It follows the direction of steepest descent in the objective on the statistical manifold of the distribution being optimised.\footnote{Where steepness is defined with respect to a KL divergence.}
    \item For MLE objectives, it asymptotically approaches Newton's method near the optimum.
\end{enumerate}
Properties 1 and 2 apply always, and so they are necessarily inherited by SNGD. However, when $f$ is a MLE objective for $q$ (as in the experiment of Section \ref{subsec:experiments_negbin}), $\tilde f$ is \emph{not}, in general, an MLE objective for $\tilde q$, and so SNGD loses property 3. Contrast this with NGD in $\tilde f$ with respect to $q$: although the objective function is the same as that of SNGD, it is simply a reparameterised MLE objective for $q$, and so property 3 is retained. In other words, the asymptotic efficency of NGD is retained under change of parameterisation, but not under change of distribution.

\begin{figure}[h]
    \centering
    \begin{subfigure}[t]{\linewidth}
    \centering
        \includegraphics[height=.05\linewidth]{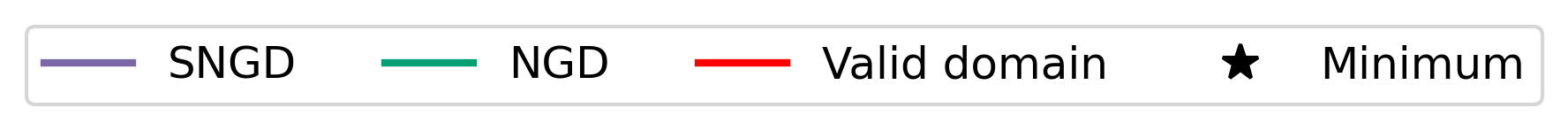}
    \end{subfigure}
    \setlength{\baselineskip}{0.\baselineskip}
    \begin{subfigure}[t]{.45\linewidth}
        \label{fig:negbin_sheeq_gradsbig}
        \centering
        \includegraphics[width=\linewidth]{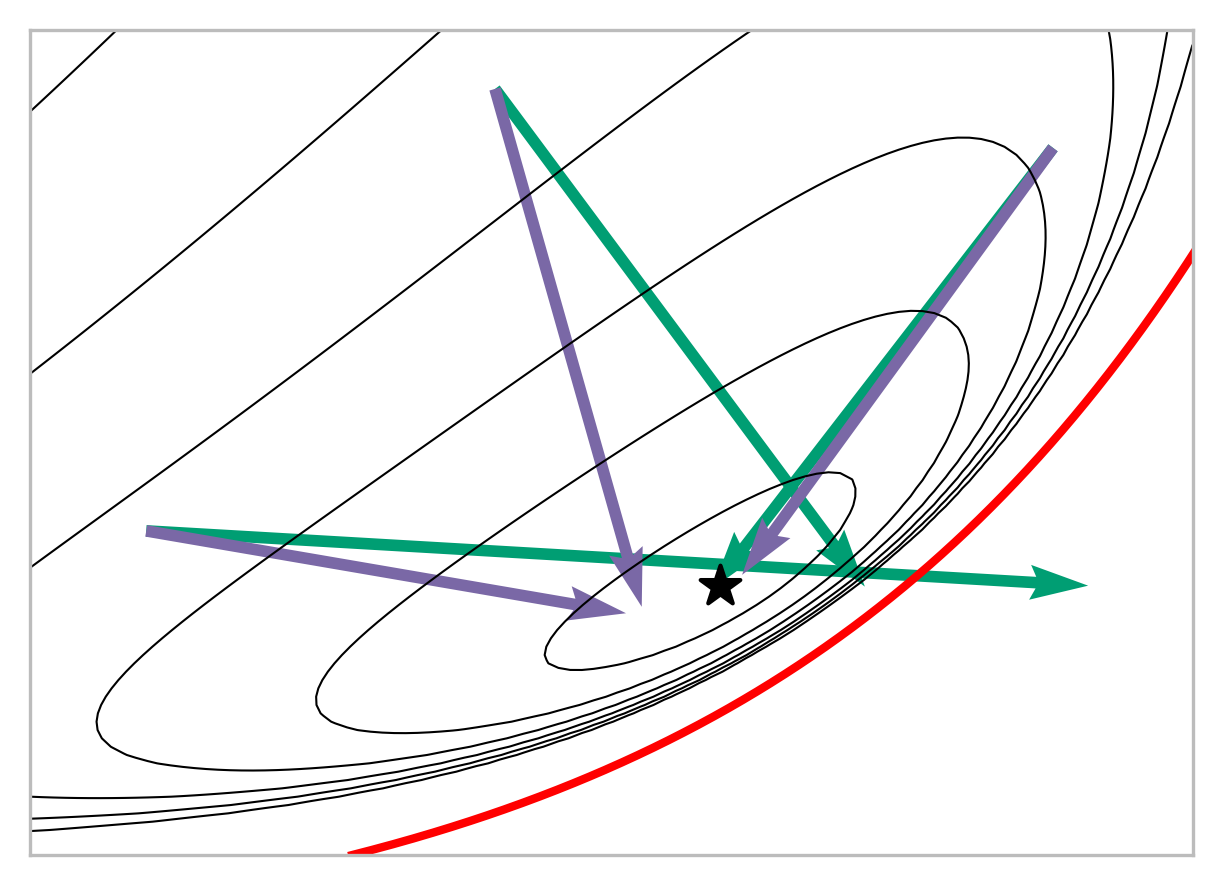}
        \caption{Vectors represent a full (undamped) step under each method. When far from the optimum, a full step of SNGD more consistently moves to a low value of the objective, compared to that of NGD under $q$.}
    \end{subfigure}
    \hfill
    \begin{subfigure}[t]{.45\linewidth}
        \label{fig:negbin_sheeq_gradssmall}
        \centering
        \includegraphics[width=\linewidth]{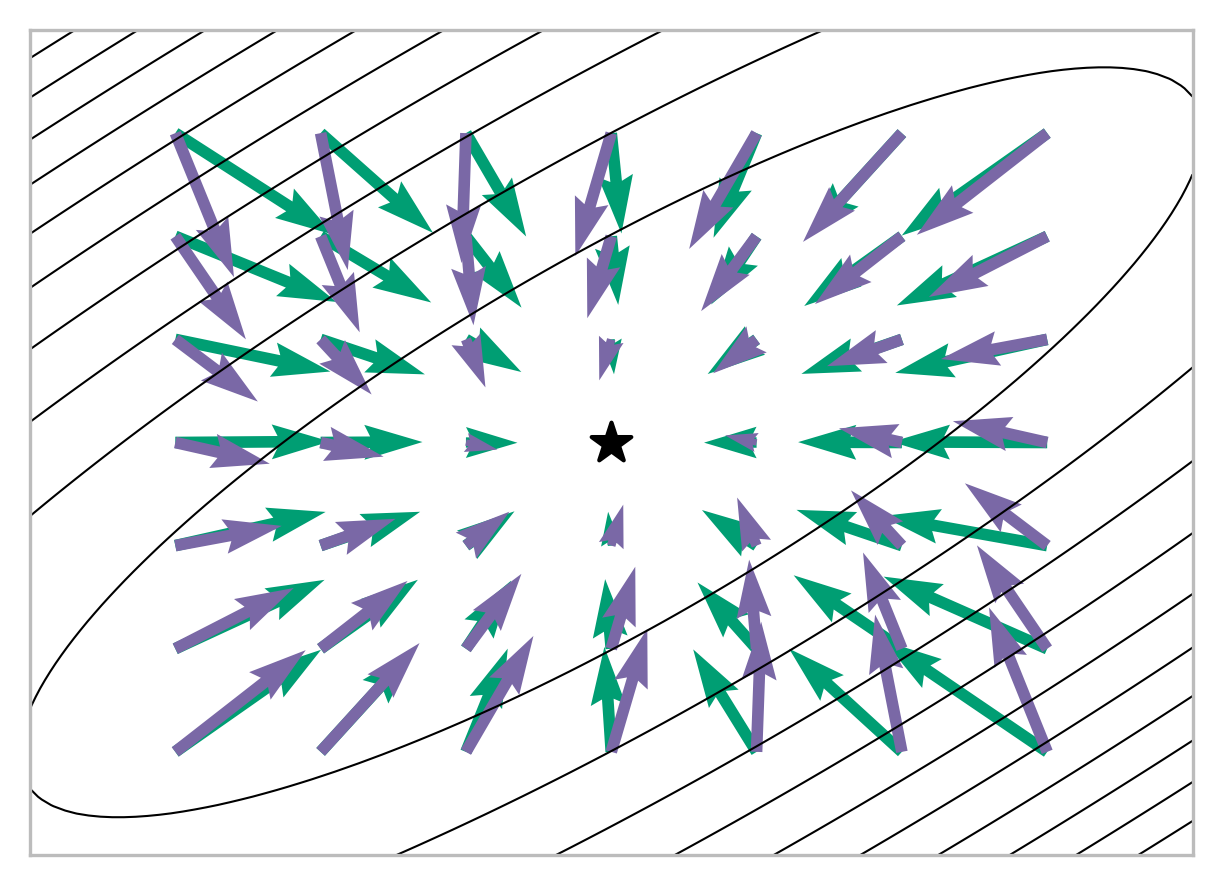}
        \caption{The optimum magnified $\times 1000$. At this scale the loss is well approximated by a quadratic. The Fisher of $q$, $F(\tilde\theta)$, approaches the Hessian. This is not true of $\tilde{F}(\tilde\theta)$, because the loss is not a MLE objective for $\tilde q$. Here the gradient vector lengths have been scaled down by a factor of 3 for clarity.}
    \end{subfigure}
    \caption{Comparison of SNGD and NGD in the sheep experiment of Section \ref{subsec:experiments_negbin}. Contours are of $\tilde f(\tilde\theta)$, an MLE objective for $q$.}
    \label{fig:ngdcomparison}
\end{figure}

On the face of it, then, when $f$ is a MLE objective for $q$, SNGD appears less desirable than NGD under $q$, inheriting only two of the properties listed above. However, none of these properties say anything about performance over large (or indeed, any \emph{non-inifinitesimal}) steps in parameter space. But there is one special case in which we can say something about the performance of NGD over large steps in parameter space: when NGD is applied to a MLE objective of an EF distribution with respect to the mean parameters of that distribution, a single undamped step will converge straight to the optimum (see Proposition \ref{thm:natgrad_meanparams_singlestep} in Appendix \ref{app:efnatgrads}).

When we apply SNGD with EF $\tilde q$, and with respect to mean parameters $\tilde\theta$, we meet 2 of the criteria for single-step convergence.  Although $\tilde f(\tilde\theta)$ is not in general a MLE for $\tilde q$, if it behaves \emph{approximately} like a MLE objective for $\tilde q$, then SNGD can make rapid progress over large distances in parameter space, and  \emph{approximately} converge in a single step.  And so, while SNGD loses the asymptotic efficiency of NGD under $q$, it is possible to obtain improved performance in the early to mid stages. This is exactly what we observe in the experiment of Section \ref{subsec:experiments_negbin}, as is further illustrated in Figure \ref{fig:ngdcomparison}.

The ability to make rapid progress early on is arguably more important for the overall performance of an optimiser in practice; however, if convergence to an exact (within machine precision range) optimum is required, it may be beneficial to switch to Newton-type methods during the late phase, an approach previously proposed by \citet{tatzel2022latephase}.

We now attempt to characterise what it means for $\tilde f$ to behave \emph{approximately} like a MLE for $\tilde q$. Let $\tilde q$ be an EF with mean parameters $\tilde\theta$. Furthermore, let $f(\tilde\theta)$ be an MLE objective for $q$ so that
\begin{align}
\begin{split}
    \tilde f(\tilde\theta) &= f(g(\tilde\theta)) \\
        &= -\mathbb{E}_{p_\mc{D}(x)}\big[\log q_{g(\tilde\theta)}(x)\big]
\end{split}
\end{align}
where $p_\mc{D}(x) = \frac{1}{n}\sum_{i=1}^n \delta(x - x_i)$ is the empirical density function. Using a trivial identity, we can rewrite this as
\begin{align}
\begin{split}
    \tilde{f}(\tilde\theta) &= \kl[\big]{\tilde q_{\tilde\theta^*}}{\tilde q_{\tilde\theta}} - \mathbb{E}_{p_\mc{D}(x)}\big[\log q_{g(\tilde\theta)(x)}\big] - \kl[\big]{\tilde q_{\tilde\theta^*}}{\tilde q_{\tilde\theta}} \\
        &= \tilde f^*(\tilde\theta) + h(\tilde\theta) \label{eqn:ftilde_decomp}
\end{split}
\end{align}
where $\tilde\theta^*$ are optimal parameters of $\tilde f$, and we have defined
\begin{align}
    \tilde f^*(\tilde\theta) &= \kl[\big]{\tilde q_{\tilde\theta^*}}{\tilde q_{\tilde\theta}},
\end{align}
and
\begin{align}
    h(\tilde\theta) &= -\mathbb{E}_{p_\mc{D}(x)}\big[\log q_{g(\tilde\theta)(x)}\big] - \kl[\big]{\tilde q_{\tilde\theta^*}}{\tilde q_{\tilde\theta}}.
\end{align}
$\tilde f^*(\tilde\theta)$ can be minimised by a single NGD step with respect to $\tilde q$ and $\tilde\theta$ (see Proposition \ref{thm:natgrad_meanparams_singlestep}). The second term, $h(\tilde\theta)$, can then be seen as distorting the natural gradient step, acting to move it away from the optimum $\tilde\theta^*$. If $h(\tilde\theta)$ is (approximately) constant with respect to $\tilde\theta$, then SNGD will (approximately) converge in a single step.

In Figure \ref{fig:negbin_sheeq_decomp} we show that in the sheep experiment of Section \ref{subsec:experiments_negbin}, the effect of the distortion term is indeed small, even when far from the optimum, illustrating why SNGD is able to make rapid progress on this problem.

\begin{figure}[ht]
    \centering
    \setlength{\baselineskip}{0.\baselineskip}
    \begin{subfigure}[t]{.3\linewidth}
        \centering
        \includegraphics[width=\linewidth]{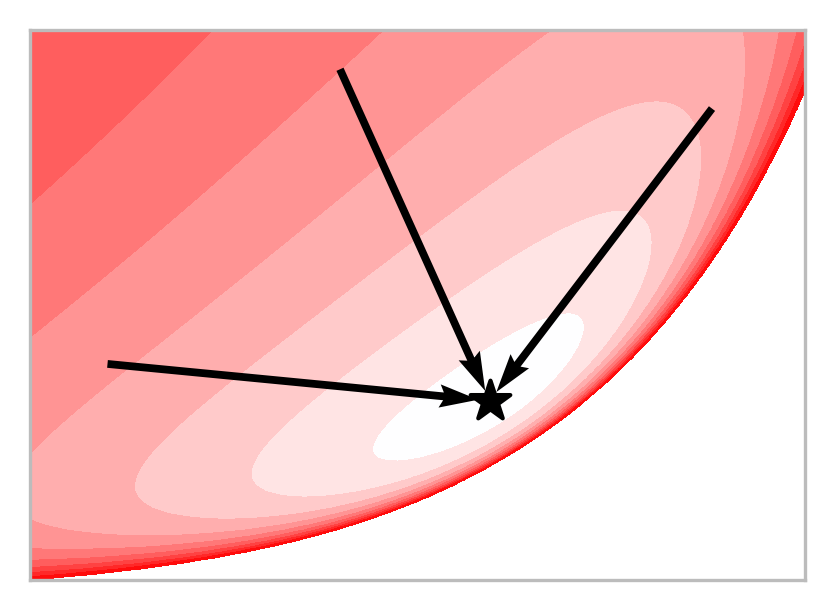}
        \caption{Contours of $\tilde f^*(\tilde\theta)$, Vectors show $[\tilde F(\tilde\theta)]^{-1}\nabla\tilde f^*(\tilde\theta)$ at 3 distinct locations; these always move straight to the optimum of $\tilde f$.}
        \label{fig:negbin_sheep_decomp_gammakl}
    \end{subfigure}
    \hfill
    \raisebox{.1\linewidth}{\scalebox{2}{$+$}}
    \hfill
    \begin{subfigure}[t]{.3\linewidth}
        \centering
        \includegraphics[width=\linewidth]{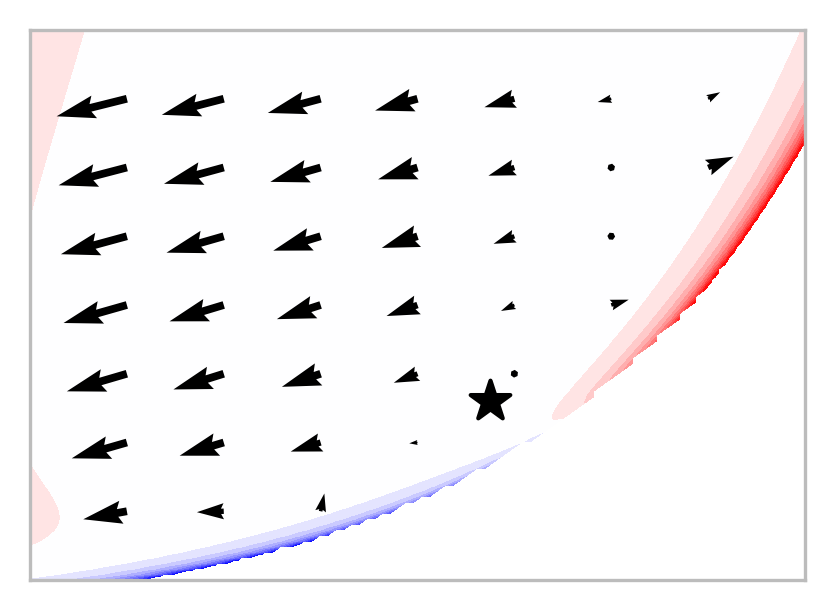}
        \caption{Contours of the distortion term $h(\tilde\theta)$. Vectors show $[\tilde F(\tilde\theta)]^{-1}\nabla h(\tilde\theta)$. at a regular grid of points}
        % \label{fig:negbin_sheeq_decomq_nuisance}
    \end{subfigure}
    \hfill
    \raisebox{.1\linewidth}{\scalebox{2}{$=$}}
    \hfill
    \begin{subfigure}[t]{.3\linewidth}
        \centering
        \includegraphics[width=\linewidth]{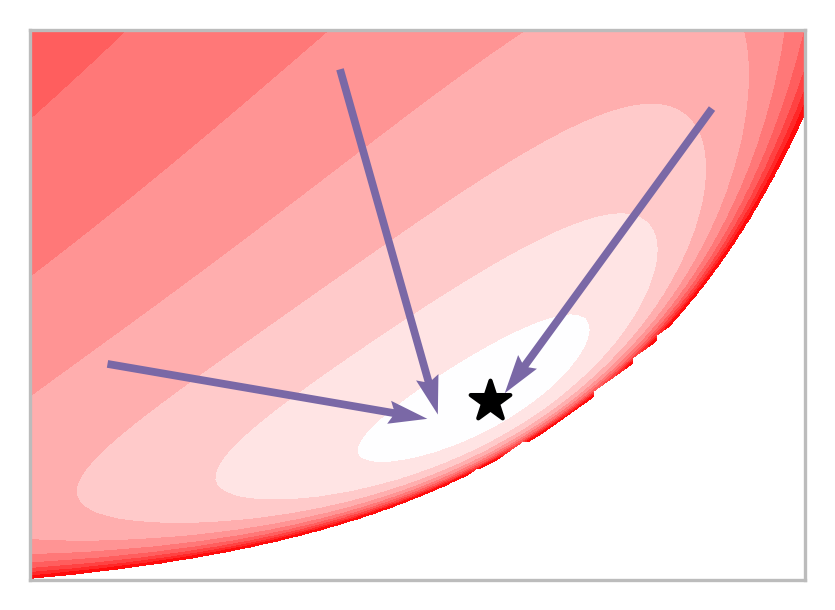}
        \caption{Contours of $\tilde f(\tilde\theta)$. Vectors now show the resulting SNGD step, $[\tilde F(\tilde\theta)]^{-1}\nabla\tilde f(\tilde\theta)$. This is the sum of the previous two vector fields.}
        % \label{fig:negbin_sheeq_decomq_ftilde}
    \end{subfigure}
    \caption{The objective, and SNGD steps for the sheep experiment of Section \ref{subsec:experiments_negbin}, each decomposed according to equation \eqref{eqn:ftilde_decomp}. The star marks the optimum of $\tilde f$.}
    \label{fig:negbin_sheeq_decomp}
\end{figure}

To add some intuition behind $h(\tilde\theta)$, let us assume that the data were generated from the model distribution $q$ with parameters $\theta^*$. Furthermore, take the infinite data limit,\footnote{We do this for convenience; for finite samples, the resulting expressions hold in expectation (over random data sets).} so that
\begin{align}
    \tilde f(\tilde\theta) &= -\mathbb{E}_{q_{\theta^*}(x)}\big[\log q_{g(\tilde\theta)}(x)\big]
\end{align}
and $\theta^* = g(\tilde\theta^*)$. The distortion term $h(\tilde\theta)$ can then be written as
\begin{align}
    \begin{split}
    h(\tilde\theta) &= -\mathbb{E}_{q_{\theta^*}(x)}\big[\log q_{g(\tilde\theta)}(x)\big] - \kl[\big]{\tilde q_{\tilde\theta^*}}{\tilde q_{\tilde\theta}} \\
        &= \kl[\big]{q_{g(\tilde\theta^*)}}{q_{g(\tilde\theta)}} - \kl[\big]{\tilde q_{\tilde\theta^*}}{\tilde q_{\tilde\theta}} + c\end{split}\label{eqn:distortionterm_infinitedata}
\end{align}
where $c$ is constant with respect to $\tilde\theta$.  For $h(\tilde\theta)$ to be roughly constant, we then require 
\begin{align}
    \nabla_{\tilde\theta}\big(\kl[\big]{\tilde q_{\tilde\theta^*}}{\tilde q_{\tilde\theta}}\big) &\approx \nabla_{\tilde\theta}\big(\kl[\big]{q_{g(\tilde\theta^*)}}{q_{g(\tilde\theta)}}\big). \label{eqn:behaveslike}
\end{align}
This is similar to the statement expressed by equation \eqref{eqn:klequality}, concerning the effect of $\tilde\theta$ on \emph{reverse} KL divergences between nearby points, whereas now we have a statement about the effect of $\tilde\theta$ on \emph{forward} KL divergences from the (potentially distant) optimal point.

%% file: appendix_efnatgrads.tex
\label{app:efnatgrads}

When NGD is applied to objectives involving a KL divergence (or related quantity), it has special properties when it is performed with respect to either the natural or mean parameters of an EF distribution, as the following propositions show.

\begin{proposition}
    \label{thm:natgrad_meanparams_singlestep}
    For EF $q$ with mean parameters $\mu \in \mathcal{M}$, and mean domain $\mathcal{M}$, let
    \begin{align}
        f(\mu) &= -\mathbb{E}_{p^*(x)}[\log q_\mu(x)] + c
    \end{align}
    where $p^*$ is any distribution with $\mathrm{supp}[p^*] \subseteq \mathrm{supp}[q]$, and $c$ is constant with respect to $\mu$. Then, $\forall\ \mu \in \mathcal{M}$,
    \begin{align}
        \mu - [F(\mu)]^{-1} \nabla f(\mu) &= \argmin_{\mu'\in\mathcal{M}}f(\mu').
    \end{align}
\end{proposition}
\begin{proof}
    First, let $\eta$ be dually coupled with $\mu$, so that $\mu = \mathbb{E}_{q_\eta(x)}[t(x)]$. From \eqref{eqn:ef_natgrad_meanparams}, we have that
    \begin{align}
    \begin{split}
        \mu - [F(\mu)]^{-1} \nabla f(\mu) &= \mu + \nabla_\eta\big(\mathbb{E}_{p^*(x)}\big[\log q_\eta(x)\big]\big) \\
            &= \mu + \nabla_\eta\big(\mathbb{E}_{p^*(x)}\big[t(x)^\top\eta - A(\eta) + \log\nu(x)\big]\big) \\
            &= \mu - \nabla A(\eta) + \mathbb{E}_{p^*(x)}\big[t(x)\big] \\
            &= \mathbb{E}_{p^*(x)}\big[t(x)\big]
    \end{split}
    \end{align}
    where we have used the standard EF identity $\nabla A(\eta) = \mu(\eta)$. It remains to show that $\mathbb{E}_{p^*(x)}[t(x)] = \argmin_{\mu'\in\mathcal{M}}f(\mu')$. Let us reparameterise $f(\mu)$ as $f_\eta(\eta) = f(\mu(\eta))$, then
    \begin{align}
        \begin{split}\nabla f_\eta(\eta) &= -\nabla_\eta\big(\mathbb{E}_{p^*(x)}\big[t(x)^\top\eta - A(\eta) + \log\nu(x)\big]\big) \\
            &= \mu - \mathbb{E}_{p^*(x)}\big[t(x)\big],
        \end{split}
    \end{align}
    and equating to zero, we find the unique stationary point at $\mu = \mathbb{E}_{p^*(x)}[t(x)]$. To show that this is a minimum, note that 
    \begin{align}
        \nabla^2 f_\eta(\eta) &= \nabla^2 A(\eta).
    \end{align}
    Given that $A$ is strictly convex and twice-differentiable (see \citet{wainwright2008expfam}), then $f_\eta$ is also strictly convex, and so must be minimised when $\mu = \mathbb{E}_{p^*(x)}[t(x)]$.
\end{proof}

Proposition \eqref{thm:natgrad_meanparams_singlestep} says that (undamped) NGD with respect to the mean parameters of EF $q$ will converge in a single step when the objective is a cross entropy from \emph{any} distribution $p^*$ to $q_\mu$ plus a constant that does not depend on $\mu$. Two straightforward corollaries of this are that NGD also has single step convergence when $f(\mu) = \mathrm{KL}\big[p^*(x)\ ||\ q_\mu(x)\big]$ (a forward KL divergence), or when $f(\mu) = -\sum_{i=1}^n \log q_\mu(x_i)$ (a MLE objective).

Typically, when performing MLE of an EF distribution, we simply think of it as `solving' for the parameters, or `moment matching'; Proposition \eqref{thm:natgrad_meanparams_singlestep} shows that actually, this is equivalent to performing (undamped) NGD with respect to $\mu$.

There is an \emph{almost}-symmetry regarding NGD with respect to the natural parameters of $q$, as the following proposition states.

\begin{proposition}
    \label{thm:natgrad_natparams_singlestep}
    For EF $q$ with natural parameters $\eta \in \mathcal{O}$, and natural domain $\mathcal{O}$, let
    \begin{align}
        f(\eta) &= \mathrm{KL}\big[q_\eta(x)\ ||\ q_{\eta^*}(x)\big] + c
    \end{align}
     where $q_{\eta^*}$ is from the same EF as $q$ with natural parameter $\eta^*$. Then, $\forall\ \eta \in \mathcal{O}$,
    \begin{align}
        \eta - [F(\eta)]^{-1} \nabla f(\eta) &= \argmin_{\eta'\in\mathcal{O}}f(\eta').
    \end{align}
\end{proposition}
\begin{proof}
    First, let $\mu$ be dually coupled with $\eta$, so that $\mu = \mathbb{E}_{q_\eta(x)}[t(x)]$. From \eqref{eqn:ef_natgrad_natparams}, we have that
    \begin{align}
    \begin{split}
        \eta - [F(\eta)]^{-1} \nabla f(\eta) &= \eta - \nabla_\mu\big(\mathrm{KL}\big[q_\mu(x)\ ||\ q_{\eta^*}(x)\big]\big) \\
            &= \eta - \nabla_\mu\big(A^*(\mu) + A(\eta^*) - \mu^\top\eta^*\big) \\
            &= \eta - \nabla A^*(\mu) + \eta^* \\
            &= \eta^*
    \end{split}
    \end{align}
    where: $q_\mu$ denotes the EF distribution with mean parameter $\mu$; $A^*$, the convex dual of A, is the negative entropy of $q$; and we have used the identity $\nabla A^*(\mu) = \eta$. It is clear from the properties of the KL divergence that $\eta^*$ must be the unique minimiser for $f(\eta)$.
\end{proof}
Note that Proposition \ref{thm:natgrad_natparams_singlestep} requires $q_{\eta^*}$ to be from the same family as $q_\eta$, which is not the case for the corresponding result with respect to mean parameters. This is a generalisation of a result found in \citet{hensman2013gp,salimbeni2018natgrad} for sparse Gaussian process VI. It is similar to to a result given by \citet{sato2001vb} for variational Bayes. Note that an analogous result to Proposition \ref{thm:natgrad_natparams_singlestep} can be derived for any Bregman divergence in the same manner.

%% file: appendix_efmixture.tex
\label{app:efmixture}

An EF mixture with $k$ components has density
\begin{align}
    q(x) &= \sum_{i=1}^k\pi_i\nu(x)\exp\big(t(x)^\top \eta_i - A(\eta_i)\big)
\end{align}
where $\pi_i > 0$ are the mixture probabilities with $\pi \in \Delta^{k-1}$, and $\eta_i$ are the natural parameters of mixture component $i$. We can also define a related \textit{mixture model}, a joint distribution in which we consider the mixture identity $z \in \{1, ..., k\}$ as a latent variable, with density
\begin{align}
    q(z, x) &= \pi_z\nu(x)\exp\big(t(x)^\top \eta_z - A(\eta_z)\big).\label{eqn:efmixturemodel}
\end{align}
A finite EF mixture is not, in general, an EF.  However, the corresponding mixture model (joint distribution) \textit{is} an EF. Consider its density,
\begin{align}
\begin{split}
    q(z, x) &= \pi_z\nu(x)\exp\left(t(x)^\top \eta_z - A(\eta_z)\right) \\
        &= \sum_{i=1}^k\mathbb{I}_i(z)\pi_i\nu(x)\exp\big(t(x)^\top \eta_i - A(\eta_i)\big) \\
        &= \nu(x)\prod_{i=1}^k\left[\pi_i\exp\big(t(x)^\top \eta_i - A(\eta_i)\big)\right]^{\mathbb{I}_i(z)} \\
        &= \nu(x)\exp\Bigg(
    \sum_{i=1}^k\mathbb{I}_i(z)\log\pi_i + \mathbb{I}_i(z)t(x)^\top \eta_i - \mathbb{I}_i(z)A(\eta_i)\Bigg) \\
        &= \nu(x)\exp\Bigg[
    \Bigg(\sum_{i=1}^{k-1}\mathbb{I}_i(z)\Big(\log\frac{\pi_i}{\pi_k} - A(\eta_i) + A(\eta_k)\Big)\Bigg) + \Bigg(\sum_{i=1}^k\mathbb{I}_i(z)t(x)^\top \eta_i\Bigg) + \log\pi_k - A(\eta_k)\Bigg], \label{eqn:efmixture_expfam}
\end{split}
\end{align}
where $\mathbb{I}_i(z)$ is an indicator function for the singleton set $\{i\}$. We can then recognise \eqref{eqn:efmixture_expfam} as having the form of an EF density with: base measure $\nu(.)$; log-partition $A(\eta_k) - \log\pi_k$; and sufficient statistics and natural parameters as listed in Table \ref{tbl:efmixture_statsparams}. 
Note that the statistic functions are linearly independent, and so the EF defined by \eqref{eqn:efmixturemodel} is minimal.

\begin{table}[ht]
    \caption{Sufficient statistics and natural parameters for an EF mixture model. Note that $t$ is a vector-valued function, and each $\eta_i$ is a vector of equal dimension.}
  \label{tbl:efmixture_statsparams}
  \centering
  \begin{tabular}{cc}
    \toprule
    Sufficient statistic & Natural parameter \\
    \midrule
        $\mathbb{I}_1(z)$ & $\log\frac{\pi_1}{\pi_k} - A(\eta_1) + A(\eta_k)$ \\
        ... & ... \\
        $\mathbb{I}_{k-1}(z)$ & $\log\frac{\pi_1}{\pi_k} - A(\eta_{k-1}) + A(\eta_k)$ \\
        $\mathbb{I}_1(z)t(x)$ & $\eta_1$ \\
        ... & ... \\
        $\mathbb{I}_k(z)t(x)$ & $\eta_k$ \\
    \bottomrule
  \end{tabular}
\end{table}